\pdfoutput=1

\documentclass[Afour,sageh,times,numbered]{./support/sagej}
\usepackage{moreverb,url}

\usepackage[colorlinks,bookmarksopen,bookmarksnumbered,citecolor=red,urlcolor=red]{hyperref}

\newcommand\BibTeX{{\rmfamily B\kern-.05em \textsc{i\kern-.025em b}\kern-.08em
T\kern-.1667em\lower.7ex\hbox{E}\kern-.125emX}}

\usepackage{times}
\usepackage{amsmath,amssymb,amsopn,amstext,amsfonts}
\usepackage{cancel}
\usepackage{soul}
\usepackage{lipsum,multicol}
\usepackage{stfloats}
\everymath{\displaystyle}
\usepackage{booktabs}
\usepackage{threeparttable}
\usepackage{tabularx}
\usepackage{pifont}  
\usepackage{amsmath,amsthm}
\usepackage{xfrac}
\usepackage{balance}
\usepackage{color}
\usepackage{mathtools}

\usepackage{tabularx}
\usepackage{bm}
\usepackage{diagbox}
\usepackage{float}
\usepackage{epstopdf}
\usepackage{url}
\usepackage{multirow}
\usepackage{multicol}
\usepackage{tikz}
\usepackage{enumitem}
\usepackage[ruled,vlined,linesnumbered]{algorithm2e}
\usepackage{pifont}
\usepackage{siunitx}
\usepackage[table,xcdraw]{xcolor}  
\usepackage{caption}

\definecolor{myblue}{rgb}{0.435, 0.667, 0.969}

\DeclareMathAlphabet\mathbfcal{OMS}{cmsy}{b}{n}

\newtheorem{theorem}{Theorem}

\definecolor{celadon}{rgb}{0.78, 0.93, 0.80}
\pagecolor{celadon}
\nopagecolor
\soulregister\cite7
\soulregister\citep7
\soulregister\citet7
\soulregister\ref7
\soulregister\it7
\soulregister\pageref7

\usepackage{arydshln}
\graphicspath{{figures/}}
\DeclareGraphicsExtensions{.pdf,.png,.jpg,.eps}

\begin{document}
\title{Memory-Efficient Boundary Map for Large-Scale Occupancy Grid Mapping
}

\author{Benxu Tang\affilnum{1}, Yunfan Ren\affilnum{1}, Yixi Cai\affilnum{1, 2}, Fanze Kong\affilnum{1}, Wenyi Liu\affilnum{1}, Fangcheng Zhu\affilnum{1}, Longji Yin\affilnum{1}, Liuyu Shi\affilnum{1} and Fu Zhang\affilnum{1}}

\affiliation{\affilnum{1}Mechatronics and Robotic Systems (MaRS) Laboratory, Department of Mechanical Engineering, The University of Hong Kong, Hong Kong SAR, China.\\
\affilnum{2}Department of Robotics, Perception, and Learning, KTH Royal Institute of Technology, Stockholm, Sweden.}

\corrauth{Fu Zhang, fuzhang@hku.hk}

\begin{abstract}

    Determining the occupancy status of locations in the environment is a fundamental task for safety-critical robotic applications. Traditional occupancy grid mapping methods subdivide the environment into a grid of voxels, each associated with one of three occupancy states: free, occupied, or unknown. These methods explicitly maintain all voxels within the mapped volume and determine the occupancy state of a location by directly querying the corresponding voxel that the location falls within. However, maintaining all grid voxels in high-resolution and large-scale scenarios requires substantial memory resources. In this paper, we {introduce} a novel representation that only maintains the boundary of the mapped volume. {Specifically, we explicitly represent the boundary voxels, such as the occupied voxels and frontier voxels, while free and unknown voxels are automatically represented by volumes within or outside the boundary, respectively.} As our representation maintains only a closed surface in two-dimensional (2D) space, instead of the entire volume in three-dimensional (3D) space, it significantly reduces memory consumption. Then, based on this 2D representation, we propose a method to determine the occupancy state of arbitrary locations in the 3D environment. We term this method as boundary map.
    Besides, we design a novel data structure for maintaining the boundary map, supporting efficient occupancy state queries.
    Theoretical analyses of the occupancy state query algorithm are also provided.
    Furthermore, to enable efficient construction and updates of the boundary map from the real-time sensor measurements, we propose a global-local mapping framework and corresponding update algorithms.
    Extensive benchmark experiments were conducted on various datasets, with results demonstrating that our method significantly reduces memory consumption while maintaining highly efficient map queries and updates. Moreover, we showcase a real-world application of the boundary map by deploying our mapping framework on a memory-constrained micro aerial vehicle (MAV) platform, enabling the robot to navigate in large and complex unknown environments. 
    {Finally, we will make our implementation of the boundary map open-source on GitHub to benefit the community:~\href{https://github.com/hku-mars/BDM}{\tt github.com/hku-mars/BDM}.}

\end{abstract}


\maketitle

\section{Introduction}
\label{sec:intro}

\begin{figure*}[t]
    \centering
    \includegraphics[width=\linewidth]{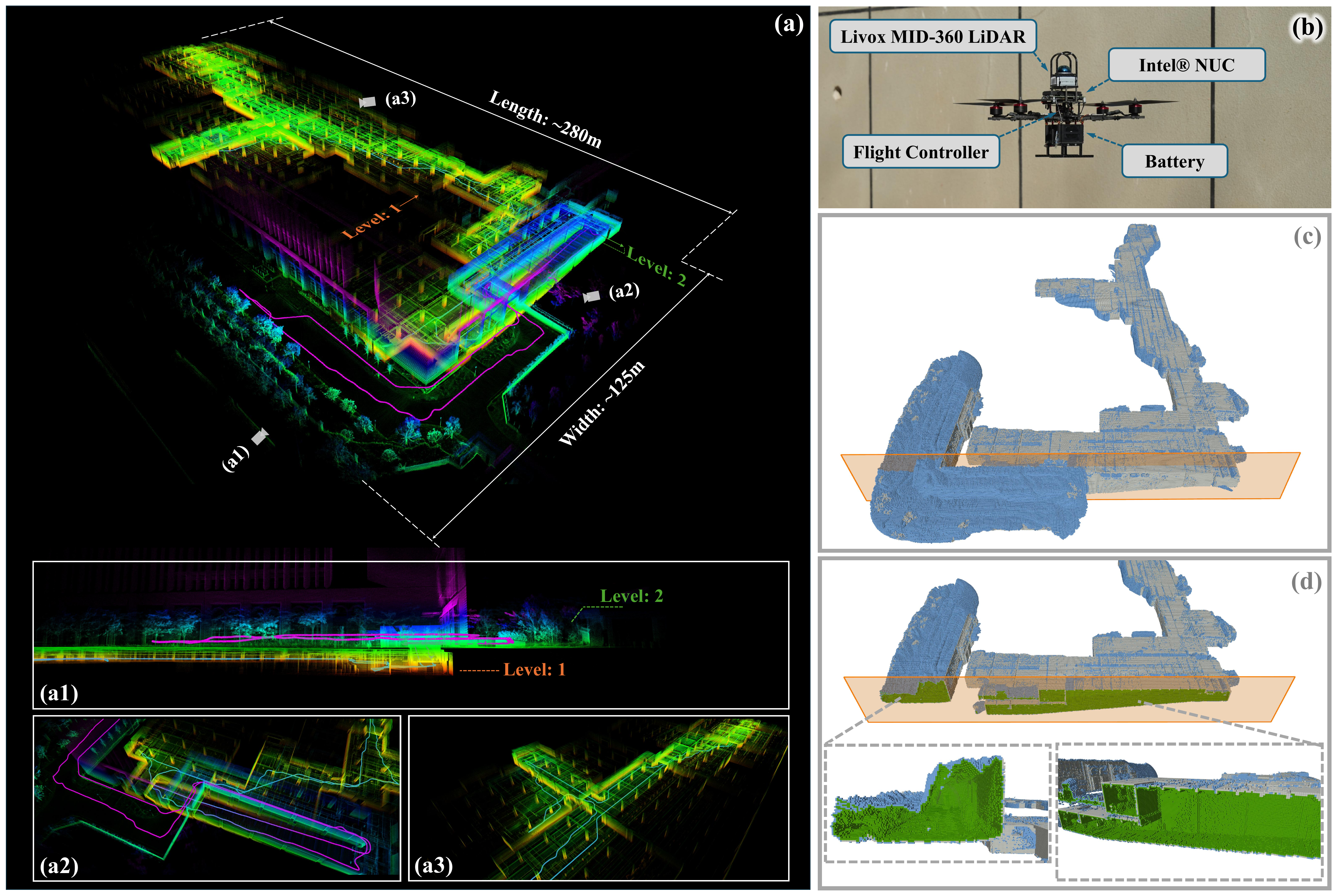}
    \caption{MAV performing a long-range autonomous navigation task in an unknown multi-level building, utilizing the proposed mapping framework. (a) Accumulated point cloud and the MAV trajectory at the end of the task, also visualized from multiple viewpoints in (a1), (a2), and (a3). (b) Hardware configuration of the MAV platform used to perform the navigation task. (c) Overview of the resulting global boundary map, which forms a closed structure. As the viewpoint is external, the low-dimensional property of the map is not evident. To better illustrate its internal structure, the map is sliced along the orange plane. (d) Cross-sectional view revealing the hollow interior of the boundary map. Voxels stored in boundary map are classified into three types, visualized in distinct colors. The green voxels lie on the interior layer of the map, having an occupancy state of free. The grey and blue voxels are located on the exterior layer, with occupancy states of occupied and unknown, respectively.}
    \label{fig:cover}
\end{figure*}

Autonomous robots have attracted significant attention in recent years and have been deployed across diverse domains, including search and rescue~\cite{rouvcek2020darpa, tranzatto2022cerberus}, disaster response~\cite{kawatsuma2012emergency, seungsub2017study}, surveillance~\cite{yoder2016autonomous, tabib2021autonomous}, and 3D reconstruction~\cite{isler2016information, schmid2020efficient}. These applications not only enhance operational efficiency but also ensure the safety of human personnel. Despite notable advancements, deploying autonomous robots in complex, unknown, and large-scale environments remains a formidable challenge. A critical issue lies in the efficient and accurate determination of the occupancy status—free, occupied, or unknown—for locations in the environment. This capability is foundational for autonomous navigation tasks, such as obstacle avoidance~\cite{lopez2017aggressive, kong2021avoiding, ren2022bubble, ren2025safety}, target search~\cite{yang2022far}, active perception and exploration~\cite{bircher2016receding, dang2019graph, zhou2021fuel, cao2021tare, tang2023bubble}.

A solution for the problem is the occupancy grid mapping, which can generally be divided into two categories: (i) grid-based methods~\cite{moravec1996robot, niessner2013real, ren2023rog} and (ii) octree-based methods~\cite{hornung2013octomap, duberg2020ufomap}. Grid-based methods discretize the environment into a fixed grid. Each voxel in grid stores an occupancy state of free, occupied or unknown. These grid voxels are mapped to a specific memory location in an array, or a hash table. Determining the occupancy state of locations in the environment involves directly querying for the corresponding voxel that the location falls within.
Octree-based methods, manage these grid voxels using a hierarchical tree structure. By merging voxels with identical occupancy status, octrees effectively reduce memory consumption compared to grid-based methods. 

A recent work D-Map~\cite{cai2023occupancy} introduces a hybrid representation that uses an octree managing unknown voxels and a grid map storing occupied voxels. This method {leverages} depth images to determine occupancy states in map updates, {avoiding the ray casting process that {is} commonly used in classical methods. This strategy enables occupancy state determination for large grid cells in the octree, thus avoiding the exhaustive visits of the smallest cells. Furthermore, grid cells with determined states are removed at each update, demonstrating a decremental property and reducing redundant grid visits, which further enhances map update efficiency.}

{However, challenges become pronounced in large-scale environments (e.g., spanning several kilometers \cite{geiger2013vision, maddern20171, agarwal2020ford, jung2024helipr}) and tasks require high-resolution mapping. All existing occupancy grid maps explicitly represent both free and occupied volumes in the three-dimension (3D), and as the environment size and resolution increase, memory consumption can become prohibitive (e.g., exceeding 120GB \cite{moravec1996robot}), restricting their application in such tasks. In addition to memory, large-scale environments also lead to substantial increases in map update time and query latency. Existing methods make trade-offs among memory consumption, update efficiency, and query performance. Grid-based methods offer high efficiency in map updates and queries, with constant time complexity of \(\mathcal{O}(1)\)~\cite{ren2023rog}. However, the memory consumption {grows} rapidly as map scale or resolution increases, making them impractical for environments larger than a few hundred meters.
Octree-based methods are more memory-efficient, but their hierarchical structure introduces computational overhead. The time complexity of updating or querying a voxel is \(\mathcal{O}(\text{log}(\frac{D}{d}))\), where \(D\) is the environment scale and \(d\) is the map resolution~\cite{cai2023occupancy}. Thus, as scale increases or resolution becomes finer, latency in updates and queries grows accordingly, which may degrade real-time performance in such settings.
D-Map improves map update efficiency but still requires traversing the octree during map queries, incurring the same time complexity of \(\mathcal{O}(\text{log}(\frac{D}{d}))\). Moreover, this method assumes a static environment. Once a voxel's occupancy state is determined in an update, it remains unchanged. This limits its applicability in dynamic real-world scenarios. Consequently, there is a pressing need for novel techniques that facilitate efficient occupancy mapping in large-scale environments by minimizing memory overhead, and simultaneously enabling rapid map updates and queries, while also supporting dynamic scenarios.}

To address these challenges, we propose a novel map representation and a real-time mapping framework. Unlike existing methods, which explicitly represent the entire mapped volume in the three-dimension (3D), {our method represents only the two-dimensional (2D) boundary voxels (see Section~\ref{sec:boundary}).} This low-dimensional representation significantly reduces memory consumption, achieving improvements by several times to orders of magnitude compared to both grid- and octree-based methods. 
{A problem {with} using 2D boundary voxel representation lies in the determination of the occupancy state of an arbitrary location in the 3D environment. To address this issue, we propose an occupancy state determination method based on the boundary voxels. Besides, to support efficient occupancy state queries, we design a novel data structure to maintain the boundary map. Our approach achieves an average time complexity of \(\mathcal{O}(1)\) for occupancy state queries, matching the efficiency of grid-based methods.} Furthermore, we introduce a global-local mapping framework to facilitate efficient real-time map updates, with the map update time also comparable to the grid-based methods.

In summary, {we propose a novel occupancy grid mapping system that} significantly reduces memory consumption while maintaining high efficiency for both map queries and updates.
Our contributions can be summarized as follows:

\section*{Contributions}

\begin{itemize}
    \item \textbf{The Boundary Map:}  
    We propose a novel occupancy grid map, {which represents only the two-dimensional (2D) boundary voxels instead of the entire three-dimensional (3D) volume. This low-dimensional representation significantly improves memory efficiency.}
    {Based on the 2D boundary voxels, we propose a method for determining the occupancy state of arbitrary locations in the 3D environment. We term our method as boundary map.
    In addition, we introduced a novel data structure to maintain the boundary map, from which an efficient query algorithm is then proposed.
    Our approach enables rapid querying of the occupancy state for locations in the environment, achieving an average time complexity of $\mathcal{O}(1)$.}

    \item \textbf{Real-Time Mapping Framework:} 
    {We developed a global-local mapping framework that supports efficient map updates from real-time sensor measurements. This framework leverages the boundary map as the global map and a robot-centric local map implemented as a uniform occupancy grid. Besides, we design a corresponding update method based on this map structure. We perform the ray casting process on local map to integrate new sensor measurements into the mapping framework.
    Additionally, we design an incremental method to compute boundary voxels from local map and update them to the global boundary map. We also propose an efficient method to load occupancy states from the global boundary map to the local map for fusing the new sensor measurements. 
    Furthermore, we incorporate a sliding mechanism for the local map and propose an incremental method to update the local and global map in the slide-in and slide-out regions.
    These designs ensure efficient updates of the mapping framework, making the overall update time comparable to that of grid-based methods and significantly faster than the octree-based methods. 
    }
    
    \item \textbf{Benchmarking and Real-World Application:}  
    Comprehensive benchmark experiments were conducted, comparing our framework with state-of-the-art methods in terms of memory consumption, update efficiency, query performance, and map accuracy. {Additionally, the proposed framework is successfully deployed in a long-range autonomous navigation task of a micro aerial vehicle (MAV) in a large-scale real-world scenario. Our mapping framework supports the MAV to navigate to long-range goals across multiple floors of a building. The environment spans approximately \(\mathrm{280m} \times \mathrm{125m} \times \mathrm{20m} \), and the total flight distance reached approximately \(\mathrm{1.25km}\). This demonstrates the practical applicability of our mapping framework.}
\end{itemize}

{
The remainder of this paper is organized as follows. Section~\ref{sec:relate} reviews related work on occupancy mapping. Section~\ref{sec:overview} presents an overview of our method, followed by a detailed introduction of the boundary map in Sections~\ref{sec:boundary}, \ref{sec:determine}, \ref{sec:data_and_impl}, and the mapping framework in Section~\ref{sec:framework}. Extensive benchmark experiments on various datasets are presented in Section~\ref{sec:benchmake} and the real-world application is introduced in Section~\ref{sec:real-world}. Finally, discussions are provided in Section~\ref{sec:discuss}, and conclusions are summarized in Section~\ref{sec:conclude}. }

\section{Related Work}
\label{sec:relate}
{In this section, we review previous studies on occupancy mapping, focusing on their occupancy state determination techniques and the underlying data structures they employ.}
\subsection{Occupancy State Determination Techniques}

Robotic applications are frequently safety-critical, especially when they involve operating in cluttered and unknown environments. These tasks demand that the mapping module accurately determine the occupancy state of a query location in the environment—whether it is free, occupied, or unknown—enabling the subsequent planning and control modules to make informed decisions to achieve goals safely. Over the years, various methods have been developed to tackle this challenge. Depending on their underlying assumptions, these approaches can be broadly classified into two main categories: Continuous Occupancy Models and Discrete Occupancy Models.

{Continuous Occupancy Models} assume an implicit spatial correlation within the environment. One prominent approach is the use of Gaussian Processes (GP), as proposed by \cite{o2012gaussian}, to model the relationship between sensor measurements and physical locations. GPs have the ability to predict the occupancy state for any location in the environment, but they come with significant computational overhead, particularly during occupancy state queries, where matrix inversion is required, leading to a time complexity of \( \mathcal{O}(N^3) \), where \(N\) is the number of accumulated sensor measurements since the beginning of the robotic application. This cubic complexity is also present during model training, making the approach computationally expensive. Furthermore, this method requires storing all accumulated sensor measurements, which leads to considerable memory consumption. Several techniques have been proposed to mitigate these issues. For example, \cite{kim2012building} reduces training overhead by clustering data and using Gaussian mixture models. \cite{kim2015gpmap} introduces sparse Gaussian processes to reduce training costs, and \cite{wang2016fast} divides the environment into blocks and maintains an octree structure within each block to improve computational efficiency. Other works \cite{guizilini2018towards, o2018variable, zhi2019continuous} compress sensor measurements into Gaussian distributions or specialized kernels to create more compact environment representations.
However, querying the occupancy state of a location on a continuous occupancy map still involves complex computations, which can hinder real-time performance, particularly on resource-constrained robotic platforms. Additionally, continuous mapping methods often require complex training procedures, further limiting their ability to perform in real time.

In contrast, {Discrete Occupancy Models} such as \cite{moravec1996robot, hornung2013octomap, duberg2020ufomap, ren2023rog} partition the environment into a grid of voxels, assuming that the occupancy state of each voxel is independent of others. Each voxel holds an independent occupancy state of free, occupied or unknown. This allows the occupancy state of a query location to be readily retrieved from the corresponding voxel that the location falls within. As a result, Discrete Occupancy Models typically exhibit significantly lower time complexity in occupancy state querying compared to continuous models, making them well-suited for real-time robotic applications. Moreover, these models construct the occupancy map directly from sensor measurements, bypassing the inference required by continuous models, which enhances their robustness and reliability for robotic tasks.
{However, when a more detailed environmental representation is required, Discrete Occupancy Models necessitate a higher map resolution (i.e., a smaller voxel size), leading to a cubical increase in the number of voxels to maintain. Furthermore, when applying in large-scale environments, it also leads to a substantial number of voxels in map. Both of these scenarios lead to significant memory demands, rendering these models challenging to apply in such extensive scenes.}

Both existing continuous and discrete occupancy models share a critical limitation: they explicitly represent full three-dimensional (3D) volume of free and occupied regions to fully capture the environment. However, in large-scale robotic applications, these volumes can become tremendous, particularly the free regions, which makes these methods challenging to represent such scenes without extensive memory usage. 

To address this, we introduce a novel low-dimensional boundary-based representation for occupancy states, built upon the discrete occupancy models. This representation only explicitly maintains voxels {located} on the boundary of the mapped volume (see Section~\ref{sec:boundary}).
These boundary voxels are significantly fewer in number compared to full volume of voxels, due to the dimensional reduction, and thus substantially reduce memory consumption.

\subsection{Data Structures}

In addition to occupancy state determination techniques, the efficiency of queries and updates, as well as memory consumption, is closely tied to the data structure and associated query methods of a map. Existing approaches can be broadly categorized into grid-based and octree-based methods, while grid-based methods can be further classified as array-based and hash-based methods, based on their underlying data structures.

Array-based methods, such as \cite{roth1989building, elfes1995robot, moravec1996robot}, represent the environment using a uniform grid structure, where all voxels are mapped into a contiguous memory block (i.e., an array). This design enables efficient map operations, including queries and updates, with a constant time complexity of \( \mathcal{O}(1) \). However, its primary drawback is the substantial memory consumption that escalates with increases in environment scale or map resolution, rendering it impractical for large-scale or high-resolution occupancy mapping tasks. To mitigate this memory burden, a recent method ROG-Map~\cite{ren2023rog} proposes maintaining only the uniform grid voxels surrounding the robot. As the robot moves, the map slides accordingly, discarding occupancy information in region that slides out the grid. This map structure and sliding mechanism ensure constant time complexity for map operations and fixed memory usage, as the map size remains constant. Nevertheless, this spatial forgetting approach results in the irreversible loss of information from previously explored areas.

An information-lossless alternative for improving memory efficiency is voxel hashing, namely the hash-based method. It is initially proposed by \cite{niessner2013real} and later adapted for occupancy mapping by \cite{zhou2020ego, zhou2023racer}. This method stores only voxels in the mapped volume (i.e., the free and occupied voxels) in a hash table, avoiding the need to allocate memory for the entire grid. Classified as a grid-based approach, it achieves an average-case time complexity of \( \mathcal{O}(1) \) for map operations, though its worst-case complexity rises to \( \mathcal{O}(n) \) due to potential hash collisions. As the environment size or map resolution increases, the growing number of voxels heightens the risk of hash collisions, potentially degrading the performance of map queries and updates. Additionally, maintaining the mapped volume can still lead to significant memory consumption, especially in large-scale scenarios or at high map resolutions.

To further address memory constraints, octree-based methods, such as Octomap \cite{hornung2013octomap}, employ an octree structure to recursively subdivide the environment into smaller grids. This allows voxels with identical occupancy states to be merged, reducing map size and memory usage. UFOMap \cite{duberg2020ufomap} makes implementation-level enhancements of this octree-based structure, which further improves memory efficiency. However, the trade-off for these octree-based methods is that map operations exhibit a logarithmic time complexity of \( \mathcal{O}(\log (\frac{D}{d})) \), where \(D\) represents the environment scale and \(d\) denotes the map resolution. This logarithmic complexity compromises the efficiency of map operations including queries and updates. A more recent approach, D-Map \cite{cai2023occupancy}, leverages a hybrid structure, storing unknown voxels in an octree and occupied voxels in a separate hash grid map. This method eliminates the need for ray casting—commonly used in the grid-based and octree-based methods—by introducing a novel strategy that utilizes depth images to determine occupancy states during map updates. This update strategy enables occupancy state determination for large grid cells in the octree, avoiding the exhaustive visits of the smallest cells. Moreover, D-Map removes grid cells with determined occupancy state at each update, which reduces redundant grid visits. These strategies collectively contribute to {an} efficient map update in D-Map. 
However, D-Map assumes a static environment, which poses challenges in handling sensor noise and dynamic obstacles compared to ray casting-based methods such as~\cite{moravec1996robot, hornung2013octomap}. As a result, D-Map struggles in dynamic settings and with sensor noise: once a location is marked as occupied, it cannot be reverted to free, even if future scans no longer observe the obstacle. This limitation restricts its applicability in real-world scenarios where environments have frequent changes.

{Unlike existing methods which make trade-offs between memory consumption, map update efficiency and query performance, we present a comprehensive mapping system that ensures both query and update efficiency while significantly reducing the memory consumption. First, we design a data structure to maintain the boundary map, supporting rapid queries. 
This novel data structure integrates a 2D hash-based grid map to store boundary voxels. An efficient algorithmic implementation for occupancy state queries is then proposed based on this data structure, achieving an average query time complexity of \( \mathcal{O}(1) \), comparable to the array- and hash-based methods. In addition, we introduce a real-time global-local mapping framework and a corresponding update method to achieve efficient updates. This framework achieves update efficiency comparable to array-based methods, while inheriting the capability to effectively handle sensor noise and dynamic objects. Most notably, by leveraging the low-dimensional boundary representation, our approach demonstrates significantly superior memory efficiency compared to the methods discussed above, making it highly suitable for robotic applications in large-scale scenes and high-resolution missions.}

\begin{figure*}[t] 
    \centering
    \includegraphics[width=\textwidth]{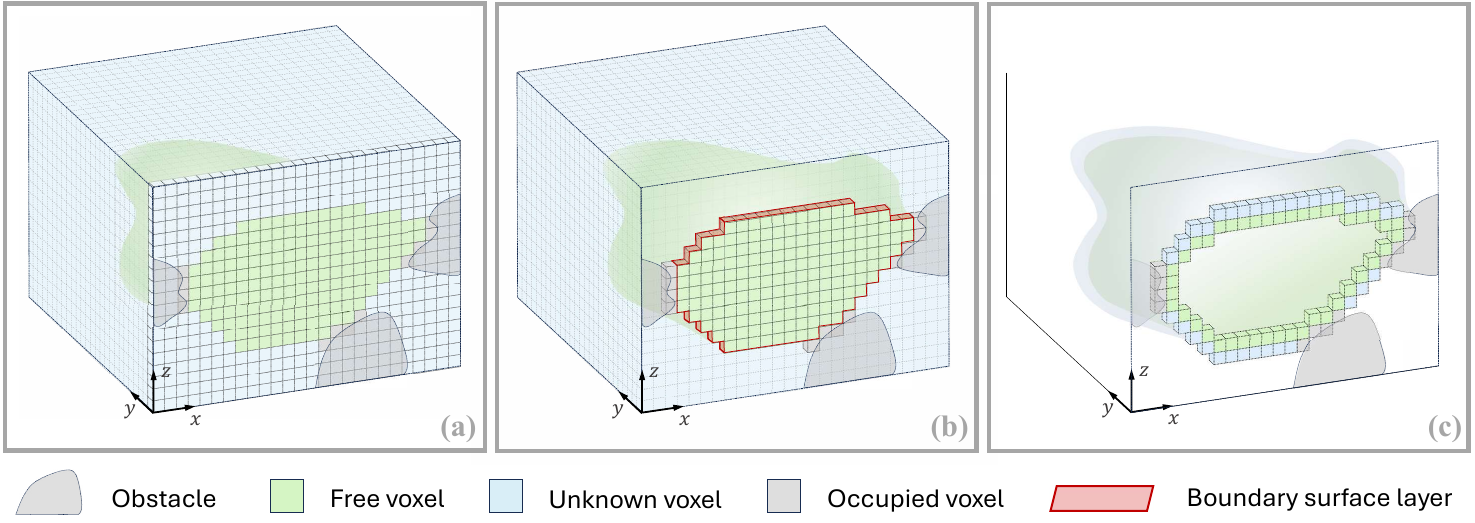}
    \vspace{-10pt} 

    \caption{(a) A uniform occupancy grid where each voxel is classified into one of three states: free, unknown, or occupied, represented by green, blue, and grey, respectively. For clarity, only the voxels along the cross-sectional plane are shown, with the remaining voxels represented by a simplified shape. (b) The boundary surface layer of the free space, highlighted in red, which separates the free regions from the adjacent occupied or unknown regions.  (c) Our proposed boundary map, which includes the free, unknown and occupied voxels that neighboring the boundary surface layer.}

    \label{fig:boundary_map}
\end{figure*}

\section{Overview}
\label{sec:overview}

{In the following sections, we present the detailed designs of our proposed method. To begin, we {introduce} the definition of the boundary map in Section~\ref{sec:boundary}. In Section~\ref{sec:determine}, we present theoretical foundations and methods for determining the occupancy state of an arbitrary location in the 3D environment by the 2D boundary map.
Next, in Section~\ref{sec:data_and_impl}, we describe the data structure designed to maintain the boundary map (see Section~\ref{sec:data_struct}), followed by a detailed algorithm based on this data structure that enables efficient occupancy state queries (see Section~\ref{sec:determine_impl}). Furthermore, we provide a time complexity analysis of this algorithm in Section~\ref{sec:determine_complx}.
Finally, Section~\ref{sec:framework} introduces our real-time global-local mapping framework. In this framework, the local map is a robot-centric occupancy grid map that maintains only the occupancy information around the robot. The global map is the low-dimensional boundary map, maintaining occupancy information located outside the local map region (see Section~\ref{sec:map_structure}). The complete update method of the mapping framework is then detailed in Section~\ref{sec:update}. Given this global-local structure where the global boundary map does not maintain information within the robot-centric region, we introduce a region-based query strategy for determining the occupancy state for different locations in the environment (see Section~\ref{sec:query}).}

\section{The Boundary Map}
\label{sec:boundary}

\begin{figure}[t] 
    \centering
    \includegraphics[width=\linewidth]{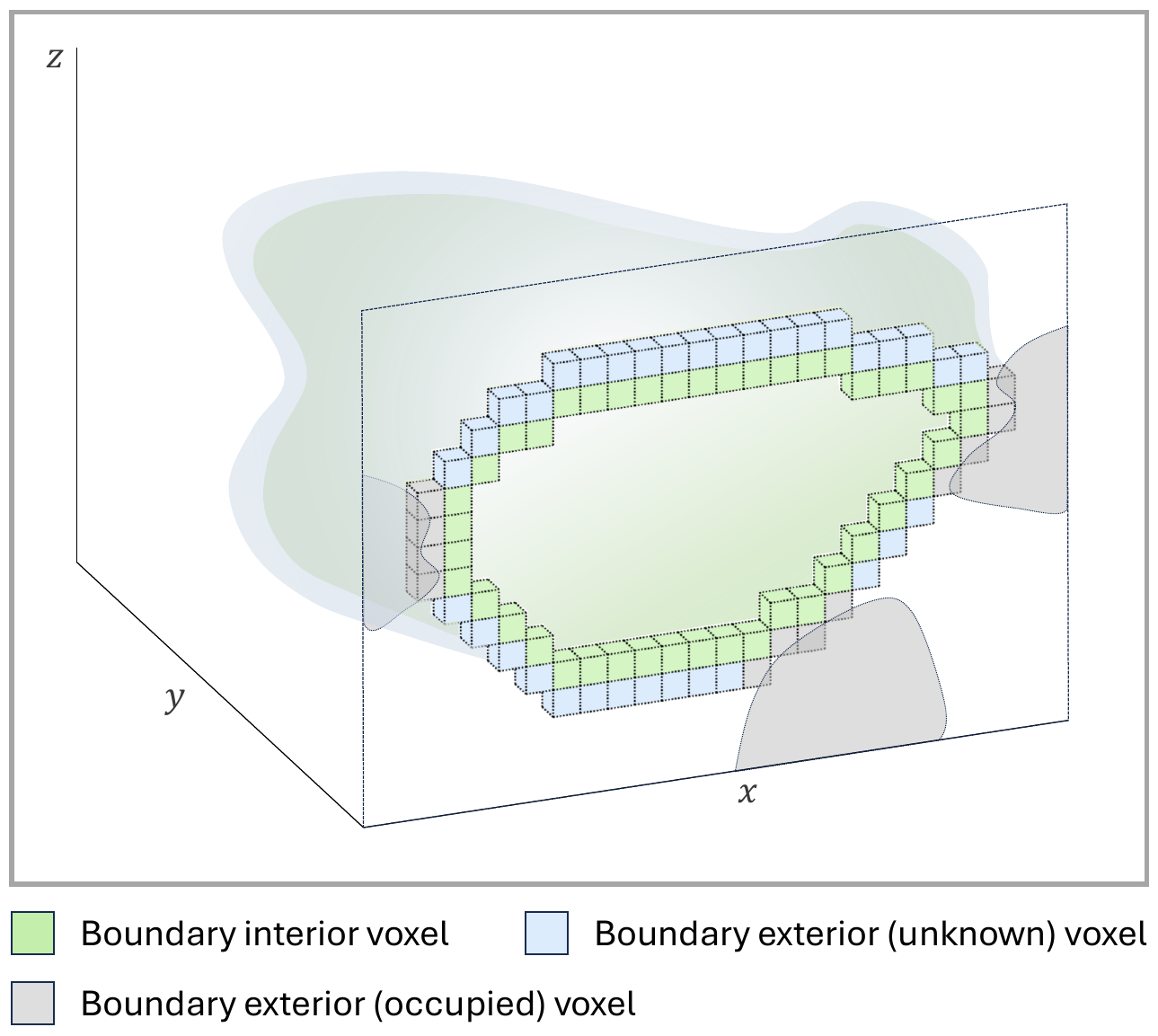}
    \vspace{-10pt} 

    \caption{The boundary map is composed by \textbf{boundary voxels}, which include \textbf{boundary interior voxels} (\(\mathbf{b}_{\mathrm{int}}\)) and \textbf{boundary exterior voxels} (\(\mathbf{b}_{\mathrm{ext}}\)). Boundary interior voxels are free voxels neighboring unknown or occupied voxel(s), and are visualized in green. Boundary exterior voxels include \textbf{boundary exterior (unknown) voxels} (\(\mathbf{b}_{\mathrm{ukn}}\)) and \textbf{boundary exterior (occupied) voxels} (\(\mathbf{b}_{\mathrm{occ}}\)), represented in blue and grey, respectively. Boundary exterior (unknown) voxels have an unknown occupancy state and neighboring free voxel(s). Boundary exterior (occupied) voxels are all voxels whose occupancy state is occupied.}

    \label{fig:boundary_def}
\end{figure}

{We define the voxels that constitute the boundary map as \textbf{boundary voxels}. In the following, we present the categorization of boundary voxels based on their occupancy states and neighboring configurations.}

In occupancy grid mapping, the known regions that are observed by sensor measurements are composed of voxels labeled as either free or occupied, while the remainder are categorized as unknown. We begin by focusing on the representation of occupied voxels. Since a sensor can only capture the surface of 3D objects, only voxels corresponding to the object’s surface are typically marked as occupied. Consequently, occupied voxels inherently exhibit a boundary-like structure. This further implies that occupied voxels constitute only a small subset of the entire map. We therefore explicitly store all occupied voxels in the boundary map, classifying them as one category of boundary voxels.

In contrast, free voxels often constitute the majority of the known regions in the environment. Explicitly storing all of them leads to substantial memory consumption, particularly in large-scale scenarios. To address this limitation, we propose a boundary representation of the free space that encodes only the interface between free and other non-free regions, avoiding the need to store the entire volume of free voxels. To formalize this representation, we first introduce the concept of the {boundary surface layer}. This layer consists of surfaces that separate free voxels from adjacent unknown or occupied voxels, as illustrated in red in Figure~\ref{fig:boundary_map}(b). While representing this boundary surface layer as a continuous geometric shape within a voxel-based framework would be challenging, we instead utilize voxel representations. Specifically, the boundary surface layer is represented by pairs of adjacent voxels. Each such pair consists of two 6-connected neighboring voxels that ``straddle'' the boundary surface layer. These voxel pairs are also categorized as boundary voxels.

We now describe the formal definition of these voxel pairs. The first type of voxel of such pair corresponds to the voxel residing on the {interior} side of the boundary surface layer. We define it as the {boundary interior voxel}, denoted as \(\mathbf{b}_\mathrm{int}\). Specifically, a voxel is classified as a {boundary interior voxel}, if its occupancy state is free and it has at least one 6-connected neighbor that is either unknown or occupied.

The second type lies on the {exterior} side of the boundary surface layer. Specifically, these are voxels whose occupancy state is either unknown or occupied, and that have at least one 6-connected neighbor that is free. Since all occupied voxels are already classified as boundary voxels and explicitly included in the boundary map, we only need to represent the exterior voxels that are in the unknown state. Formally, such voxels are defined as the {boundary exterior (unknown) voxel}, denoted as \(\mathbf{b}_\mathrm{ukn}\).

Notably, the occupied voxels also lie at the exterior of the boundary surface layer. Thus, we name them as the {boundary exterior (occupied) voxel}, denoted as \(\mathbf{b}_\mathrm{occ}\). 
In addition, we denote the boundary exterior (unknown) voxel and the boundary exterior (occupied) voxel together as the {boundary exterior voxel}, denoted as \(\mathbf{b}_\mathrm{ext}\).

In summary, the boundary map is composed of \textbf{boundary voxels}. The boundary voxel is classified as \textbf{boundary interior voxel} \(\mathbf{b}_\mathrm{int}\) and \textbf{boundary exterior voxel} \(\mathbf{b}_\mathrm{ext}\), where the \(\mathbf{b}_\mathrm{ext}\) is further classified into \textbf{boundary exterior (unknown) voxel} \(\mathbf{b}_\mathrm{ukn}\) and \textbf{boundary exterior (occupied) voxel} \(\mathbf{b}_\mathrm{occ}\). {The definition of boundary voxels is formulated as follows:}
\begin{equation}
\label{eqa:boundary_def}
{
\begin{aligned}
&\mathbf{b}_\mathrm{int} \in \left\{ \mathbf{n} \;\middle|\;
\begin{aligned}
&\texttt{occ}(\mathbf{n}) = \textit{free}, \\
&\exists\, \texttt{nbr}_6(\mathbf{n}) \in \{ \textit{unknown}, \textit{occupied} \}
\end{aligned}
\right\}, \\[4pt]
&\mathbf{b}_\mathrm{ukn} \in \left\{ \mathbf{n} \;\middle|\;
\begin{aligned}
&\texttt{occ}(\mathbf{n}) = \textit{unknown}, \\
&\exists\, \texttt{nbr}_6(\mathbf{n}) = \textit{free}
\end{aligned}
\right\}, \\[4pt]
&\mathbf{b}_\mathrm{occ} \in \left\{ \mathbf{n} \;\middle|\;
\texttt{occ}(\mathbf{n}) = \textit{occupied}
\right\}.
\end{aligned}
}
\end{equation}
{where \(\mathbf{n}\) denotes a voxel in the environment, \(\texttt{occ}(\mathbf{n})\) denotes the occupancy state of the voxel \(\mathbf{n}\), and \(\texttt{nbr}_6(\mathbf{n})\) denotes occupancy states of the 6-connected neighbors of the voxel \(\mathbf{n}\).} In addition, we define the \textbf{boundary voxel status} of a voxel \(\mathbf{n}\) to indicate whether it is a boundary voxel and, if so, its specific type.

An illustration of the definition of boundary voxels is presented in Figure~\ref{fig:boundary_def}.

\section{Occupancy State Determination}
\label{sec:determine}
In this section, we present the theoretical foundations and methods for determining the occupancy state of an arbitrary location in the environment by the boundary map. 

\begin{figure}[t] 
    \centering
    \includegraphics[width=\linewidth]{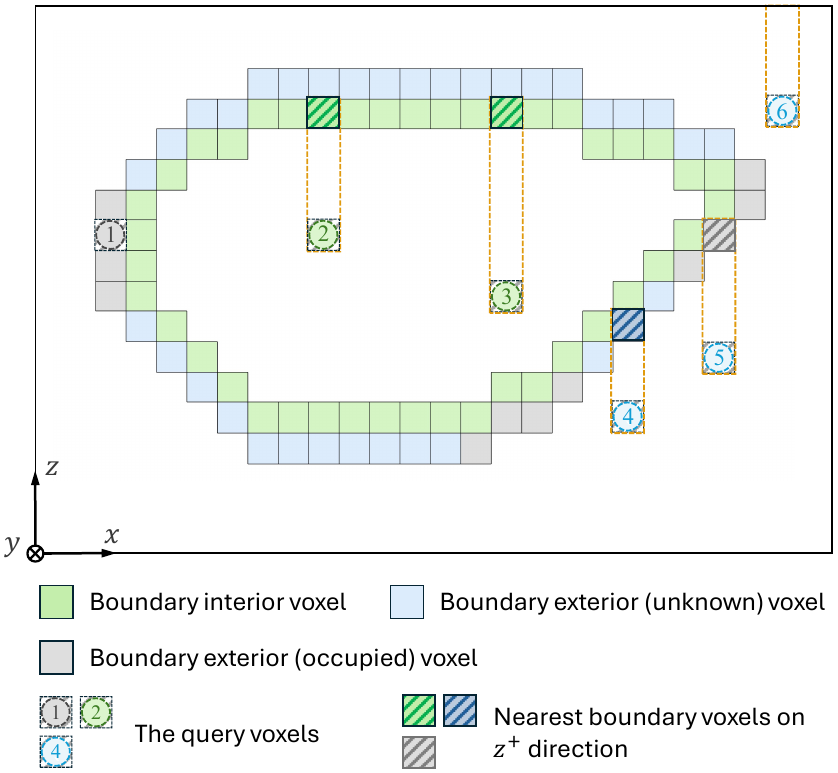}

    \caption{Occupancy state determination based on our boundary map. The diagram shows the determination process for several example query voxels (nos. 1–5). If the query voxel is located exactly on the boundary map, the query voxel state can be determined by type of the boundary voxel directly (query voxel no. 1). Otherwise, we search along the \(z^+\) direction {(i.e., a representative direction for illustration propose)} to find the nearest boundary voxel on the boundary map, denoted as \(\mathbf{b}_{{nn}}\). If \(\mathbf{b}_{{nn}}\) is a boundary interior voxel \(\mathbf{b}_{\mathrm{int}}\), the query voxel is determined as free (query voxels no. 2 and 3). If the \(\mathbf{b}_{{nn}}\) is a boundary exterior voxel \(\mathbf{b}_{\mathrm{ext}}\)(i.e., either unknown \(\mathbf{b}_{\mathrm{ukn}}\) or occupied \(\mathbf{b}_{\mathrm{occ}}\)), the query voxel is determined as unknown (query voxels no. 4 and 5). If \(\mathbf{b}_{{nn}}\) does not exist, the query voxel is determined as unknown too (query voxel no. 6).
}

    \label{fig:query}
\end{figure}

In the following discussions, voxels are represented by their indices. Let the index of a voxel be denoted by 
\begin{equation}
\mathbf{p}_i = (p_i^x, p_i^y, p_i^z) \in \mathbb{Z}^3,
\end{equation}
and its corresponding position in the environment by 
\begin{equation}
\mathbf{p}_d = (p_d^x, p_d^y, p_d^z) \in \mathbb{R}^3.
\end{equation}
The relationship between a voxel's index and its position is given by:
\begin{equation}
\begin{split}
p_i^x &= \text{round}\left( \frac{p_d^x}{d} \right), \\
p_i^y &= \text{round}\left( \frac{p_d^y}{d} \right), \\
p_i^z &= \text{round}\left( \frac{p_d^z}{d} \right),
\end{split}
\label{eqa:round}
\end{equation}
where \( d \) is the map resolution. For brevity, in subsequent discussions we refer to voxels solely by their integer indices and omit the subscript \( i \).

In occupancy grid mapping, determining the occupancy state of a location in the environment corresponds to determin{ing} the occupancy state of the voxel that the location falls within. We refer to this voxel as the query voxel, denoted as \(\mathbf{q} = (q^x, q^y, q^z) \in \mathbb{Z}^3\). In the following, we introduce the method for determining the occupancy state of the query voxel \(\mathbf{q}\). This determination procedure is also outlined in Algorithm~\ref{alg:occupancy_state}.

\begin{algorithm}[t]
 \SetAlgoLined
 \caption{Determine Occupancy State}
 \label{alg:occupancy_state}
 \KwIn{Query voxel $\textbf{q}$, Search direction $\mathcal{E}$}
 \KwOut{Occupancy state of $\textbf{q}$: \textit{free}, \textit{occupied}, or \textit{unknown}}
 \BlankLine
 $\mathbf{b}_{nn}, r_\text{min} \gets \mathtt{FindNearestInDirection}(\mathbf{q}, \mathcal{E})$\;
\If{$\mathbf{b}_{nn} \neq \mathtt{null}$}{
    \If{$\mathbf{b}_{nn} = \mathbf{b}_\mathrm{int}$}{
        \Return \textit{free}\;
    }
    \Else{
        \If{$\mathbf{b}_{nn} = \mathbf{b}_\mathrm{occ}$ \textbf{and} $r_\mathrm{min} = 0$}{
            \Return \textit{occupied}\;
        }
    }
}
\Return \textit{unknown}\;
\end{algorithm}

In the subsequent discussions, we use the superscripts \( ^{+} \) and \( ^{-} \) to denote the positive and negative directions along an axis, respectively. For example, the positive direction along the \( z \)-axis is denoted by \( z^+ \) and the negative direction by \( z^- \).

To determine the occupancy state of a query voxel \(\mathbf{q}\), we begin by searching for its nearest boundary voxel till the spatial extent of the environment along one of these six directions: \(\{x^+, x^-, y^+, y^-, z^+, z^- \} \).
The selected direction is referred to as the search direction \(\mathcal{E}\), and it serves as an input to Algorithm~\ref{alg:occupancy_state}. For illustrative purposes, we select \(z^+\) as the search direction in the following discussion.

We begin by considering the case where this nearest boundary voxel is successfully found along the search direction \(\mathcal{E}\) (e.g., \(z^+\)). We denote this boundary voxel as \(\mathbf{b}_{nn} = (b_{nn}^x, b_{nn}^y, b_{nn}^z) \in \mathbb{Z}^3 \), and define the corresponding nearest distance to the query voxel \(\mathbf{q}\) as \(r_{\text{min}}\), which is computed as:

\begin{equation}
\label{eqa:r_min}
r_{\text{min}} = \left|{b}_{{nn}}^z - {q}^z \right|.
\end{equation}
The above process is named as \(\mathtt{FindNearestInDirection}\) function (Line 1). 

Then, the occupancy state of the query voxel \(\mathbf{q}\) is determined based on the type of \(\mathbf{b}_{{nn}}\) and the nearest distance \( r_{\text{min}} \).
If \(r_{\text{min}} = 0\), indicating that \(\mathbf{q}\) lies exactly on the boundary map, its occupancy state is defined by the boundary voxel definition in Equation~\ref{eqa:boundary_def}. Specifically, in this case, \(\mathbf{q}\) corresponds to one of three types: boundary interior voxel \(\mathbf{b}_{\mathrm{int}}\), boundary exterior (unknown) voxel \(\mathbf{b}_{\mathrm{ukn}}\), or boundary exterior (occupied) voxel \(\mathbf{b}_{\mathrm{occ}}\), with its occupancy state determined as free, unknown, or occupied, respectively.

If \( r_{\text{min}} > 0 \), the following theorem is introduced to determine the occupancy state of \(\mathbf{q}\).

\begin{theorem}
\label{thm:query}

If \( r_{\text{min}} > 0 \), the occupancy state of the query voxel \(\mathbf{q}\) is determined as follows: if \(\mathbf{b}_{{nn}}\) is a boundary interior voxel \(\mathbf{b}_{\mathrm{int}}\), then the occupancy state of \(\mathbf{q}\) is determined as {free}. Conversely, if \(\mathbf{b}_{{nn}}\) is a boundary exterior voxel \(\mathbf{b}_{\mathrm{ext}}\), including both \(\mathbf{b}_{\mathrm{ukn}}\) and \(\mathbf{b}_{\mathrm{occ}}\), the occupancy state of \(\mathbf{q}\) is {unknown}.
\end{theorem}
{{An intuitive explanation for the theorem is given below, with a rigorous derivation presented in Appendix~\ref{app:proof_query}.} Traversing a boundary surface layer indicates an occupancy state transition between a free voxel and another occupancy state (i.e., unknown or occupied). 
By construction, there can be no boundary surface layer between \(\mathbf{b}_{{nn}}\) and \(\mathbf{q}\), since any such surface would imply the existence of a closer boundary voxel to \(\mathbf{q}\). It follows that no occupancy state transition occurs along the path from \(\mathbf{b}_{{nn}}\) to \(\mathbf{q}\). Therefore, if \(\mathbf{b}_{{nn}}\) has an occupancy state of free, the query voxel \(\mathbf{q}\) must also be free.  
Conversely, if \(\mathbf{b}_{{nn}}\) has an unknown or occupied state, then \(\mathbf{q}\) is unknown. The reason \(\mathbf{q}\) cannot be occupied is that all occupied voxels are explicitly encoded in the boundary map, and \(\mathbf{q}\) is not a boundary voxel by the condition of the theorem.}

Next, we consider the case where no boundary voxel is found along the search direction \(\mathcal{E}\) within the spatial extent of the environment. Under this condition, the occupancy state of \(\mathbf{q}\) is determined by the following theorem.

\begin{theorem}
\label{thm:query_null}
If no boundary voxel is found along the search direction \(\mathcal{E}\) within the extent of the environment, the occupancy state of the query voxel \(\mathbf{q}\) is determined as unknown.
\end{theorem}

{A rigorous proof of this theorem is provided in Appendix~\ref{app:proof_query_null}. In the following, we present an intuitive understanding of the
proof’s approach. We consider a reference voxel located beyond the spatial extent of the environment. By the condition of the theorem, no boundary surface layer exists between the reference voxel and the query voxel \(\mathbf{q}\), it follows that no transition has occurred between free and other occupancy states. Since this reference voxel lies outside the observable domain, its occupancy state is considered unknown. Moreover, the query voxel \(\mathbf{q}\) cannot be in occupied state as it is not a boundary voxel. Thus, the occupancy state of \(\mathbf{q}\) is unknown.}

In summary, if $\mathbf{b}_{{nn}}$ exists, we first compute its distance to the query voxel \(\mathbf{q}\). This distance is then used to determine whether the query voxel \(\mathbf{q}\) lies on the boundary map. If it does, its occupancy state is directly determined by the type of $\mathbf{b}_{{nn}}$. Otherwise, the occupancy state is determined by Theorem~\ref{thm:query}. If $\mathbf{b}_{{nn}}$ does not exist, the occupancy state of the query voxel \(\mathbf{q}\) is determined as unknown by Theorem~\ref{thm:query_null}. {This procedure is summarized in Algorithm~\ref{alg:occupancy_state} (Lines~2--12), in which some conditional branches are merged for brevity and clarity.} An illustrative example is provided in Figure~\ref{fig:query}.


\section{Data Structure and Implementation}
\label{sec:data_and_impl}
{In this section, we first introduce a novel data structure designed to maintain the boundary map (see Section~\ref{sec:data_struct}). Figure~\ref{fig:hash_2d_grid} presents an illustration of this data structure.
Based on the proposed data structure, we describe the algorithmic implementation for efficiently determining the occupancy state of an arbitrary location in the environment (i.e., performing a map query) (see Section~\ref{sec:determine_impl}). Then, an analysis of time complexity of the algorithm is presented in Section~\ref{sec:determine_complx}.}

\subsection{Data Structure of the Boundary Map}
\label{sec:data_struct}

\begin{figure}[t]
    \centering
    \includegraphics[width=\linewidth]{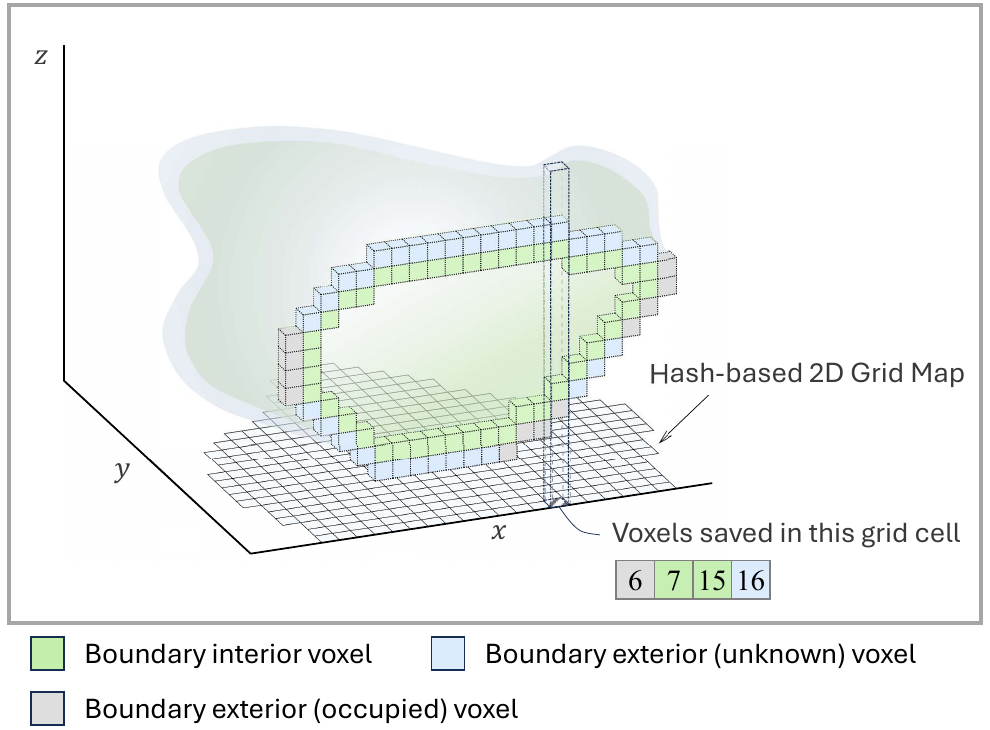}

    \caption{The boundary map, which consists of boundary voxels in the three-dimension (3D), is maintained in a two-dimensional (2D) data structure. Without loss of generality, \(z\)-axis is selected as the projection axis for illustration purpose in the manuscript. The boundary map is projected along the \(z\)-axis onto the \((x, y)\)-plane. Then a hash-based 2D grid map is maintained on the \((x, y)\)-plane where each grid cell stores boundary voxels projected onto that cell. Each boundary voxel in the grid cell is represented by its \(z\)-coordinate and a type indicator indicating its type (i.e., \(\mathbf{b}_{\mathrm{int}}\), \(\mathbf{b}_{\mathrm{ukn}}\), and \(\mathbf{b}_{\mathrm{occ}}\)). {The numbers in the figure correspond to the \(z\)-coordinates of the boundary voxels.}}

    \label{fig:hash_2d_grid}
\end{figure}

We begin by introducing the concepts of the projection axis and the projection plane. One of the \(x\), \(y\), or \(z\) axes is designated as the projection axis. The plane orthogonal to this axis is referred to as the projection plane, defined by the remaining two axes in their natural sequential order. To facilitate explanation, we use \(z\) as the projection axis throughout the remainder of this manuscript, and accordingly, the projection plane corresponds to the \((x, y)\)-plane.

We project all boundary voxels along the \(z\)-axis onto the \((x, y)\)-plane. Then we maintain a 2D grid map on this plane, where each grid cell stores boundary voxels projected onto it. The boundary voxels stored in a 2D grid cell {are} organized by an array and are sorted based on their \( z \)-coordinates. In other words, a grid cell \(\mathbf{c} = (c^x, c^y) \in \mathbb{Z}^2\) of the 2D grid map maintains all boundary voxels that share the same \((x, y)\)-coordinates as \((c^x, c^y)\).

The 2D grid map is implemented based on a hash table. Each grid cell \(\mathbf{c} = (c^x, c^y)\) is indexed by a hash key computed as:
\begin{equation}
\texttt{Hash}(\mathbf{c}) = (\texttt{P} \cdot c^x + c^y) \bmod \texttt{Q},
\end{equation}
where \(\texttt{P}\) is a prime number chosen as 1441 for improved hash distribution, and \(\texttt{Q}\) is the hash table size.


For the boundary voxel saved in the 2D grid cell, it suffices to store its \(z\)-coordinate and an indicator of its type. Each boundary voxel belongs to one of three types: boundary interior voxel \( \mathbf{b}_\mathrm{int} \), boundary exterior (unknown) voxel \( \mathbf{b}_{\mathrm{ukn}} \) or boundary exterior (occupied) voxel \( \mathbf{b}_{\mathrm{occ}} \). We represent each boundary voxel using a {32-bit integer (\(\texttt{Int32}\))}: the lower 30 bits encode the \(z\)-coordinate, and the upper 2 bits are sufficient to encode the boundary voxel type.

{Unlike conventional occupancy grid maps that maintain the entire three-dimensional (3D) volume, the boundary map stores only two-dimensional (2D) boundary voxels, leading to a significantly reduced voxel count. This low-dimensional representation also enables the boundary map to be maintained by a 2D grid. Moreover, each boundary voxel requires only 4 bytes of storage, which further enhances the memory efficiency.} 

\subsection{Implementation of Occupancy State Determination}
\label{sec:determine_impl}

To determine the occupancy state of a query voxel \(\mathbf{q} = (q^x, q^y, q^z) \in \mathbb{Z}^3\), we first lookup the hash table for the 2D grid cell \(\mathbf{c}_{{q}} = (q^x, q^y) \in \mathbb{Z}^2\), which having the same \((x, y)\)-coordinates with the query voxel \(\mathbf{q}\). 
{If the 2D grid cell \(\mathbf{c}_{{q}}\) does not exist in the hash table, it implies that this grid cell \(\mathbf{c}_{{q}}\) contains no boundary voxels, which indicates \(\mathbf{b}_{{nn}}\) does not exist. In this case, we assign a \texttt{null} flag to \(\mathbf{b}_{{nn}}\). On the other hand, if the grid cell \(\mathbf{c}_{{q}}\) is found in the hash table, we proceed by retrieving the array stored in the grid cell \(\mathbf{c}_{{q}}\).} This array contains all boundary voxels in the environment whose \((x, y)\)-coordinates equal \((q^x, q^y)\). Given that the \(z\)-axis is selected as the projection axis, the search direction \(\mathcal{E}\) can be chosen as either \(z^+\) or \(z^-\). Since the voxels in each array are pre-sorted along the \(z\)-axis, the \(\mathbf{b}_{{nn}}\) can be located efficiently using binary search. Specifically, if the search direction is \(z^+\), \(\mathbf{b}_{{nn}}\) is the first voxel in the array whose \(z\)-coordinate is greater than or equal to \(q^z\). Conversely, if the search direction is \(z^-\), \(\mathbf{b}_{{nn}}\) is the first voxel whose \(z\)-coordinate is less than or equal to \(q^z\). {When search direction is $z^+$, the \(\mathbf{b}_{{nn}}\) does not exist if the $z$-coordinates of all boundary voxels saved in this grid cell \(\mathbf{c}_{{q}}\) are less than $q^z$. On the other hand, when the search direction is $z^-$, this is the case when the $z$-coordinates of all boundary voxels saved in \(\mathbf{c}_{{q}}\) are greater than $q^z$. In both cases, we assign a \texttt{null} flag to \(\mathbf{b}_{{nn}}\).}
This procedure corresponds to the implementation of the \texttt{FindNearestInDirection} function in Algorithm~\ref{alg:occupancy_state} (Line 1). Once \(\mathbf{b}_{{nn}}\) is identified, the occupancy state of the query voxel \(\mathbf{q}\) can be determined by following the remaining steps outlined in Algorithm~\ref{alg:occupancy_state} (Line 2--{12}).

In occupancy mapping, determining the occupancy state of a given location in the environment is commonly referred to as a query operation. Accordingly, we refer to the above algorithm as the query algorithm in our method.

\subsection{Time Complexity Analysis}
\label{sec:determine_complx}
The time complexity for the query algorithm consists of two components. (i) Retrieving the boundary voxels in the 2D grid cell through a hash table lookup, which requires \(\mathcal{O}(1)\) time on average. (ii) Locating the voxel \(\mathbf{b}_{{nn}}\) involves a binary search on all boundary voxels saved in the 2D grid cell. The time complexity of this second step depends on the number of such voxels, while this number varies for different query voxel's locations. {Therefore, we derive the average time complexity considering the occupancy state queries across the entire map, as provided below.}

\begin{theorem}
\label{thm:time_complx}
The average time complexity for performing a map query is \(\mathcal{O}(1)\).
\end{theorem}

\begin{proof}
See Appendix~\ref{app:proof_time_complx}.
\end{proof}

{Note that the 2D grid map is also implemented by a hash table. In the worst case, query performance can degrade from $\mathcal{O}(1)$ to $\mathcal{O}(n)$ due to hash collisions. However, compared with the hash-based 3D grid map, our approach maintains only a 2D grid, which results in a significantly smaller hash table and therefore fewer collisions.} 

\section{Global-local Mapping Framework}
\label{sec:framework}

We propose a global-local mapping framework designed to support updates of the boundary map from the real-time sensor measurements. The framework integrates a robot-centric local map with the global map. {Specifically, the local map is a fixed-size occupancy grid map centered on the vehicle’s current position, maintaining occupancy information around the robot. The global map is a boundary map, maintaining occupancy information outside the local map region. The detailed map structure is described in Section~\ref{sec:map_structure}.} Based on this global-local structure, we introduce the corresponding update method in Section~\ref{sec:update} and the region-based query strategy in Section~\ref{sec:query}. An overview of the proposed mapping framework is illustrated in Figure~\ref{fig:overview}.

\begin{figure*}[t]
    \centering
    \includegraphics[width=\textwidth]{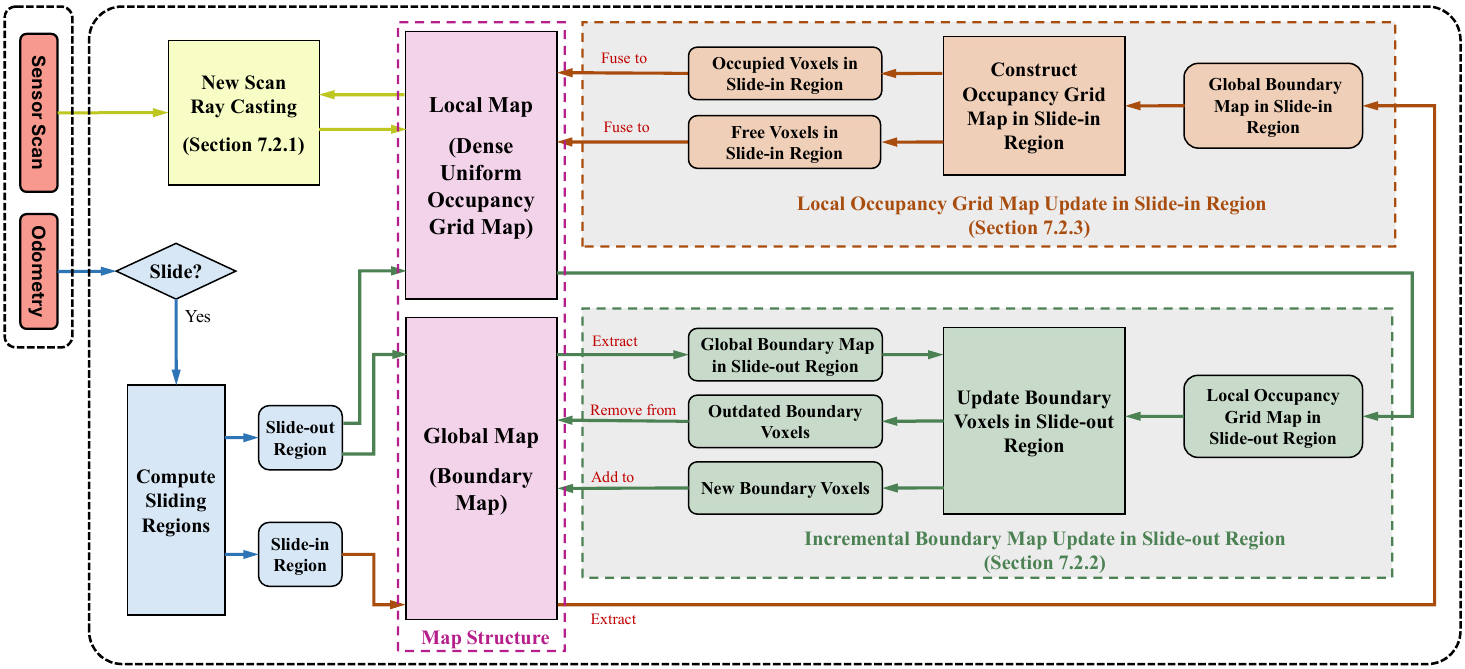}
    \vspace{-10pt}

    \caption{Overview of the proposed global-local mapping framework. The system takes odometry and sensor scans as inputs. The pink block illustrates the map structure, consisting of a global map (the low-dimensional boundary map) and a local map (a dense uniform occupancy grid map). The mapping framework is updated once every new odometry and sensor data arrive. For each point in the new scan, the \textbf{New Scan Ray Casting} performs ray casting and probabilistic occupancy update on the local occupancy grid map. The new odometry is used to determine if the local map needs to slide and, if so, compute the slide-in region \(\mathcal{S_I}\) and slide-out region \(\mathcal{S_O}\) of the local map. If a slide occurs, the occupancy state within the slide-out region of the local map is updated to the global map by the  \textbf{Incremental Boundary Map Update} (green block), while the occupancy state within the slide-in region of the local map is loaded from the global map via the \textbf{Local Occupancy Grid Map Update} (brown block).}

    \label{fig:overview}
\end{figure*}

\subsection{Map Structure}
\label{sec:map_structure}

{Similar to ROG-Map~\cite{ren2023rog}, the local map is implemented as a uniform occupancy grid map, which centers on vehicle's current position and dynamically moves with the vehicle.}
The map size is fixed and is configured based on the sensor's sensing range, ensuring full coverage of the sensing region. Each voxel in grid stores a floating-point value representing its log-odds occupancy probability. The grid is stored in memory as a one-dimensional array. We adopt a similar index mapping scheme used in ROG-Map to convert 3D voxel indices to unique addresses in this array. This mapping scheme ensures that any voxel within the local map is consistently assigned a fixed address, regardless of the local map's sliding, thereby enabling a zero-copy behavior.

The global map complements the local map by maintaining the occupancy information outside the robot-centric region. 
The detailed data structure is described in Section~\ref{sec:data_struct}.

\subsection{Map Update}
\label{sec:update}

\begin{figure*}[t] 
    \centering
    \includegraphics[width=\textwidth]{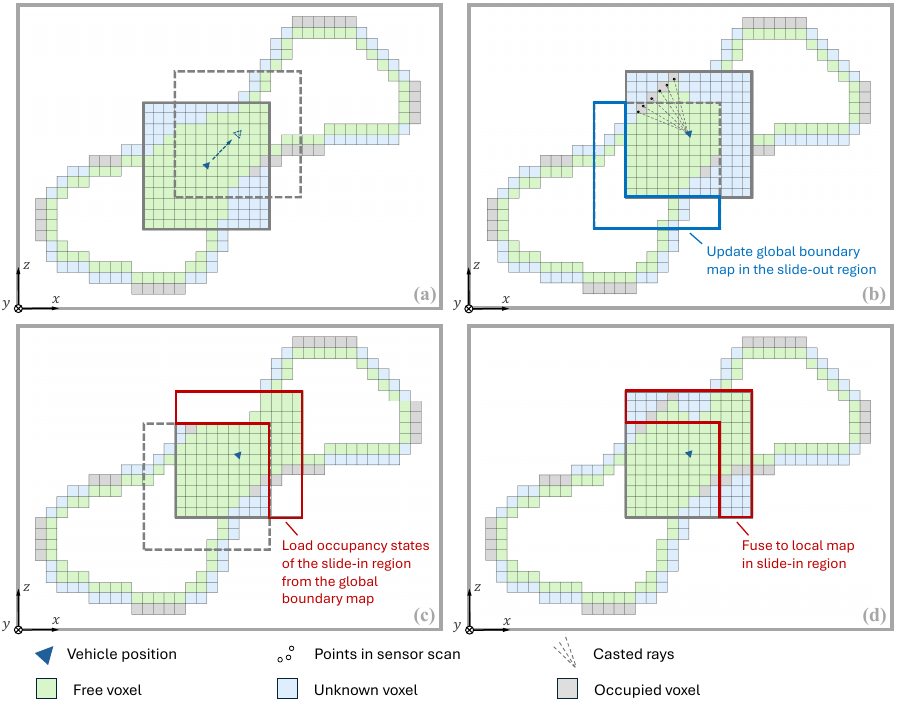}
    \vspace{-10pt} 

    \caption{Illustration of the local and global map update if a slide occurs in Figure~\ref{fig:overview}. (a) The global map (boundary map) and local map (a uniform occupancy grid map) before sliding, and the local map region before and after the sliding. (b) The local map slides to the new position and updates the occupancy state within the map region (including both slide-in region and the overlapped region) by ray casting all points in the new LiDAR scan (\textbf{New Scan Ray Casting}). Meanwhile, boundary voxels in the slide-out region are computed and updated to the global boundary map (\textbf{Incremental Boundary Map Update}). (c) From the global boundary map, occupancy states within the slide-in region are load{ed}. (d) The loaded occupancy states within the slide-in region are fused into the local map (\textbf{Local Occupancy Grid Map Update}). }

    \label{fig:update}
\end{figure*}

When a new sensor scan is received, along with the vehicle's current position, the framework is updated accordingly. The occupancy information in the local map is updated by the new scan via the ray casting technique (see Section~\ref{sec:new_scan}). The received vehicle's position determines whether the local map needs to slide. If the vehicle moves beyond a certain threshold distance from the current local map center, a local map sliding is triggered. 
For the region slides out from the local map, the occupancy information is represented by boundary voxels and incrementally updated into the global map (see Section~\ref{sec:update_global}). For the region slides into the local map, the occupancy states are loaded from the global map and fused into the local map (see Section~\ref{sec:restore}). We denote the slide-out and slide-in regions as \(\mathcal{S_O}\) and \(\mathcal{S_I}\), respectively. 
Notably, the ray casting process and the local map sliding are not synchronously triggered. When a sensor scan arrives, the latest measurements is updated into the local map, but this updated information is not immediately synchronized into the global map. Instead, we store this updated information in a difference logger \(\mathcal{L}\), {which records the difference in occupancy information between the local map and global map.} Specifically, \(\mathcal{L}\) contains local map  voxels with updated occupancy state. Then, when a local map sliding occurs, the difference logger \(\mathcal{L}\) is used to incrementally update the global map in the slide-out region \(\mathcal{S_O}\), synchronizing the occupancy information of global map in \(\mathcal{S_O}\) with that in local map. An illustration of the full update process of this mapping framework is provided in Figure~\ref{fig:update} and described in Algorithm~\ref{alg:map_update}.

\begin{algorithm}[ht]
\caption{Map Update}
\label{alg:map_update}
\small
\SetAlgoLined
\KwIn{Vehicle position $\mathbf{x}_{k}$, Point cloud scan $\mathcal{C}_k$
}

{\textbf{Notation:} Current local map origin \textbf{o}}

\BlankLine
\textbf{Algorithm Start}
\BlankLine

$\mathtt{NewScanRayCasting}(\mathbf{x}_{k}, \mathcal{C}_k)$\;

\If{$\| \mathbf{x}_k - \mathbf{o}\| > t$}{
$\mathcal{S_I}, \mathcal{S_O} \gets \mathtt{ComputeSlidingRegions}(\mathbf{o}, \mathbf{x}_{k})$\;
$\mathtt{IncrementalBoundaryMapUpdate}(\mathcal{S_O})$\;
$\mathtt{LocalGridMapUpdate}(\mathcal{S_I})$\;
$\mathbf{o} \gets \mathbf{x}_k$;
}

\BlankLine
\textbf{Algorithm End}
\BlankLine

\SetKwFunction{FMain}{IncrementalBoundaryMapUpdate}
\SetKwProg{Fn}{Function}{:}{}
\Fn{\FMain{$\mathcal{S_O}$}}{
    $\mathcal{L} \gets \mathtt{GetUpdatedLocalMapVoxels}()$\;
    $\mathcal{U} \gets \varnothing$\;
    \ForEach{voxel $\mathbf{n}_{v} \in \mathcal{L}$}{
        \If{$\mathbf{n}_{v} \in \mathcal{S_O}$}{
            {$\mathbf{E} \gets \{\mathbf{n}_{v}\} \cup$ 6-neighbors of $\mathbf{n}_{v}$}\;
            
\ForEach{{voxel $\mathbf{e} \in \mathbf{E}$}}{
    \If{{$\mathbf{e} \notin \mathcal{U}$}}{
        {$\mathcal{U} \gets \mathcal{U} \cup \{\mathbf{e}\}$}\;
    }
}
            $\mathcal{L} \gets \mathcal{L} \setminus \{\mathbf{n}_{v}\}$\;
        }
    }

    \tcp{\text{Remove outdated boundary voxels}}
    {$\mathcal{S}_{\mathcal{O}}^{2d} \gets \mathtt{Get2DProjectionArea}(\mathcal{S_O}, z)$}\;
    \ForEach{2D grid cell $\mathbf{c}_{g} \in\mathcal{S}_{\mathcal{O}}^{2d}$}{
    {$\mathtt{h}_c \gets \mathtt{GenerateHashKey}(\mathbf{c}_{g})$}\;
        {$\mathbf{B}_c \gets \mathtt{RetrieveValue}(\mathtt{h}_c)$}\;
        \ForEach{voxel $\mathbf{b}_c \in \mathbf{B}_c$}{
            \If{{$\mathbf{b}_c \in \mathcal{U}$}}{
                {$\mathtt{RemoveFrom}(\mathbf{b}_c, \mathbf{B}_c)$}\;
            }
        }
        }
    \tcp{\text{Add new boundary voxels}}
    \ForEach{voxel $\mathbf{n}_{s} \in \mathcal{U}$}{
        $\mathtt{ComputeBoundaryVoxelStatus}(\mathbf{n}_{s})$\;
        \If{$\mathtt{isBoundaryVoxel}(\mathbf{n}_{s})$}{
            {$\mathbf{c}_s \gets \mathtt{Get2DGridCell}(\mathbf{n}_{s})$}\;
            {$\mathtt{h}_s \gets \mathtt{GenerateHashKey}(\mathbf{c}_{s})$}\;
            {$\mathbf{B}_s \gets \mathtt{RetrieveValue}(\mathtt{h}_s)$}\;
            {$\mathtt{AddTo}(\mathbf{n}_s, \mathbf{B}_s)$}\;
        }
    }

    $\mathtt{Sort}(\mathcal{S}_{\mathcal{O}}^{2d})$\;
}
\textbf{End Function}

\SetKwFunction{FRestore}{LocalGridMapUpdate}
\Fn{\FRestore{$\mathcal{S_I}$}}{
    {$\mathcal{S}_{\mathcal{I}}^{2d} \gets \mathtt{Get2DProjectionArea}(\mathcal{S_I}, z)$}\;
    $\mathcal{G} \gets \varnothing$\;
    \ForEach{2D grid cell $\mathbf{c}_a \in \mathcal{S}_{\mathcal{I}}^{2d}$}{
        {$\mathtt{h}_a \gets \mathtt{GenerateHashKey}(\mathbf{c}_a)$}\;
        {$\mathbf{B}_a \gets \mathtt{RetrieveValue}(\mathtt{h}_a)$}\;
        {$\mathbf{B}_d \gets \mathtt{IdentifyBoundaryVoxels}(\mathbf{B}_a, \mathcal{S_I})$}\;
        {$\mathcal{G} \gets \mathcal{G} \cup \mathtt{ExtractFromBoundary}(\mathbf{B}_d)$}\;
        $\mathcal{G} \gets \mathcal{G} \cup \mathtt{ConstructFreeVoxels}(\mathbf{B}_d)$\;
    }

    $\mathtt{Fuse}(\mathcal{G})$\;
}
\textbf{End Function}
\end{algorithm}

\subsubsection{New Scan Ray Casting}
\label{sec:new_scan}
Similar to most occupancy grid mapping methods~\cite{moravec1996robot, hornung2013octomap, duberg2020ufomap, ren2023rog}, when a sensor scan arrives, we {probabilistically} integrate the measurement of the received sensor scan into the local map using a ray casting technique~\cite{amanatides1987fast}. We denote the received sensor scan as \(\mathcal{C}_k\). For each point \( \textbf{p} \in \mathbb{R}^3 \) in the scan \(\mathcal{C}_k\), a ray is cast from the sensor origin to the point \( \textbf{p} \) using the 3D Digital Differential Analyzer (3D-DDA) algorithm. A voxel is considered a \texttt{hit} if the point \( \textbf{p} \) lies within it, and a \texttt{miss} if the ray passes through the voxel.

We denote the log-odds occupancy probability value of a voxel \( \mathbf{n} \) in the local map as \(\texttt{L}(\mathbf{n})\). This value is updated by the sensor scan incrementally as follows:

\begin{equation}
    \texttt{L}(\mathbf{n}) = \texttt{L}(\mathbf{n}) + \delta_\text{n}
\label{eqa:prob_g}
\end{equation}
where \( \delta_\text{n} \) is computed as follows:

\begin{equation}
    \delta_\text{n} = n_{\text{hit}} \cdot l_{\text{hit}} + n_{\text{miss}} \cdot l_{\text{miss}} 
\label{eqa:prob_cnt}    
\end{equation}
where $n_{\text{hit}}$ and $n_{\text{miss}}$ represent the number of times a voxel being \texttt{hit} and \texttt{miss}, respectively, and \( l_{\text{hit}} \) and \( l_{\text{miss}} \) are the corresponding log-odds probability values for \texttt{hit} and \texttt{miss}.

We also adopt a clamping policy which is employed by numerous occupancy grid maps~\cite{hornung2013octomap, duberg2020ufomap, ren2023rog} to robustly handle dynamic environments. Specifically, the log-odds occupancy probability value of each voxel is constrained between a lower bound \(l_{\text{min}}\), and an upper bound \(l_{\text{max}}\). This constraint enables the map to adapt rapidly to environmental changes. This clamping policy is written as:

\begin{equation}
    \texttt{L}(\mathbf{n}) = \max\left(\min\left(\texttt{L}(\mathbf{n}), l_{\text{max}}\right), l_{\text{min}}\right)
\label{eqa:prob_clmap}
\end{equation}

The discrete occupancy state (i.e., free, occupied, unknown) is computed from the floating-value occupancy probability by thresholding:

\begin{equation}
\texttt{occ}(\textbf{n}) = 
\begin{cases} 
free & \text{if } \texttt{L}(\textbf{n}) \leq l_{\text{free}} \\
occupied & \text{if } \texttt{L}(\textbf{n}) \geq l_{\text{occ}} \\
unknown & \text{otherwise}
\end{cases}
\label{eqa:state}
\end{equation}
where \(l_{\text{free}}\) and \(l_{\text{occ}}\) are refer{red} to as the log-odds threshold values for the free and occupied states, respectively.

{
The difference logger \(\mathcal{L}\) is updated at each ray casting process. Specifically, for each voxel modified during the ray casting, we compare its current occupancy state to its state before the ray casting. Voxels whose occupancy state changes are added to \(\mathcal{L}\). }

\subsubsection{Incremental Boundary Map Update}
\label{sec:update_global}
{In this section, we describe the incremental update process for the global boundary map in the slide-out region \(\mathcal{S_O}\). This procedure is illustrated in green blocks in Figure~\ref{fig:overview} and detailed in Algorithm~\ref{alg:map_update} (see function \texttt{IncrementalBoundaryMapUpdate}, Lines {11--45}). }

{First, we evaluate the boundary voxel status (see Equation~\ref{eqa:boundary_def}) of the global map voxels located in \(\mathcal{S_O}\). If a boundary voxel's status has changed, it is marked as outdated and removed from the global map. Otherwise, it is retained. The detailed procedure is as follows. The difference logger \(\mathcal{L}\) maintains local map voxels whose occupancy states have changed during updates. We then identify those voxels in \(\mathcal{L}\) that lie within the slide-out region \(\mathcal{S_O}\). For each identified voxel, we compute its six neighbors and add both the voxel and its neighbors into a container \(\mathcal{U}\), {if they are not already present in \(\mathcal{U}\) (Lines 16--21).} After processing, such identified voxel is removed from \(\mathcal{L}\) (Line {22}). Then, the slide-out region \( \mathcal{S_O} \) is projected onto the \((x, y)\) plane along the projection axis \( z \), resulting in a 2D area denoted as \( \mathcal{S}_\mathcal{O}^{2d} \) (Line {25}). Note that in the previous section, the \(z\)-axis is chosen as the projection axis; thus, we maintain this setting here. 

Next, we iterate over each 2D grid cell {$\mathbf{c}_g = (g^x, g^y) \in \mathbb{Z}^2$} within the area $\mathcal{S}_\mathcal{O}^{2d}$.
{For each cell $\mathbf{c}_g$, we generate its hash key $\mathtt{h}_c$, namely \texttt{GenerateHashKey} (Line~27). We then index the 2D grid map (i.e., the global boundary map) and retrieve the value at key $\mathtt{h}_c$, which corresponds to an array of boundary voxels, denoted as $\mathbf{B}_c$. This operation is referred to as \texttt{RetrieveValue} (Line~28).
Subsequently, we examine each boundary voxel $\mathbf{b}_c \in \mathbf{B}_c$.
If the boundary voxel $\mathbf{b}_c$ is not contained in $\mathcal{U}$, its {boundary voxel status} remains unchanged, as formally proven in Appendix~\ref{app:proof_upd}.
Otherwise, its {boundary voxel status} may change, and we remove it from $\mathbf{B}_c$, namely the \texttt{RemoveFrom} operation (Line~31).}

{Subsequently, for each voxel \(\mathbf{n}_s = (s^x, s^y, s^z) \in \mathbb{Z}^3\) in \(\mathcal{U}\), we compute its new boundary voxel status (Line 36).
If \(\mathbf{n}_s\) is classified as a boundary voxel (Line 37), we first determine its corresponding 2D grid cell on the \((x, y)\)-plane, which is \(\mathbf{c}_s = (s^x, s^y) \in \mathbb{Z}^2\). This operation is referred to as \texttt{Get2DGridCell} (Line 38). Then, we generate the hash key of cell \(\mathbf{c}_s\), denoted as \(\mathtt{h}_s\) (Line 39). We then retrieve the value at key \(\mathtt{h}_s\), which corresponds to an array of boundary voxels, denoted as \(\mathbf{B}_s\) (Lines 40). Finally, we add \(\mathbf{n}_s\) to this array \(\mathbf{B}_s\), namely \texttt{AddTo} operation (Line 41).}

{Finally, to enable efficient occupancy state queries, the arrays of all updated 2D grid cells in global map within \(\mathcal{S}_\mathcal{O}^{2d}\) are sorted by the \(z\)-coordinates of the boundary voxels. This ensures that future occupancy state queries can leverage binary search for rapid access. This procedure is referred to as the \texttt{Sort} operation (Line {44}). {In addition, this operation can be further improved by sorting only modified columns in $\mathcal{S}^{2d}_\mathcal{O}$, rather than all columns.}}

\subsubsection{Local Occupancy Grid Map Update}
\label{sec:restore}
\begin{figure*}[t] 
    \centering
    \includegraphics[width=\textwidth]{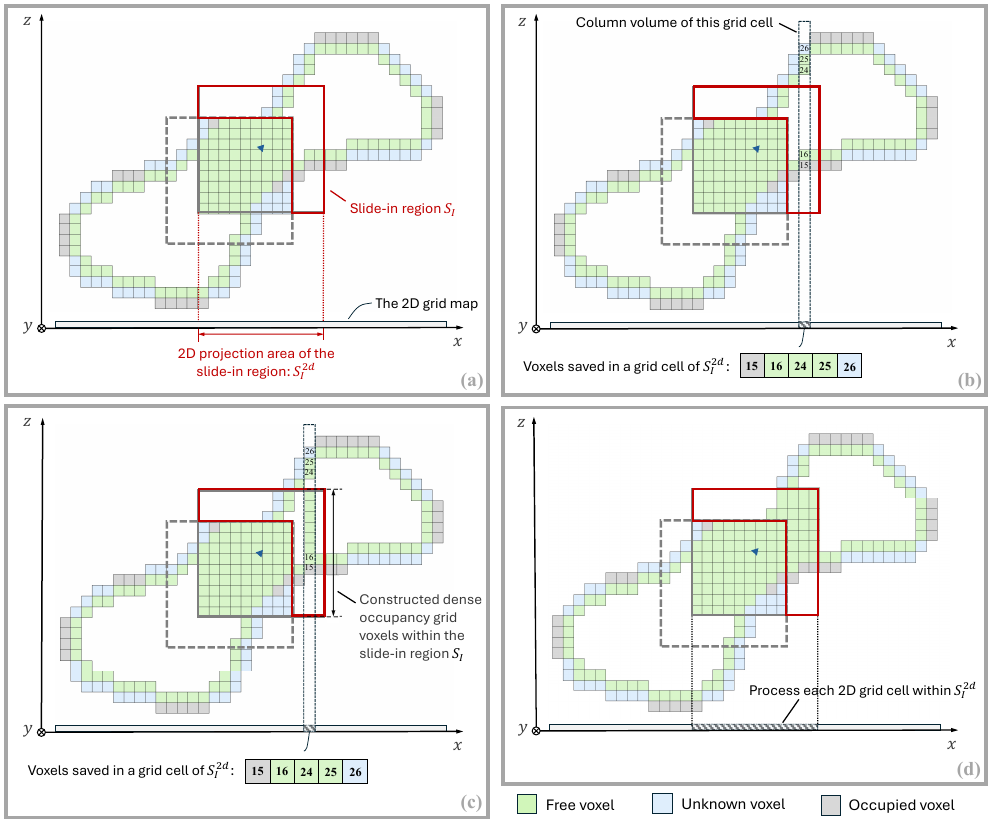}
    
    \caption{
    Illustration of step (c) in Figure~\ref{fig:update}: loading occupancy states within the slide-in region from the global boundary map. (a) Projecting the slide-in region \(\mathcal{S_I}\) to the \((x, y)\)-plane, denoted as \(\mathcal{S}_\mathcal{I}^{2d}\). (b) For each 2D grid cell within \(\mathcal{S}_\mathcal{I}^{2d}\), retrieve the boundary voxels saved in the global boundary map. (c) Constructing a dense occupancy grid tile consisting of free voxels and occupied voxels above the 2D grid cell but within the slide-in region. Voxels not encoded in the tile are unknown voxels requiring no processing. (d) This procedure is repeated for all 2D grid cells within \(\mathcal{S}_\mathcal{I}^{2d}\), resulting in a dense occupancy grid map within the slide-in region \(\mathcal{S_I}\). This map contains only free and occupied voxels and is finally fused to the local map in step (d) in Figure~\ref{fig:update}.
    }
    \label{fig:construct}
\end{figure*}

In this section, we introduce the procedure for the local map update process, which constructs the dense occupancy grid voxels from the boundary map in the slide-in region \( \mathcal{S_I} \) and then fuses them into the local map. 

A straightforward approach would be to query the global map for the occupancy state of every voxel within \( \mathcal{S_I} \). However, this requires traversing all voxels in \( \mathcal{S_I} \), including the unknown ones. Notably, the unknown voxels constitute a substantial portion of \( \mathcal{S_I} \), contributing to considerable computational overhead. To improve efficiency, we propose a construction method that only traverses the free and occupied voxels in \( \mathcal{S_I} \). {This construction process is illustrated in brown blocks in Figure~\ref{fig:overview}. The detailed procedure is provided below and outlined in Algorithm~\ref{alg:map_update} (see function \texttt{LocalGridMapUpdate}, Line {46--57}).}

Initially, the slide-in region \( \mathcal{S_I} \) is projected along the \( z \)-axis onto the \( (x, y) \) plane, resulting in a two-dimensional projected area denoted as \( \mathcal{S}_\mathcal{I}^{2d} \). This process is namely {\texttt{Get2DProjectionArea}} function (Line {47}). An illustration of this process is shown in Figure~\ref{fig:construct}(a).
We define the column volume of a grid cell \(\mathbf{c}_a\) in the area \( \mathcal{S}_\mathcal{I}^{2d} \), as the vertical region formed by voxels that share the same \((x, y)\)-coordinates as \(\mathbf{c}_a\), and whose \(z\)-coordinates span the full height of the environment. This column volume of \(\mathbf{c}_a\) is denoted as \(\mathcal{V}_a\). An illustration of the column volume is provided in Figure~\ref{fig:construct}(b). We first focus on constructing dense occupancy grid voxels—both free and occupied—within the overlapped region between \(\mathcal{V}_a\) and \(\mathcal{S_I}\), denoted as \(\mathcal{V}_a \cap \mathcal{S_I}\). 

{In the following, we describe the detailed procedures for constructing dense occupancy grid voxels in the region \(\mathcal{V}_a \cap \mathcal{S_I}\). For each 2D grid cell \(\mathbf{c}_a = (a^x, a^y) \in \mathcal{S}_{\mathcal{I}}^{2d}\), we first retrieve all boundary voxels that stored in the cell, and denoted them as \(\mathbf{B}_a\) (Line 50--51).
Next, we filter the boundary voxel in \(\mathbf{B}_a\)} that located within the slide-in region \(\mathcal{S_I}\), and denoted them as \(\mathbf{B}_d\). This process is referred to as \texttt{IdentifyBoundaryVoxels} (Line {52}). As occupancy state of the boundary voxel can be directly determined from its definition (see Equation~\ref{eqa:boundary_def}), we first filter the boundary voxels classified as either free or occupied and add them to a container \(\mathcal{G}\). This step is referred to as \texttt{ExtractFromBoundary} function (Line {53}). Since all occupied voxels are explicitly stored in the boundary map, this step ensures that all occupied voxels in region \(\mathcal{V}_a \cap \mathcal{S_I}\) are identified and added to \( \mathcal{G} \). In other words, occupied voxels do not require further construction.

Subsequently, we focus on constructing the free voxels within the region \(\mathcal{V}_a \cap \mathcal{S_I}\) that are not explicitly stored in the boundary map. This procedure, referred to as \texttt{ConstructFreeVoxels} (Line {54}), is described in detail below.

First, we discuss the case where there is no boundary voxel exists in region \(\mathcal{V}_a \cap \mathcal{S_I}\) (i.e., \(\mathbf{B}_d\) is an empty set). This case arises either when \(\mathbf{c}_a\) is not exist in the global map, or when none of the boundary voxels stored in \(\mathbf{c}_a\) lie within the slide-in region \(\mathcal{S_I}\). We introduce the following theorem:

\begin{theorem} 
\label{thm:restore_null}

If \(\mathbf{B}_d\) is an empty set, all voxels in region \(\mathcal{V}_a \cap \mathcal{S_I}\) have the same occupancy state.
\end{theorem}

\begin{proof}
see Appendix~\ref{app:proof_restore_null}.
\end{proof}

Based on the above theorem, to determine the occupancy state of voxels in region \(\mathcal{V}_a \cap \mathcal{S_I}\), we arbitrarily select a voxel within it as a representative and query its occupancy state. Since all occupied voxels reside exclusively on the boundary map, the queried voxel must be either free or unknown. The corresponding occupancy state is then propagated to all voxels in \(\mathcal{V}_a \cap \mathcal{S_I}\): if free, these voxels are added to \(\mathcal{G}\).

Next, we consider the case where \(\mathbf{B}_d\) is not an empty set.
We focus on the boundary interior voxels \(\mathbf{b}_\mathrm{int} \in \mathbf{B}_d\). For each such voxel, denoted as \(\mathbf{b}_i\), we perform a scan in both the \(z^+\) and \(z^-\) directions, marking each voxel along the vertical line until either (i) another boundary voxel is encountered, or (ii) the traversal exits the slide-in region \(\mathcal{S_I}\). This process yields two voxel tiles, denoted as \(\mathcal{T}^+_i\) and \(\mathcal{T}^-_i\), respectively. These tiles are then added to a voxel tiles set \(\mathbf{T}\).
Note that either \(\mathcal{T}^+_i\) or \(\mathcal{T}^-_i\) may be empty—this occurs when \(\mathbf{b}_i\) has an adjacent boundary voxel in the corresponding direction (i.e., \(z^+\) or \(z^-\)). An example is shown in Figure~\ref{fig:construct}(b) and (c), where the set \(\mathbf{B}_d\) contains two boundary voxels: a boundary exterior (occupied) voxel \(\mathbf{b}_\mathrm{occ}\) at \(z = 15\), and a boundary interior voxel \(\mathbf{b}_\mathrm{int}\) at \(z = 16\). In this case, \(\mathcal{T}^+_i\) for the \(\mathbf{b}_\mathrm{int}\) voxel at \(z = 16\) contains voxels with \(z\)-coordinates ranging from 17 to 22, since the voxel with \(z = 23\) falls outside the slide-in region \(\mathcal{S_I}\). Meanwhile, \(\mathcal{T}^-_i\) is empty as the \(\mathbf{b}_\mathrm{int}\) has an adjacent boundary voxel \(\mathbf{b}_\mathrm{occ}\) in \(z^-\) direction. After constructing the voxel tiles set \(\mathbf{T}\). The occupancy state of voxels in region \(\mathcal{V}_a \cap \mathcal{S_I}\) is determined by the following theorem.

\begin{theorem} 
\label{thm:restore_free}

All voxels encoded in the constructed tiles in \(\mathbf{T}\) are determined as free, while any remaining non-boundary voxels that are not included in \(\mathbf{T}\) are considered unknown.

\end{theorem}
\begin{proof}
see Appendix~\ref{app:proof_restore_RG}.
\end{proof}

Based on this theorem, all voxels included in \(\mathbf{T}\) are marked as free and added to \(\mathcal{G}\). For the example shown in Figure~\ref{fig:construct}(c), voxels with \(z\)-coordinates ranging from 17 to 22 are encoded in \(\mathbf{T}\), thus are marked as free. Recall that boundary voxels classified as either free or occupied have already been incorporated into \(\mathcal{G}\). {Thus, at this point, the construction process of the dense occupancy grid voxels for the region \(\mathcal{V}_a \cap \mathcal{S_I}\) is complete. }

This construction process is then repeated for every 2D grid cell in the area \(\mathcal{S}_\mathcal{I}^{2d}\), thereby constructing all {free} and {occupied} voxels within the slide-in region \(\mathcal{S}_\mathcal{I}\), which are added into \(\mathcal{G}\). An illustration is provided in Figure~\ref{fig:construct}(d). We also refer to \(\mathcal{G}\) as a constructed dense occupancy grid map in the slide-in region \(\mathcal{S_I}\), which is then fused into the local map. The fusing process is described in detail as follows.

For each voxel in \(\mathcal{G}\), we update the occupancy probability values of the corresponding voxel in the local map based on the constructed occupancy state (i.e., free or occupied), using the following rule:
\begin{equation}
\mathtt{L}(\mathbf{n}) = 
\begin{cases}
\mathtt{L}(\mathbf{n}) + l_{\text{free}}, & \text{if voxel is \textit{free}} \\
\mathtt{L}(\mathbf{n}) + l_{\text{occ}}, & \text{if voxel is \textit{occupied}}
\end{cases}
\label{eqa:log_odds_init}
\end{equation}
Here, \( l_{\text{free}} \) and \( l_{\text{occ}} \) denote the log-odds threshold values for the {free} and {occupied} states, respectively, as defined in Equation~\ref{eqa:state}. 
After this update, we apply a clamping process (see Equation~\ref{eqa:prob_clmap}) to ensure that the updated log-odds occupancy values remain within the range bounded by \(l_{\text{min}}\) and \(l_{\text{max}}\). This operation integrates the occupancy states constructed from the global boundary map into the local map, which we refer to as the \texttt{Fuse} function (see Line {56}). {Note that when a voxel slides out of the local map, its log-odds probability is converted into a discrete occupancy state and represented in the global boundary map. When the voxel re-enters the local map, the discrete occupancy state is converted back into a log-odds probability (see Equation~\ref{eqa:log_odds_init}). This conversion process introduces a minor degradation in accuracy compared with the baseline grid-based and octree-based methods. 
We further quantify and report this degradation in our benchmark experiments (see Section~\ref{sec:exp_acc}).
}

{Since the New Scan Ray Casting process integrates the most recent occupancy information from new sensor measurements, it is executed prior to the loading and fusing of occupancy information constructed from the global map in the slide-in region \(\mathcal{S}_\mathcal{I}\). Thus, the difference logger \(\mathcal{L}\) within the slide-in region \(\mathcal{S}_\mathcal{I}\) is also updated after the fusion is completed. Specifically,
we first obtain the voxel's occupancy state after the \texttt{Fuse} operation.
This state is then compared with the occupancy state that {is} constructed from global map, which is maintained by \(\mathcal{G}\). Voxels whose occupancy states differ are added to \(\mathcal{L}\).}

\subsection{Map Query}
\label{sec:query}

\begin{figure}[t] 
    \centering
    \includegraphics[width=\linewidth]{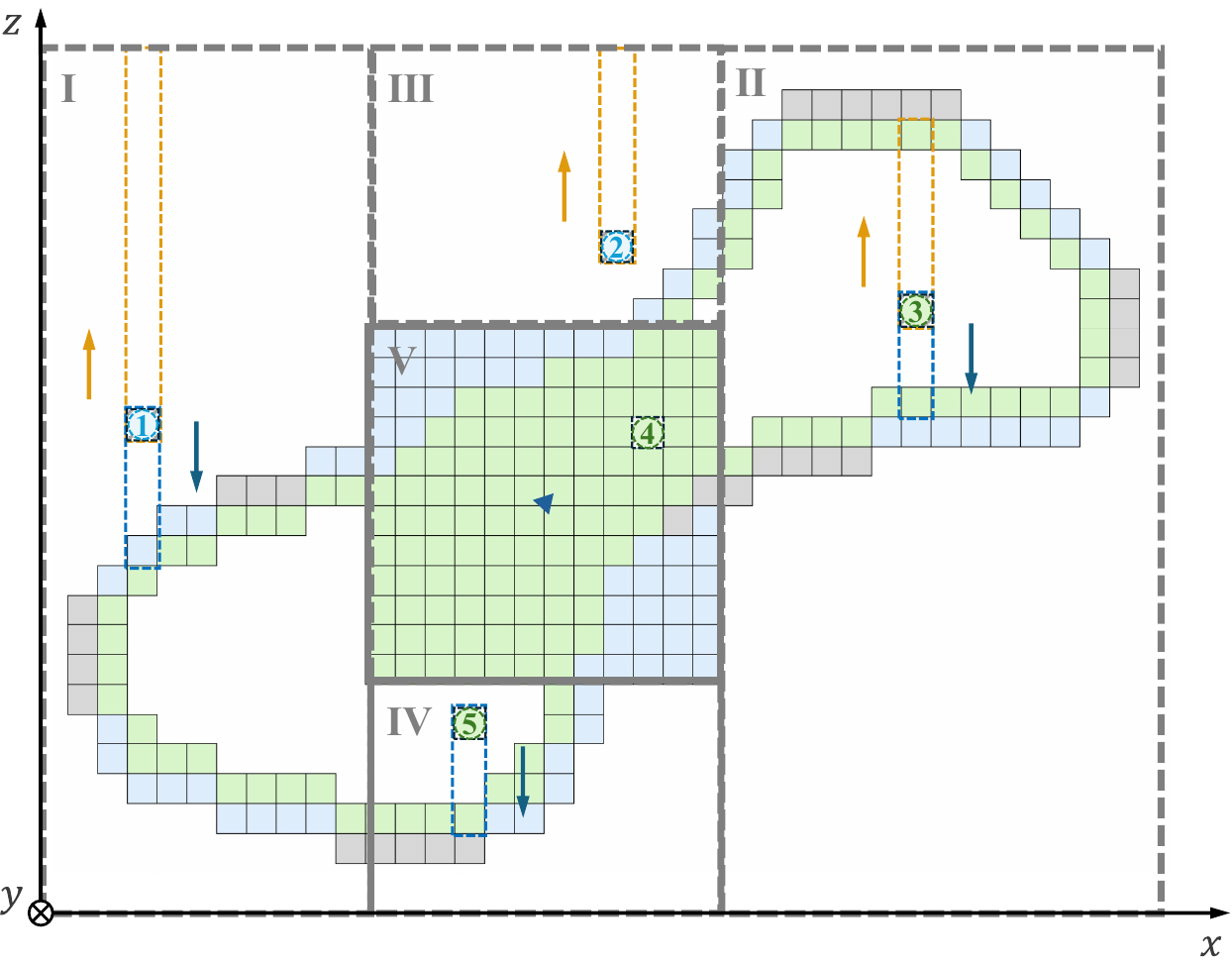} 
    \vspace{-10pt} 
    \caption{Illustration of the search direction selection in occupancy state querying in the proposed global-local mapping framework. The search direction is selected such that it does not intersect with the local map. 
    For query voxels in regions \MakeUppercase{\romannumeral 1} and \MakeUppercase{\romannumeral 2} (voxel no. 1 and 3), the search direction can be selected as \(z^+\) or \(z^-\), as neither direction intersects with the local map. For query voxels in regions \MakeUppercase{\romannumeral 3} and \MakeUppercase{\romannumeral 4} (voxel no. 2 and 5), only one direction can be chosen, which is the one not intersecting with the local map. For query voxels in region \MakeUppercase{\romannumeral 5} (voxel no. 4), which lie within the local map, the occupancy state is directly retrieved from the local map.}
    \label{fig:framework_query}
\end{figure}

Based on this global-local map structure, to query the occupancy state of a query voxel \( \mathbf{q} = (q^x, q^y, q^z) \in \mathbb{Z}^3 \) in the environment, we follow the procedures outlined below.

First, we check whether \( \mathbf{q} \) lies within the local map. If so, its occupancy state is directly obtained.

If \( \mathbf{q} \) lies outside the local map, its occupancy state is determined using the global boundary map. {In this case, the procedure of determination of the occupancy state is outlined in Algorithm~\ref{alg:occupancy_state}, which takes both the query voxel \( \mathbf{q} \) and the search direction \(\mathcal{E}\) as inputs. 
The search direction can be arbitrarily selected from any of the six directions \(\{x^+, x^-, y^+, y^-, z^+, z^-\}\), when the boundary map is spatially complete.
However, under the global-local mapping framework, where the global boundary map only maintains occupancy information outside the local map region, the selection of the search direction \(\mathcal{E}\) requires additional consideration. Specifically, the search direction must be chosen such that it does not intersect the local map. Furthermore, given the data structure of the boundary map and the \(z\)-axis is designated as the projection axis, selecting either the \(z^+\) or \(z^-\) as the search direction enables efficient occupancy state queries. Consequently, the search direction is selected as either \(z^+\) or \(z^-\), depending on which direction avoids intersecting the local map.
Specifically, when the \(z\)-coordinate of the query voxel exceeds the local map's upper bound, the search direction is set to \(z^+\). Conversely, when the \(z\)-coordinate falls below the local map's lower bound, the search direction is set to \(z^-\). This ensures that the search direction avoids intersecting the local map. Once the search direction \(\mathcal{E}\) is specified, the occupancy state of \(\mathbf{q}\) is then determined by Algorithm~\ref{alg:occupancy_state}.

To summarize, this query strategy is region-based. For a query voxel within the local map, its occupancy state is directly obtained. For a query voxel in the global map region, the search direction is determined based on its position relative to the local map. An illustrative example of this query strategy is shown in Figure~\ref{fig:framework_query}.
}

\section{Benchmark Experiments}
\label{sec:benchmake}

\begin{figure*}[t] 
    \setlength{\abovecaptionskip}{0pt}   
    \setlength{\belowcaptionskip}{5pt}    
    \centering
    \includegraphics[width=\linewidth]{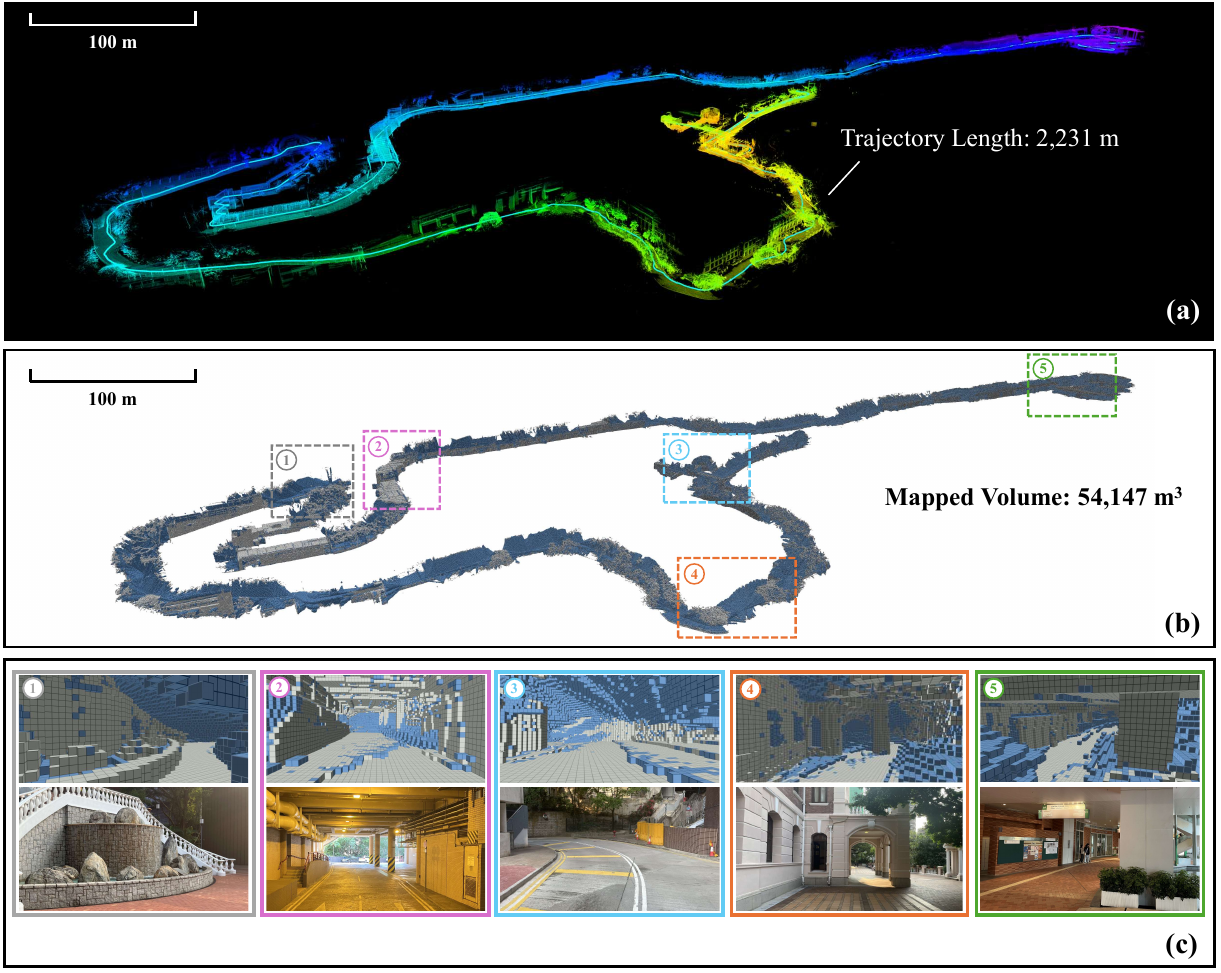}
    \vspace{-10pt}

    \caption{The HKU private dataset \textit{hku\_campus}. (a) The trajectory and accumulated point cloud scans. The mapping results of our method are shown in (b) and (c), focusing on the boundary map (i.e., the global map). The mapped volume is referred to the volume enclosed by the boundary. (b) provides an overview of the boundary map, while (c) presents interior views alongside the corresponding real-world image. For visualization purposes, only boundary exterior voxels are rendered, including boundary exterior (unknown) voxels (colored in blue) and boundary exterior (occupied) voxels (colored in grey). Boundary interior voxels, which are adjacent to the boundary exterior voxels, are omitted to improve visual clarity.}

    \label{fig:hku_campus}
\end{figure*}

\begin{table*}[h]
	\setlength{\tabcolsep}{5.2pt}
	\centering
	\caption{Details of datasets used in the benchmark experiments.}
	\label{tab:dataset}	
	\begin{threeparttable}
		\begin{tabular}{@{}cccccccc@{}}
			\toprule
			Sequence & Environment Scale & Traveled & Number & Average Points & Mapped & {Sensor} & {Local Map} \\
            & (Bounding Box) & Distance & of & Number & Volume & {Range} & {Size}\\
			& ($\mathrm{m}^3$) & ($\mathrm{m}$)  & Scans & (per scan)    &  ($\mathrm{m}^3$) & {($\mathrm{m}$)} & {($\mathrm{m}^3$)} \\ \midrule
			\textit{ford\_1}      & $\mathrm{11,416\times8,060\times94}$ & 16,639   & 7,757       & 41,108 & 3,843,031 & {65} & {$\mathrm{130\times130\times16}$} \\
			\textit{ford\_2}     & $11,584\times8,027\times102$ & 24,064  & 10,713       & 40,241 & 4,781,935   & {65} & {$\mathrm{130\times130\times16}$} \\
			\textit{ford\_3}     & $6,659\times3,860\times85$   & 9,558  & 8,692     & 42,501 & 2,905,112 & {65} & {$\mathrm{130\times130\times16}$} \\
			\textit{kitti\_00}   & $653\times586\times58$     & 3,724  & 4,541     & 121,495 & 243,903 & {45} & {$\mathrm{90\times90\times6}$}\\
			\textit{kitti\_02}    & $1,035\times688\times94$   & 5,067  & 4,661      & 125,628 & 328,737 & {45} & {$\mathrm{90\times90\times6}$} \\
			\textit{hku\_campus}    & $708\times254\times86$  & 2,231  & 15,854     & 6,136  & 54,147 & {10} & {$\mathrm{20\times20\times20}$} \\
			\textit{uav\_flight}    & $242\times182\times22$  & 502  & 3,307     & 5,369  & 17,863  & {10} & {$\mathrm{20\times20\times20}$}   \\ \bottomrule
		\end{tabular}
	\end{threeparttable}
 \vspace{-0.2cm}
\end{table*}

{
Extensive benchmark experiments were conducted to evaluate the performance of our mapping framework against several state-of-the-art methods, including Uniform Grid (UG) \cite{moravec1996robot}, Hash Grid (HG) \cite{niessner2013real}, Octomap (Octo) \cite{hornung2013octomap}, UFOMap (UFO) \cite{duberg2020ufomap}, and D-Map \cite{cai2023occupancy}.} Uniform Grid and Hash Grid are grid-based methods. Uniform Grid maintains a full 3D grid structure that covers the entire environment. It allocates a contiguous memory block (i.e., an array) to store all voxels within this grid. Hash Grid utilizes voxel hashing techniques to enhance memory efficiency. It maintains all voxels in the mapped volume in a hash table. Octomap and UFOMap are octree-based methods that represent the environment using a hierarchical octree structure. UFOMap extends Octomap with implementation-level enhancements, offering improved memory and computational efficiency. D-Map~\cite{cai2023occupancy} features a hybrid structure, maintaining unknown voxels in an octree and storing occupied voxels in a hash-based grid map. Unlike the other methods, D-Map eliminates ray casting during map updates. For Octomap, UFOMap, and D-Map, we used their open-source implementations available on GitHub repositories\footnote{https://github.com/OctoMap/octomap}\footnote{https://github.com/UnknownFreeOccupied/ufomap}\footnote{https://github.com/hku-mars/D-Map}. For Hash Grid, we employ the hash table implemented by the standard C++ library. {The datasets used in the benchmark experiments are detailed in Section~\ref{sec:dataset}, and the experimental setup is described in Section~\ref{sec:setup}. The performance of each method is evaluated in terms of memory consumption (see Section~\ref{sec:exp_mem}), update efficiency (see Section~\ref{sec:exp_upd}), query efficiency (see Section~\ref{sec:exp_query}), and map accuracy (see Section~\ref{sec:exp_acc}).}

\subsection{Datasets}
\label{sec:dataset}
We conducted experiments on two public datasets and two private datasets. The first public dataset is the Ford AV dataset~\cite{agarwal2020ford}, collected by Ford vehicles equipped with a 32-line rotating 3D laser scanner (Velodyne HDL-32E). From this dataset, three large-scale sequences: \textit{ford\_1}, \textit{ford\_2}, and \textit{ford\_3} were selected for evaluation. The second public dataset is the KITTI dataset \cite{geiger2013vision}, captured using a 64-line rotating 3D laser scanner (Velodyne HDL-64E). Two large-scale sequences \textit{kitti\_00} and \textit{kitti\_02} were selected for evaluation. 

The first private dataset was collected at The University of Hong Kong using a handheld device equipped with a semi-solid-state 3D LiDAR (Livox MID-360). The visualization of accumulated scans and the trajectory are presented in Figure~\ref{fig:hku_campus}(a). In this dataset, the odometry estimation is provided by FAST-LIO2~\cite{xu2022fast}. The second private dataset was collected by the onboard Livox MID-360 LiDAR installed on a MAV during a flight in an unconstructed outdoor field. A visualization of this dataset is provided in~\cite{ren2023rog}. More information for these datasets sequences, including the environment scale, travel distance, total number of scans, average points number per scan and mapped volume, is listed in Table~\ref{tab:dataset}. The mapped volume is referred to as the total volume of all known voxels (i.e., all free and occupied voxels). It also represents the volume enclosed by the boundary.

\begin{figure*}[t] 
    \centering
    \includegraphics[width=\textwidth]{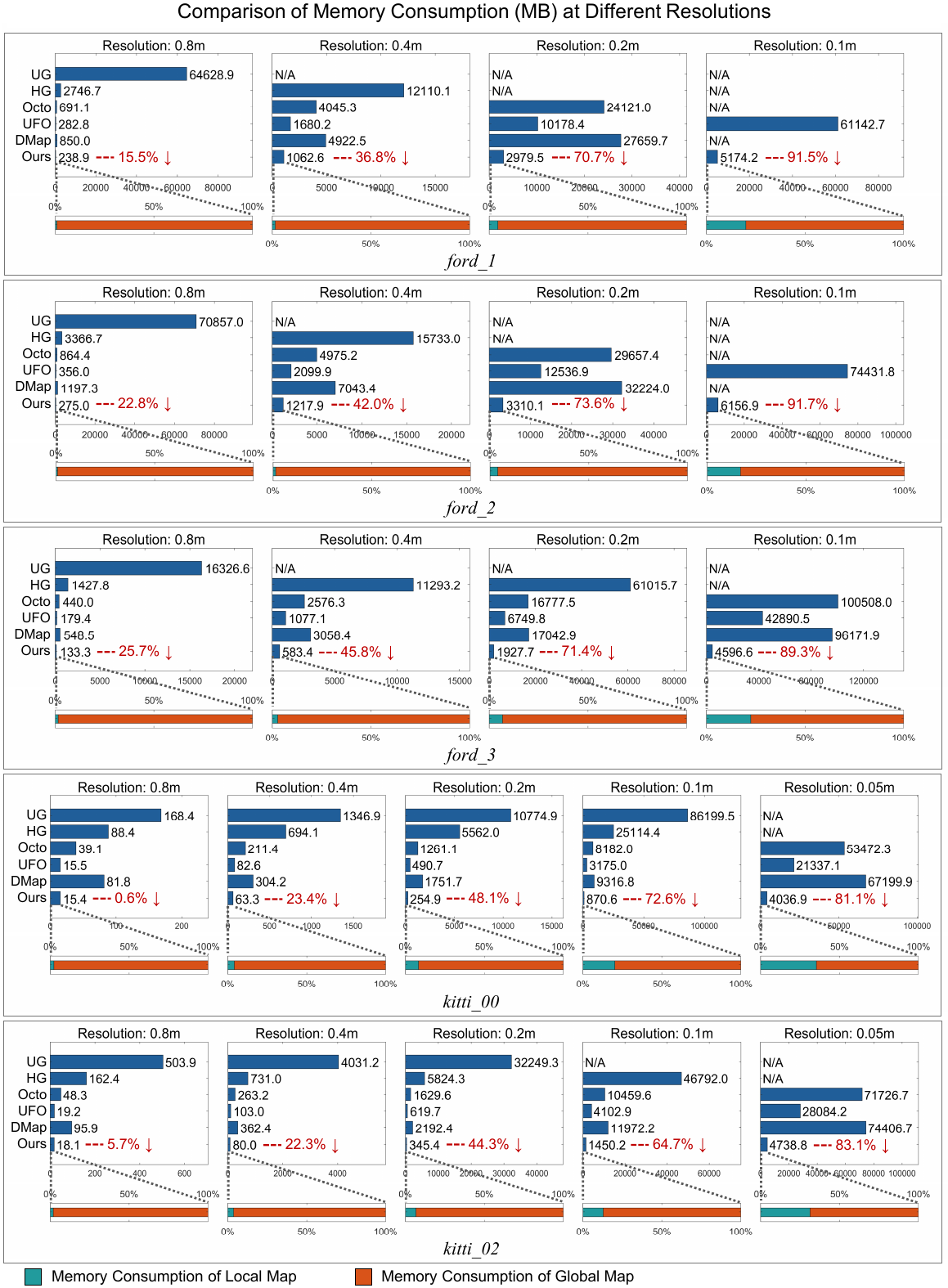}
    \vspace{-10pt} 
    \caption{Memory consumption (in MB) of UG (Uniform Grid), HG (Hash Grid), Octo (Octomap), UFO (UFOMap), D-Map, and our method {on the \textit{ford\_1}, \textit{ford\_2}, \textit{ford\_3}, \textit{kitti\_00}, and \textit{kitti\_02} sequences}. The percentage shown in red indicates the memory reduction of our method relative to the best of the rest methods in comparison. For our method, we also show the memory breakdown of the local and global maps, respectively. N/A indicates that the method failed due to memory consumption exceeding the limit.}
    \label{fig:exp_mem_1}
\end{figure*}

\begin{figure*}[t] 
    \centering
    \includegraphics[width=\textwidth]{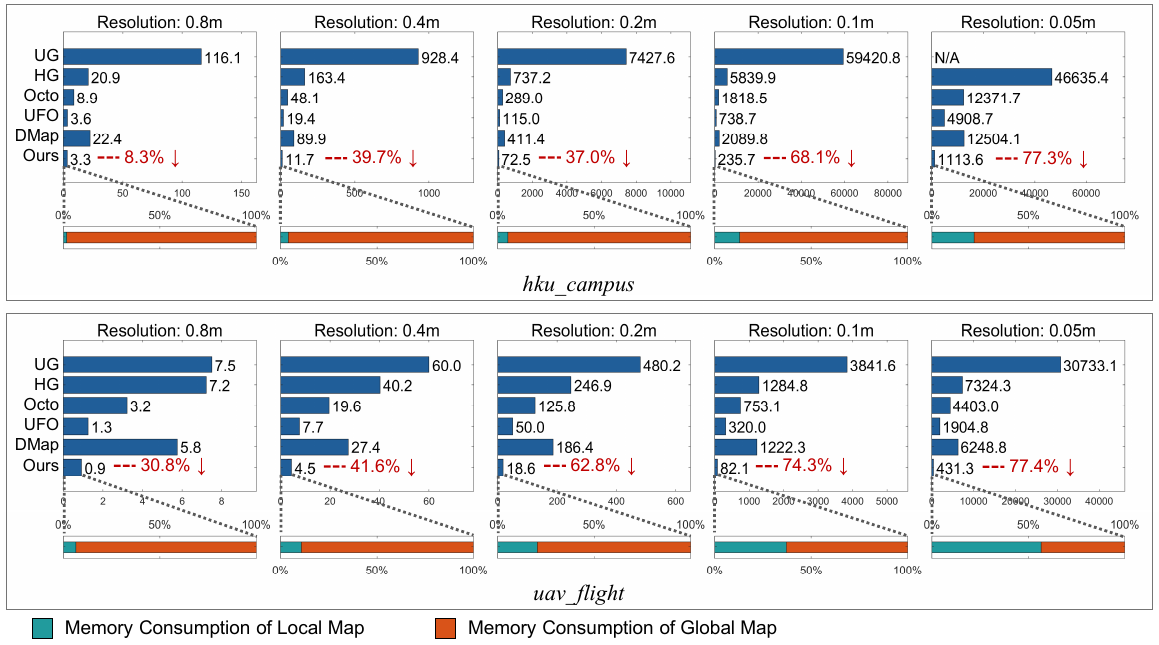}
    \vspace{-10pt} 
    \caption{{Continuation of Figure~\ref{fig:exp_mem_1}: results on the remaining sequences, \textit{hku\_campus} and \textit{uav\_flight}.}}

    \label{fig:exp_mem_2}
\end{figure*}

\subsection{Experiment Setup}
\label{sec:setup}
The experiments were conducted on a platform equipped with an Intel i7-1260P CPU, 64 GB of RAM, and 64 GB of swap space allocated on a Solid State Drive (SSD). 

The benchmark experiments were conducted across various map resolutions. For the \textit{kitti\_00}, \textit{kitti\_02}, \textit{hku\_campus}, and \textit{uav\_flight} sequences, the map resolution ranges from 0.8m to 0.05m. For the \textit{ford\_1}, \textit{ford\_2}, and \textit{ford\_3} sequences, the range is from 0.8m to 0.1m.

The probabilistic parameters for the ray casting updates were set as follows: the probabilities for a voxel \texttt{hit} and \texttt{miss} are \( p_{\text{hit}} = 0.8 \) and \( p_{\text{miss}} = 0.48 \). The clamping probabilities are \( p_{\text{max}} = 0.97 \) and \( p_{\text{min}} = 0.05 \). The threshold probabilities for the free and occupied states are \( p_{\text{free}} = 0.2 \) and \( p_{\text{occ}} = 0.8\). 
The sensor's sensing range \( R \) was determined according to the LiDAR user manual. 

For the octree-based method Octomap and UFOMap, the octree depth is set to the default value of 16 for the \textit{kitti\_00}, \textit{kitti\_02}, \textit{uav\_flight}, and \textit{hku\_campus} sequences under all map resolutions. However, for the Ford AV dataset sequences \textit{ford\_1}, \textit{ford\_2}, and \textit{ford\_3}, this default setting is insufficient under map resolution of 0.2m and 0.1m (i.e., the maximum map scale supported by the octree depth of 16 cannot accommodate the scale of these sequences). To address this limitation, we increase the octree depth to 18 for these sequences under map resolutions of 0.2m and 0.1m.

{For our method, in all experiments, the \( z \)-axis was selected as the projection axis. In typical robotic applications, the spatial extent along the \( z \)-axis is generally smaller than along the \( x \) and \( y \)-axes, making the \( z \)-axis a good choice for the projection axis to enhance computational efficiency. Notably, for specific scenarios, such as mapping tasks for vertical structures like buildings or sculptures, the \( x \)- or \( y \)-axis may be chosen as the projection axis instead.}

For every sensor scan, we define a sensing sphere \(\mathcal{R}\) which is centered at the sensor's origin, and with its radius equal to the sensor's sensing range \(R\). We then define the mapping space \(\mathcal{M}\) as the union of all such sensing spheres, which can be expressed as:

\begin{equation}
\mathcal{M} = \bigcup_{i=0}^{m} \mathcal{R}_i
\label{eqa:map_space}
\end{equation}
where \(\mathcal{R}_i\) denotes the sensing sphere corresponding to the \(i\)-th scan, and \(m\) is the total number of scans in the sequence.

The mapping space \(\mathcal{M}\) defines the region that may be updated by the sensor measurements. Any location outside this region remains unobserved and is considered unknown. For map query efficiency evaluation, random query samples are generated within this mapping space \(\mathcal{M}\). For map accuracy evaluation, we compare the occupancy state of voxels in \(\mathcal{M}\) determined by our method with that given by the baseline, and compute an accuracy percentage.

\subsection{Memory Consumption Evaluation}
\label{sec:exp_mem}

We present the memory consumption results of ours and all baseline methods in Figure~\ref{fig:exp_mem_1} and Figure~\ref{fig:exp_mem_2}. Our framework consists of a robo-centric local map (i.e., dense uniform occupancy grid map) and a global map (i.e., boundary map). A detailed memory breakdown for these two maps is also provided in the figure. Across all tested sequences and resolutions, our method consistently achieves the lowest memory consumption among all methods evaluated.

Among the baselines, octree-based methods (i.e., Octomap and UFOMap) generally exhibit better memory efficiency than grid-based methods. The UFOMap is the most memory-efficient method among the baselines, then followed by Octomap. 
In comparison with UFOMap, our method still achieves substantial memory reductions, especially in large-scale environments and at higher map resolutions. For instance, in the \textit{ford\_2} sequence at a map resolution of 0.1m, our method improves memory efficiency by 12.1 times (i.e., reduces memory consumption by \(91.7\%\)). Octomap exceeds the memory limit (i.e., 120GB) in this case. 
In the \textit{ford\_3} sequence at a map resolution of 0.1m, our method outperforms by 9.3 times (i.e., \(89.3\%\) memory reduction) compared to UFOMap, and 21.9 times (i.e., \(95.4\%\) memory reduction) compared to Octomap. 
D-Map employs a hybrid data structure and generally showcases a marginally higher memory consumption than Octomap. 

Compared to grid-based methods (i.e., Hash Grid and Uniform Grid), our method offers even more dramatic improvements. For instance, in the \textit{hku\_campus} sequence at a map resolution 0.05m, our method improves the memory efficiency by 41.9 times (i.e., \(97.6\%\) reduction) compare to Hash Grid. In the \textit{kitti\_00} sequence at a map resolution 0.1m, our method outperforms by 28.8 times (i.e., \(96.5\%\) reduction) compar{ed} to Hash Grid, and 99.0 times (i.e., \(99.0\%\) reduction) compared to Uniform Grid. 

As shown in the experimental results, our method achieves more substantial memory reduction when map resolution increases. In large-scale sequences, the memory consumption of our method grows approximately quadratically with increased resolution, while the baseline methods exhibit near-cubic growth. This divergence arises from the fact that our method stores the two-dimensional (2D) boundary voxels in the global map, whereas the baselines maintain the entire three-dimensional (3D) volume. Although our local map is also a dense uniform occupancy grid map that represents full 3D volume, it is robo-centric and maintains only the region surrounding the robot. As a result, in large-scale environments, the global map dominates overall memory usage, and the use of a 2D representation yields a near-quadratic growth.

As shown in the memory breakdown in Figure~\ref{fig:exp_mem_1} and Figure~\ref{fig:exp_mem_2}, the global map constitutes the majority of memory usage in large-scale sequences, including the Ford AV sequences (i.e., \textit{ford\_1}, \textit{ford\_2} and \textit{ford\_3}), KITTI sequences (i.e., \textit{kitti\_00} and \textit{kitti\_02}), and the \textit{hku\_campus} sequence. However, in the \textit{uav\_flight} sequence at a map resolution of 0.05m, the local map constitutes a larger portion of the total memory consumption. This is attributed to the relatively limited spatial scale of this sequence. This highlights that our method is more suitable for large-scale occupancy mapping tasks, where memory reduction is considerably more pronounced.

{Under low-resolution or small-scale environments, the memory consumption of occupancy grid maps is typically modest, allowing most existing methods to operate smoothly even on memory-constrained platforms. However, when scaling to larger environments or adopting finer resolutions, the memory demands of conventional methods grow rapidly, often exceeding practical hardware limits.
Therefore, in these challenging scenarios, the ability of our method to substantially reduce memory consumption is especially valuable.
Furthermore, this scalability suggests strong potential for even more large-scale and high-resolution applications beyond those evaluated in our experiments. As future mapping tasks push toward more challenging scene, the ability of our approach to maintain bounded memory usage through low-dimensional representation becomes an increasingly valuable asset.}

\begin{table*}[ht]
	\centering
	\caption{Comparison of map update time (ms) at different resolutions. The best and second-best results in each setting are highlighted using distinct tint colors.}
	\label{tab:update}
 \renewcommand*{\arraystretch}{1.0}
 \fontsize{8}{10}\selectfont 
	\begin{threeparttable}
		\begin{tabularx}{\textwidth}{@{}c 
		        >{\centering\arraybackslash}X
		        >{\centering\arraybackslash}X
		        >{\centering\arraybackslash}X
		        >{\centering\arraybackslash}X
		        >{\centering\arraybackslash}X
		        >{\centering\arraybackslash}X
		        >{\centering\arraybackslash}X
		        >{\centering\arraybackslash}X
		        >{\centering\arraybackslash}X
		        >{\centering\arraybackslash}X
		        >{\centering\arraybackslash}X
		        c@{}}
			\toprule
			\multicolumn{1}{l}{} & \multicolumn{4}{c}{\textit{ford\_1}} & \multicolumn{4}{c}{\textit{ford\_2}} & \multicolumn{4}{c}{\textit{ford\_3}} \\
			\cmidrule(l){2-5} \cmidrule(l){6-9} \cmidrule(l){10-13}
			Resolution (m) & 0.8 & 0.4 & 0.2 & 0.1 & 0.8 & 0.4 & 0.2 & 0.1 & 0.8 & 0.4 & 0.2 & 0.1 \\
			\midrule
    Uniform Grid & \cellcolor{myblue!25}11.56 & $\times$ & $\times$ & $\times$ & 11.89 & $\times$ & $\times$ & $\times$ & \cellcolor{myblue!85}\textbf{9.74} & $\times$ & $\times$ & $\times$ \\
    Hash Grid & 12.08 & \cellcolor{myblue!25}33.16 & $\times$ & $\times$ & \cellcolor{myblue!25}11.38 & \cellcolor{myblue!25}30.46 & $\times$ & $\times$ & 12.36 & \cellcolor{myblue!25}36.04 & 196.23 & $\times$ \\
    Octomap & 80.01 & 176.94 & 498.39 & $\times$ & 77.14 & 163.72 & 461.71 & $\times$ & 80.80 & 173.28 & 466.49 & 1353.29 \\
    UFOMap & 61.69 & 119.20 & 246.68 & \cellcolor{myblue!25}511.86 & 58.39 & 113.10 & 278.57 & \cellcolor{myblue!25}480.39 & 62.60 & 122.40 & 249.38 & 528.85 \\
    D-Map & 13.60 & 53.83 & \cellcolor{myblue!25}109.16 & $\times$ & 13.38 & 49.55 & \cellcolor{myblue!25}92.71 & $\times$ & 14.18 & 57.60 & \cellcolor{myblue!25}117.86 & \cellcolor{myblue!25}493.75 \\
    Ours & \cellcolor{myblue!85}\textbf{10.98} & \cellcolor{myblue!85}\textbf{24.89} & \cellcolor{myblue!85}\textbf{71.34} & \cellcolor{myblue!85}\textbf{219.59} & \cellcolor{myblue!85}\textbf{10.58} & \cellcolor{myblue!85}\textbf{23.61} & \cellcolor{myblue!85}\textbf{67.69} & \cellcolor{myblue!85}\textbf{207.78} & \cellcolor{myblue!25}12.04 & \cellcolor{myblue!85}\textbf{26.59} & \cellcolor{myblue!85}\textbf{75.50} & \cellcolor{myblue!85}\textbf{209.37} \\
		\end{tabularx}
	\end{threeparttable}	

	\begin{threeparttable}
		      \begin{tabularx}{\textwidth}{@{}c 
		        >{\centering\arraybackslash}X
		        >{\centering\arraybackslash}X
		        >{\centering\arraybackslash}X
		        >{\centering\arraybackslash}X
		        >{\centering\arraybackslash}X
		        >{\centering\arraybackslash}X
		        >{\centering\arraybackslash}X
		        >{\centering\arraybackslash}X
		        >{\centering\arraybackslash}X
		        c@{}}
            \toprule
            \multicolumn{1}{l}{} & \multicolumn{5}{c}{\textit{kitti\_00}} & \multicolumn{5}{c}{\textit{kitti\_02}} \\ \cmidrule(l){2-6} \cmidrule(l){7-11} 
            Resolution (m) & 0.8 & 0.4 & 0.2 & 0.1 & 0.05 & 0.8 & 0.4 & 0.2 & 0.1 & 0.05 \\ 
            \midrule
    Uniform Grid & \cellcolor{myblue!25}20.42 & \cellcolor{myblue!25}35.41 & \cellcolor{myblue!25}67.39 & \cellcolor{myblue!85}\textbf{209.19} & $\times$ & \cellcolor{myblue!25}19.92 & \cellcolor{myblue!25}35.80 & \cellcolor{myblue!25}71.61 & $\times$ & $\times$ \\
    Hash Grid & 21.97 & 40.37 & 96.56 & 342.80 & $\times$ & 21.70 & 42.56 & 89.32 & 274.34 & $\times$ \\
    Octomap & 133.20 & 251.77 & 509.20 & 1157.82 & 3355.46 & 139.04 & 261.07 & 541.39 & 1240.35 & 3706.38 \\
    UFOMap & 120.81 & 228.12 & 438.72 & 867.68 & 1701.42 & 124.23 & 236.93 & 457.62 & 906.31 & 1771.99 \\
    D-Map & \cellcolor{myblue!85}\textbf{11.86} & \cellcolor{myblue!85}\textbf{25.61} & \cellcolor{myblue!85}\textbf{48.89} & \cellcolor{myblue!25}226.12 & \cellcolor{myblue!25}1166.79 & \cellcolor{myblue!85}\textbf{11.78} & \cellcolor{myblue!85}\textbf{28.56} & \cellcolor{myblue!85}\textbf{55.68} & \cellcolor{myblue!85}\textbf{222.26} & \cellcolor{myblue!25}{938.94} \\
    Ours & 21.09 & 36.54 & 79.47 & 245.07 & \cellcolor{myblue!85}\textbf{794.19} & 20.72 & 38.13 & 83.76 & \cellcolor{myblue!25}263.03 & \cellcolor{myblue!85}\textbf{860.78} \\
        \end{tabularx}
	\end{threeparttable}	

	\begin{threeparttable}
		      \begin{tabularx}{\textwidth}{@{}c 
		        >{\centering\arraybackslash}X
		        >{\centering\arraybackslash}X
		        >{\centering\arraybackslash}X
		        >{\centering\arraybackslash}X
		        >{\centering\arraybackslash}X
		        >{\centering\arraybackslash}X
		        >{\centering\arraybackslash}X
		        >{\centering\arraybackslash}X
		        >{\centering\arraybackslash}X
		        c@{}}
            \toprule
            \multicolumn{1}{l}{} & \multicolumn{5}{c}{\textit{hku\_campus}} & \multicolumn{5}{c}{\textit{uav\_flight}} \\ \cmidrule(l){2-6} \cmidrule(l){7-11} 
            Resolution (m) & 0.8 & 0.4 & 0.2 & 0.1 & 0.05 & 0.8 & 0.4 & 0.2 & 0.1 & 0.05 \\ 
            \midrule
    Uniform Grid & \cellcolor{myblue!85}\textbf{0.68} & \cellcolor{myblue!85}\textbf{1.08} & \cellcolor{myblue!85}\textbf{1.87} & \cellcolor{myblue!85}\textbf{3.68} & $\times$ & \cellcolor{myblue!85}\textbf{0.71} & \cellcolor{myblue!85}\textbf{1.16} & \cellcolor{myblue!85}\textbf{2.04} & \cellcolor{myblue!85}\textbf{6.25} & \cellcolor{myblue!85}\textbf{12.37} \\
    Hash Grid & 1.02 & 1.71 & 3.17 & 6.74 & \cellcolor{myblue!25}20.80 & 1.10 & 1.92 & 3.68 & 9.20 & 30.78 \\
    Octomap & 4.02 & 6.88 & 13.92 & 34.78 & 120.09 & 4.56 & 7.91 & 15.93 & 41.73 & 137.84 \\
    UFOMap & 3.03 & 5.52 & 10.34 & 19.93 & 40.71 & 3.14 & 5.79 & 10.90 & 21.47 & 47.75 \\
    D-Map & 0.80 & 1.40 & 3.48 & 14.78 & 48.11 & 0.76 & 1.64 & 3.94 & 14.13 & 62.68 \\
    Ours & \cellcolor{myblue!25}0.69 & \cellcolor{myblue!25}1.18 & \cellcolor{myblue!25}2.32 & \cellcolor{myblue!25}6.21 & \cellcolor{myblue!85}\textbf{17.62} & \cellcolor{myblue!25}0.72 & \cellcolor{myblue!25}1.31 & \cellcolor{myblue!25}2.85 & \cellcolor{myblue!25}7.53 & \cellcolor{myblue!25}22.36 \\
     \bottomrule
        \end{tabularx}
	\end{threeparttable}

     \begin{tablenotes}
      \small
      \item \textit{Note:} $\times$ indicates that the method failed due to memory consumption exceeding the limit.
    \end{tablenotes}
    
 \vspace{-0.3cm}    
\end{table*}

\begin{figure*}[htbp] 
    \centering
    \includegraphics[width=\textwidth]{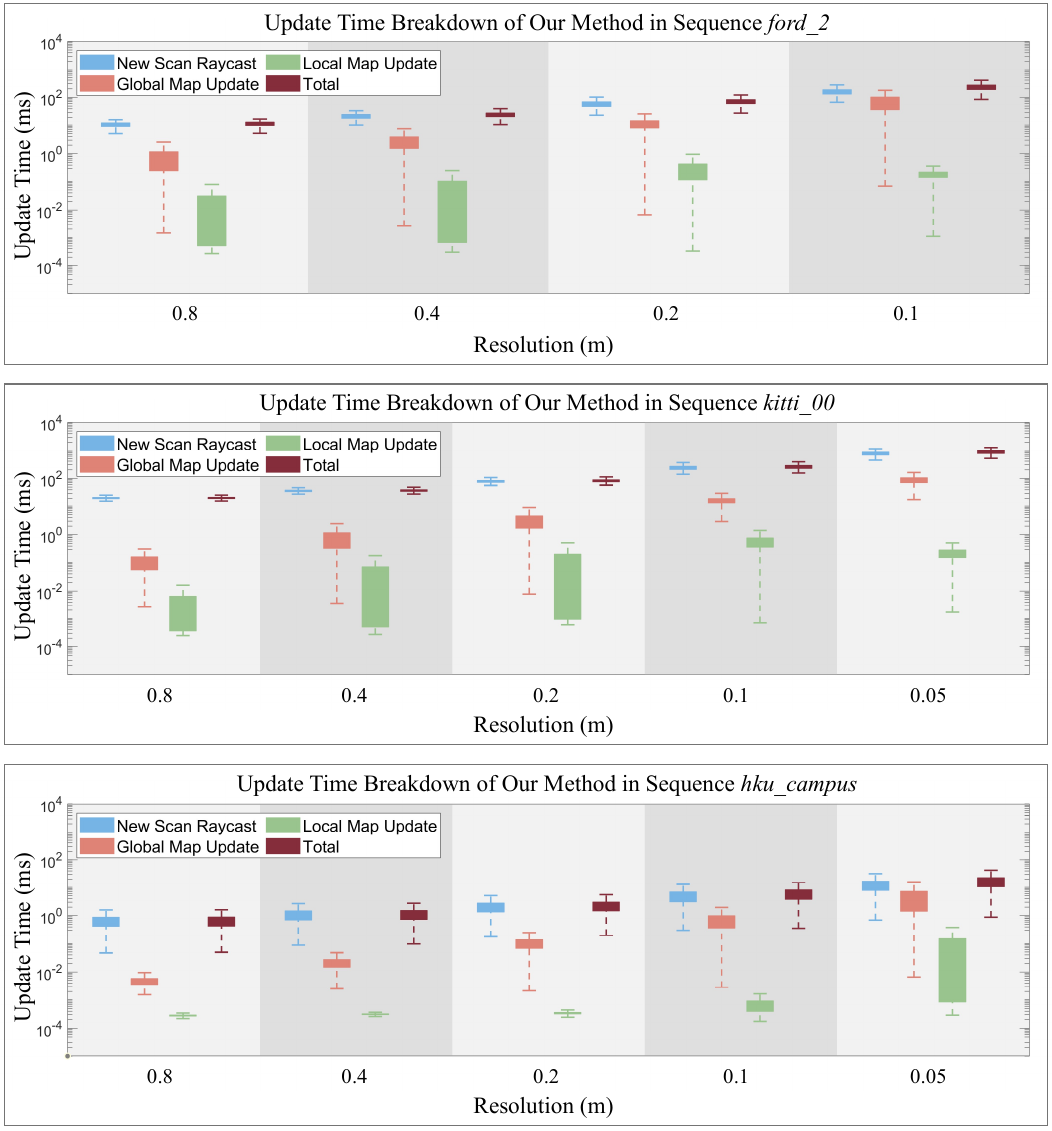}
    \vspace{-10pt} 
    \caption{{Breakdown of map update time of our method in sequence \textit{ford\_2}, \textit{kitti\_00} and \textit{hku\_campus}, respectively.}}
    \label{fig:breakdown_upd}
\end{figure*}

\subsection{Update Efficiency Evaluation}
\label{sec:exp_upd}

We report the benchmark results for map update times in Table~\ref{tab:update}. Additionally, we break down the overall map update time to analyze the contribution of the three components: the New Scan Ray Casting process, the Incremental Boundary Map Update, and the Local Occupancy Grid Map Update. Representative results are presented in Figure~\ref{fig:breakdown_upd}.

As shown in Table~\ref{tab:update}, in the KITTI sequences \textit{kitti\_00} and \textit{kitti\_02}, D-Map achieves the fastest update time in most of the resolution settings. Uniform Grid follows, with our method closely matching its performance. Hash Grid comes after with slightly slower performance, though still within a comparable range. In the Ford AV sequences \textit{ford\_1}, \textit{ford\_2}, and \textit{ford\_3}, our method achieves the best update performance, except in \textit{ford\_3} sequence at a map resolution of 0.8m. Notably, in the \textit{ford\_1} and \textit{ford\_2} sequences, Uniform Grid shows lower update efficiency compared to our method. This is attributed to Uniform Grid exhausting the available system RAM and starting to use swap space on the SSD, where memory access is slower than in physical RAM.
Hash Grid also demonstrates a marginally slower update time compared to our method in these sequences.
D-Map exhibits noticeably degraded performance in Ford AV sequences, generally ranking below the Hash Grid. In the \textit{hku\_campus} and \textit{uav\_flight} sequences, D-Map's update efficiency deteriorates further, with its update time approaching that of octree-based methods at higher map resolutions.

The octree-based methods, UFOMap and Octomap, show significantly lower update efficiency compared to our method. For instance, in the \textit{kitti\_02} sequence at a map resolution of 0.4m, our method outperforms UFOMap and Octomap by 6.2 times and 6.8 times, respectively.

In the following, we provide a comprehensive analysis {of} the results above.
For Uniform Grid, the map update process utilizes the ray casting technique. The complexity of this update process is analyzed as follows. The time complexity of updating a single voxel is \( \mathcal{O}(1) \). The number of voxels traversed by each ray can be approximated as \(\frac{R}{d}\), where \( R \) represents the sensor sensing range, and \( d \) denotes the map resolution. Consequently, the overall update process has a time complexity of \( \mathcal{O}(p\frac{R}{d}) \), where \( p \) is the total number of points in a sensor scan. The other grid-based method, Hash Grid, employs the same ray casting process. In the average case, the time complexity of updating a single voxel is also \( \mathcal{O}(1) \), making the overall update time complexity the same \( \mathcal{O}(p\frac{R}{d}) \). However, in the worst-case scenario, updating a single voxel in Hash Grid may degrade to \( \mathcal{O}(n) \), where \(n\) represents the number of all voxels maintained, due to hash collisions.

The local map in our mapping framework is implemented as a uniform occupancy grid map, resulting in the New Scan Ray Casting process having the same time complexity as the Uniform Grid as \( \mathcal{O}(p\frac{R}{d}) \). As shown in Figure~\ref{fig:breakdown_upd}, we observe that the overall map update time in our method is primarily dominated by the New Scan Ray Casting process. The Incremental Global Map Update and Local Occupancy Grid Map Update contribute only a minor portion to the total update time. This is due to several factors. i) The global map utilizes the low-dimensional boundary voxel representation, leading to fewer voxels that need to be processed during global map updates. Besides, the incremental global map update strategy also enhances the efficiency. ii) The efficient dense occupancy grid voxels construction method in local map updates. iii) The sliding mechanism and the incremental update method, which confine the global and local map updates to only the slide-in and slide-out regions. Consequently, our method demonstrates similar overall update efficiency to that of Uniform Grid.
As another grid-based method, Hash Grid exhibits slower update times compared to Uniform Grid and our method. This gap is largely due to hash collisions.

For octree-based methods such as Octomap and UFOMap, the time complexity of updating a single voxel is \(\mathcal{O}(\log(\frac{D}{d}))\), where \(D\) represents the scale of the mapping environment and \(d\) is the map resolution. These methods also employ ray casting for map updates, resulting in an overall time complexity of \(\mathcal{O}(p(\frac{R}{d}) \log(\frac{D}{d})))\). In comparison to ours and grid-based methods, this logarithmic overhead results in a significantly lower update efficiency, as demonstrated in our experimental results.

Theoretically, the map update time complexity of D-Map is \(\mathcal{O}(p(\frac{R}{d}) \log(\frac{D}{d})))\)~\cite{cai2023occupancy}, which is the same as the octree-based methods. {However, this time complexity corresponds to the worst-case scenario, and is less likely to occur in practical scenes compared to the octree-based methods. Furthermore, D-Map exhibits a decremental property that continuously removes voxels with a determined occupancy state, which also enhances its update efficiency.} As observed from the experiments, the update performance of D-Map heavily depends on the number of points in a sensor scan. For the KITTI sequences, the average number of points per scan is around 120k (as shown in Table~\ref{tab:dataset}), and in these sequences, D-Map outperforms both the grid-based methods and our method. However, in the Ford AV sequences, where the average number of points per scan is around 40k, D-Map performs worse than both the grid-based methods and our approach. In the private datasets \textit{hku\_campus} and \textit{uav\_flight}, where the average number of points per scan is around 6k, D-Map's performance further deteriorates. At higher map resolutions, its performance approaches that of the octree-based methods. {This indicates that D-Map is less efficient when handling sensor scans with low points density. This points density-related performance of D-Map is mainly due to its depth image-based method in determining the occupancy states.}

\begin{table*}[t]
	\centering
	\caption{Comparison of average map query time (ns) at different resolutions.}
	\label{tab:query_restrict}
 \renewcommand*{\arraystretch}{1.0}
 \fontsize{8}{10}\selectfont 
	\begin{threeparttable}
		\begin{tabularx}{\textwidth}{@{}c 
		        >{\centering\arraybackslash}X
		        >{\centering\arraybackslash}X
		        >{\centering\arraybackslash}X
		        >{\centering\arraybackslash}X
		        >{\centering\arraybackslash}X
		        >{\centering\arraybackslash}X
		        >{\centering\arraybackslash}X
		        >{\centering\arraybackslash}X
		        >{\centering\arraybackslash}X
		        >{\centering\arraybackslash}X
		        >{\centering\arraybackslash}X
		        c@{}}
			\toprule
			\multicolumn{1}{l}{} & \multicolumn{4}{c}{\textit{ford\_1}} & \multicolumn{4}{c}{\textit{ford\_2}} & \multicolumn{4}{c}{\textit{ford\_3}} \\
			\cmidrule(l){2-5} \cmidrule(l){6-9} \cmidrule(l){10-13}
			Resolution (m) & 0.8 & 0.4 & 0.2 & 0.1 & 0.8 & 0.4 & 0.2 & 0.1 & 0.8 & 0.4 & 0.2 & 0.1 \\
			\midrule
            Uniform Grid & {10510.9} & $\times$ & $\times$ & $\times$ & {25133.8} & $\times$ & $\times$ & $\times$ & \cellcolor{myblue!85}\textbf{41.18} & $\times$ & $\times$ & $\times$ \\
            Hash Grid & \cellcolor{myblue!25}86.28 & \cellcolor{myblue!25}123.12 & $\times$ & $\times$ & \cellcolor{myblue!25}84.11 & \cellcolor{myblue!25}132.83 & $\times$ & $\times$ & 101.73 & \cellcolor{myblue!25}288.19 & 173169.0 & $\times$ \\
            Octomap & 373.56 & 512.90 & 606.18 & $\times$ & 370.64 & 490.14 & 634.95 & $\times$ & 425.05 & 487.36 & 636.77 & 137039.0 \\
            UFOMap & 174.16 & 227.98 & \cellcolor{myblue!25}314.85 & \cellcolor{myblue!25}39516.2 & 182.43 & 234.03 & \cellcolor{myblue!25}313.96 & \cellcolor{myblue!25}82606.9 & 155.77 & 211.11 & \cellcolor{myblue!25}295.76 & \cellcolor{myblue!25}567.18 \\
            D-Map & 654.38 & 784.49 & 903.17 & $\times$ & 680.72 & 820.07 & 899.19 & $\times$ & 616.85 & 758.78 & 859.91 & 142607.3 \\
            Ours & \cellcolor{myblue!85}\textbf{78.86} & \cellcolor{myblue!85}\textbf{83.20} & \cellcolor{myblue!85}\textbf{81.32} & \cellcolor{myblue!85}\textbf{45.13} & \cellcolor{myblue!85}\textbf{80.36} & \cellcolor{myblue!85}\textbf{79.83} & \cellcolor{myblue!85}\textbf{77.39} & \cellcolor{myblue!85}\textbf{45.85} & \cellcolor{myblue!25}79.07 & \cellcolor{myblue!85}\textbf{88.78} & \cellcolor{myblue!85}\textbf{98.03} & \cellcolor{myblue!85}\textbf{56.86} \\
		\end{tabularx}
	\end{threeparttable}	

	\begin{threeparttable}
        \begin{tabularx}{\textwidth}{@{}c 
		        >{\centering\arraybackslash}X
		        >{\centering\arraybackslash}X
		        >{\centering\arraybackslash}X
		        >{\centering\arraybackslash}X
		        >{\centering\arraybackslash}X
		        >{\centering\arraybackslash}X
		        >{\centering\arraybackslash}X
		        >{\centering\arraybackslash}X
		        >{\centering\arraybackslash}X
		        c@{}}
            \toprule
            \multicolumn{1}{l}{} & \multicolumn{5}{c}{\textit{kitti\_00}} & \multicolumn{5}{c}{\textit{kitti\_02}} \\ \cmidrule(l){2-6} \cmidrule(l){7-11} 
            Resolution (m) & 0.8 & 0.4 & 0.2 & 0.1 & 0.05 & 0.8 & 0.4 & 0.2 & 0.1 & 0.05 \\  
            \midrule
            Uniform Grid & \cellcolor{myblue!85}\textbf{30.87} & \cellcolor{myblue!85}\textbf{30.43} & \cellcolor{myblue!85}\textbf{36.30} & 28773.4 & $\times$ & \cellcolor{myblue!85}\textbf{32.58} & \cellcolor{myblue!85}\textbf{41.82} & \cellcolor{myblue!85}\textbf{46.95} & $\times$ & $\times$ \\
            Hash Grid & 65.81 & 75.88 & 87.83 & \cellcolor{myblue!25}139.59 & $\times$ & 63.97 & 65.10 & 90.04 & \cellcolor{myblue!25}151.89 & $\times$ \\
            Octomap & 203.19 & 253.33 & 318.83 & 431.84 & 469013.6 & 183.22 & 227.12 & 304.98 & 380.92 & 638444.8 \\
            UFOMap & 86.53 & 128.51 & 159.48 & 203.38 & \cellcolor{myblue!25}264.28 & 91.67 & 127.28 & 149.64 & 205.21 & \cellcolor{myblue!25}254.30 \\
            D-Map & 496.90 & 477.45 & 546.25 & 666.67 & 84924.8 & 372.27 & 423.10 & 479.59 & 597.92 & 683656.5 \\
            Ours & \cellcolor{myblue!25}32.67 & \cellcolor{myblue!25}36.04 & \cellcolor{myblue!25}{36.42} & \cellcolor{myblue!85}\textbf{44.28} & \cellcolor{myblue!85}\textbf{48.23} & \cellcolor{myblue!25}55.42 & \cellcolor{myblue!25}62.96 & \cellcolor{myblue!25}62.62 & \cellcolor{myblue!85}\textbf{56.78} & \cellcolor{myblue!85}\textbf{58.80} \\
        \end{tabularx}
	\end{threeparttable}

	\begin{threeparttable}
        \begin{tabularx}{\textwidth}{@{}c 
		        >{\centering\arraybackslash}X
		        >{\centering\arraybackslash}X
		        >{\centering\arraybackslash}X
		        >{\centering\arraybackslash}X
		        >{\centering\arraybackslash}X
		        >{\centering\arraybackslash}X
		        >{\centering\arraybackslash}X
		        >{\centering\arraybackslash}X
		        >{\centering\arraybackslash}X
		        c@{}}
            \toprule
            \multicolumn{1}{l}{} & \multicolumn{5}{c}{\textit{hku\_campus}} & \multicolumn{5}{c}{\textit{uav\_flight}} \\ \cmidrule(l){2-6} \cmidrule(l){7-11} 
            Resolution (m) & 0.8 & 0.4 & 0.2 & 0.1 & 0.05 & 0.8 & 0.4 & 0.2 & 0.1 & 0.05 \\ 
            \midrule
            Uniform Grid & \cellcolor{myblue!85}\textbf{25.84} & \cellcolor{myblue!85}\textbf{30.36} & \cellcolor{myblue!85}\textbf{33.95} & \cellcolor{myblue!85}\textbf{42.08} & $\times$ & \cellcolor{myblue!85}\textbf{17.50} & \cellcolor{myblue!85}\textbf{23.60}  & \cellcolor{myblue!85}\textbf{26.46} & \cellcolor{myblue!85}\textbf{39.93} & \cellcolor{myblue!85}\textbf{40.89} \\
            Hash Grid & 51.01 & 70.05 & \cellcolor{myblue!25}71.24 & \cellcolor{myblue!25}79.34 & \cellcolor{myblue!25}108.86 & 34.55 & 48.13 & 61.13 & 84.94 & 94.57 \\
            Octomap & 129.77 & 154.54 & 178.84 & 218.77 & 285.32 & 86.86 & 205.26 & 269.57 & 269.94 & 354.00 \\
            UFOMap & 61.92 & 104.98 & 111.51 & 141.67 & 192.17 & 48.45 & 67.06 & 97.82 & 113.24 & 182.95 \\
            D-Map & 405.62 & 387.69 & 436.67 & 448.43 & 508.45 & 304.03 & 371.23 & 445.26 & 500.90 & 554.55 \\
            Ours & \cellcolor{myblue!25}36.67 & \cellcolor{myblue!25}62.87 & 76.52 & 80.97 & \cellcolor{myblue!85}\textbf{89.31} & \cellcolor{myblue!25}28.33 & \cellcolor{myblue!25}30.22 & \cellcolor{myblue!25}49.96 & \cellcolor{myblue!25}75.68 & \cellcolor{myblue!25}80.04 \\
            \bottomrule
        \end{tabularx}
	\end{threeparttable}
    
     \begin{tablenotes}
      \small
      \item \textit{Note:} $\times$ indicates that the method failed due to memory consumption exceeding the limit.
    \end{tablenotes}
    
 \vspace{-0.3cm}   
\end{table*}

\subsection{Query Efficiency Evaluation}
\label{sec:exp_query}

We randomly generate 100,000 query locations within the mapping space \(\mathcal{M}\) to evaluate the query efficiency of both our method and the baseline occupancy maps. The total query time is recorded, and the average query time per location is then computed. The results are summarized in Table~\ref{tab:query_restrict}. Note that in some high-resolution settings, certain methods demonstrate extremely high query times. This occurs because the method exhausts the available RAM and begins utilizing swap space on the SSD.

Theoretically, the average query time complexity for our method, as well as for the grid-based baselines Uniform Grid and Hash Grid, is \( \mathcal{O}(1) \). Despite this theoretical equivalence, there are subtle differences in their runtime performance. As shown in Table~\ref{tab:query_restrict}, Uniform Grid generally achieves the fastest query time when it is able to run (i.e., when it does not exceed the memory limit). Our method follows behind, with Hash Grid typically exhibiting slower performance compared to our method.

The runtime performance gap between our method and Uniform Grid is primarily attributed to two factors. First, while Uniform Grid is implemented as a contiguous array, our global map utilizes a hash-based 2D grid map, which introduces potential overhead due to hash collisions. Second, our method requires an additional binary search to determine the occupancy state apart from looking up the hash table.

As for the runtime performance gap between our method and Hash Grid, this gap arises from two main factors. First, our method employs a uniform occupancy grid for the local map, which is implemented as a contiguous array. Second, although both our global map and Hash Grid utilize a hash table, our global map maintains only a 2D grid rather than a 3D one. This significantly reduces the number of elements that are maintained in the hash table and minimizes hash collisions. As a result, despite the small overhead from the additional binary search in our method, the overall query time remains lower compared to Hash Grid. This performance advantage becomes more pronounced in large-scale, high-resolution scenarios.
For example, in the \textit{kitti\_02} sequence at a resolution of 0.1m, our method demonstrates a query speed that is 2.7 times faster than Hash Grid.
This is because combination of high map resolution and a large-scale environment results in an extensive number of voxels that are maintained in the hash table, leading to more frequent hash collisions.

The octree-based methods (i.e., Octomap and UFOMap) and D-Map all require traversing their hierarchical octree structures during occupancy state queries, which introduces a time complexity of \(\mathcal{O}(\log(\frac{D}{d}))\). As a result, the query efficiency of these methods is significantly lower than that of our method and grid-based approaches. This is also evidenced by the experimental results. {For example, in the \textit{ford\_3} sequence at a resolution of 0.1m, our method outperforms UFOMap by a factor of 10.0 in query efficiency, with UFOMap being the best-performing octree-based method.}
Notably, in terms of query efficiency, D-Map generally performs even worse than the octree-based methods. This is because D-Map employs a hybrid structure: it uses an octree to store unknown voxels and a separate hash-based grid map to store occupied voxels. As a result, querying the occupancy state in D-Map requires sequentially querying both the octree and the hash-based grid map, leading to additional overhead compared to the octree-based methods.

\begin{table*}[t]
	\centering
	\caption{Map Accuracy (\%) of our method and D-Map with Octomap as the ground truth.}
	\label{tab:acc_full}
 \renewcommand*{\arraystretch}{1.0}
 \fontsize{8}{10}\selectfont 
	\begin{threeparttable}
		\begin{tabularx}{\textwidth}{@{}c 
		        >{\centering\arraybackslash}X
		        >{\centering\arraybackslash}X
		        >{\centering\arraybackslash}X
		        >{\centering\arraybackslash}X
		        >{\centering\arraybackslash}X
		        >{\centering\arraybackslash}X
		        >{\centering\arraybackslash}X
		        >{\centering\arraybackslash}X
		        >{\centering\arraybackslash}X
		        >{\centering\arraybackslash}X
		        >{\centering\arraybackslash}X
		        c@{}}
			\toprule
			\multicolumn{1}{l}{} & \multicolumn{4}{c}{\textit{ford\_1}} & \multicolumn{4}{c}{\textit{ford\_2}} & \multicolumn{4}{c}{\textit{ford\_3}} \\
			\cmidrule(l){2-5} \cmidrule(l){6-9} \cmidrule(l){10-13}
			Resolution (m) & 0.8 & 0.4 & 0.2 & 0.1 & 0.8 & 0.4 & 0.2 & 0.1 & 0.8 & 0.4 & 0.2 & 0.1 \\
			\midrule
            Ours & \textbf{99.99} & \textbf{99.99} & \textbf{99.98} & \textbf{99.94} & \textbf{99.98} & \textbf{99.99} & \textbf{99.98} & \textbf{99.94} & \textbf{99.97} & \textbf{99.98} & \textbf{99.98} & \textbf{99.92} \\
            D-Map & 94.27 & 95.66 & 96.06 & $\times$ & 95.10 & 96.48 &96.55 & $\times$ & 93.66 & 95.42 & 95.19 & 94.70 \\
		\end{tabularx}
	\end{threeparttable}	

	\begin{threeparttable}
         \begin{tabularx}{\textwidth}{@{}c 
		        >{\centering\arraybackslash}X
		        >{\centering\arraybackslash}X
		        >{\centering\arraybackslash}X
		        >{\centering\arraybackslash}X
		        >{\centering\arraybackslash}X
		        >{\centering\arraybackslash}X
		        >{\centering\arraybackslash}X
		        >{\centering\arraybackslash}X
		        >{\centering\arraybackslash}X
		        c@{}}
            \toprule
            \multicolumn{1}{l}{} & \multicolumn{5}{c}{\textit{kitti\_00}} & \multicolumn{5}{c}{\textit{kitti\_02}} \\ \cmidrule(l){2-6} \cmidrule(l){7-11} 
            Resolution (m) & 0.8 & 0.4 & 0.2 & 0.1 & 0.05 & 0.8 & 0.4 & 0.2 & 0.1 & 0.05 \\  
            \midrule
            Ours & \textbf{99.95} & \textbf{99.95} & \textbf{99.95} & \textbf{99.91} & \textbf{99.57} & \textbf{99.98} & \textbf{99.99} & \textbf{99.99} & \textbf{99.96} & \textbf{99.67} \\
            D-Map & 96.98 & 97.97 & 98.66 & 98.56 & 98.69 & 97.94 & 98.60 & 99.05 & 98.99 & 98.93 \\
        \end{tabularx}
	\end{threeparttable}

    	\setlength{\tabcolsep}{12.5pt}
	\begin{threeparttable}
                 \begin{tabularx}{\textwidth}{@{}c 
		        >{\centering\arraybackslash}X
		        >{\centering\arraybackslash}X
		        >{\centering\arraybackslash}X
		        >{\centering\arraybackslash}X
		        >{\centering\arraybackslash}X
		        >{\centering\arraybackslash}X
		        >{\centering\arraybackslash}X
		        >{\centering\arraybackslash}X
		        >{\centering\arraybackslash}X
		        c@{}}
            \toprule
            \multicolumn{1}{l}{} & \multicolumn{5}{c}{\textit{hku\_campus}} & \multicolumn{5}{c}{\textit{uav\_flight}} \\ \cmidrule(l){2-6} \cmidrule(l){7-11} 
            Resolution (m) & 0.8 & 0.4 & 0.2 & 0.1 & 0.05 & 0.8 & 0.4 & 0.2 & 0.1 & 0.05 \\ 
            \midrule
            Ours & \textbf{99.86} & \textbf{99.90} & \textbf{99.92} & \textbf{99.92} & \textbf{99.89} & \textbf{99.28} & \textbf{99.63} & \textbf{99.76} & \textbf{99.84} & \textbf{99.86} \\
            D-Map & 85.14 & 89.93 & 92.22 & 92.42 & 91.00 & 71.53 & 71.22 & 77.49 & 83.05 & 85.41 \\
            \bottomrule
        \end{tabularx}
	\end{threeparttable}
    
     \begin{tablenotes}
      \small
      \item \textit{Note:} $\times$ indicates that D-Map failed due to memory consumption exceeding the limit.
    \end{tablenotes}
    
 \vspace{-0.3cm}   
\end{table*}

\subsection{Map Accuracy Evaluation}
\label{sec:exp_acc}

We evaluate the map accuracy of our method by comparing its results within the mapping space \(\mathcal{M}\) with those of Octomap (i.e., the ground-truth). Notably, to produce the ground-truth results, Octomap is executed on another platform with sufficient memory resources to run all sequences at all resolutions. Since Uniform Grid, Hash Grid, and UFOMap adopt the same map update process as Octomap, they are expected to produce identical mapping results. Similarly, we also compare the accuracy of D-Map against Octomap. The results are presented in Table~\ref{tab:acc_full}.

As shown, our method achieves mapping results that closely match those of Octomap. {The minor accuracy degradation observed arises from the conversion of discrete occupancy states to log-odds probabilities (see Equation~\ref{eqa:log_odds_init}) in the \texttt{Fuse} step in the Local Occupancy Grid Map Update process (see Section~\ref{sec:restore}). }

{In contrast, D-Map demonstrates a notable degradation in accuracy. This degradation arises from two primary factors. First, its depth image-based strategy can introduce errors in determining the voxel's occupancy state. Second, D-Map is built on the assumption of a static environment. Based on this assumption, it omits ray casting and probabilistic updates and adopts the more efficient update strategy. However, this design introduces a critical limitation: once a voxel is marked as occupied, it cannot be reverted to free—even if future scans no longer observe the obstacle—resulting in persistent false positives. Consequently, D-Map performs poorly in dynamic environments. This limitation is particularly evident in the \textit{uav\_flight} sequence. In this sequence, a MAV flies through unconstructed outdoor trails, stirring up large quantities of dust particles, which act as dynamic, transient tiny objects. A significant drop in map accuracy of D-Map is observed in this sequence. In contrast, both our method and Octomap incorporate the ray casting and probabilistic updates. Such mechanism is capable for handling such dynamic scenarios: voxel currently marked as occupied due to the dynamic objects can be later corrected to free when rays of subsequent sensor scans pass through it. As a result, our method still maintains high map accuracy in such scenario.}

\section{Real-world Applications}
\label{sec:real-world}

\begin{figure*}[t]
    \centering
    \includegraphics[width=\linewidth]{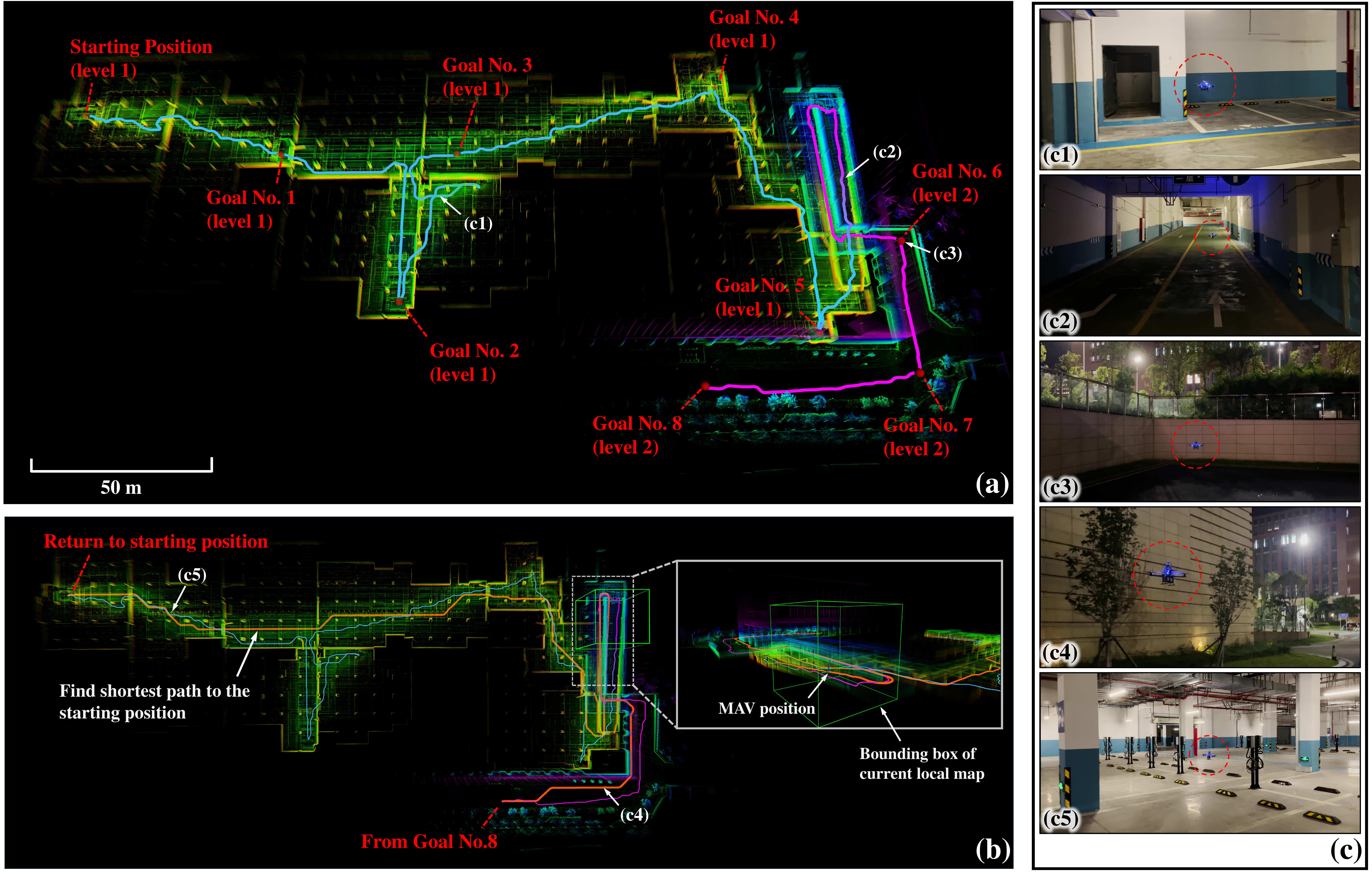}
    \caption{Illustration of a long-range autonomous navigation task conducted in a multi-level building. Starting from an initial position, the MAV sequentially navigated to a series of manually defined goals, each specified after the previous one was reached. All goals were intentionally placed beyond the MAV’s current local map to enforce long-range navigation. In total, eight goals were assigned and visited. After reaching the final goal (Goal No.8), the MAV was instructed to autonomously plan and execute a return trajectory to the starting position. (a) Locations of the eight goals. Goals No.1–5 are located on level 1, while Goals No.6–8 are on level 2. The figure also demonstrates the accumulated point cloud and the MAV trajectory upon reaching Goal No.8, with the trajectory color-coded by altitude. (b) Return trajectory from Goal No.8 to the starting position, highlighted in orange. The MAV followed the shortest path computed using the A* search algorithm based on the proposed global-local mapping framework. The figure also illustrates the local map at a position selected on the return trajectory, highlighting its limited spatial range compared to the global map. (c) Sample images of the MAV captured during the mission, with their corresponding locations on the trajectory marked in (a) and (b). Images (c1)–(c3) were taken during the outbound flight to Goal No.8, with locations annotated in (a). Images (c4) and (c5) were captured during the return flight and correspond to locations indicated in (b).}
    \label{fig:real_seg_pc}
\end{figure*}

\begin{figure}[t]
    \centering
    \includegraphics[width=\linewidth]{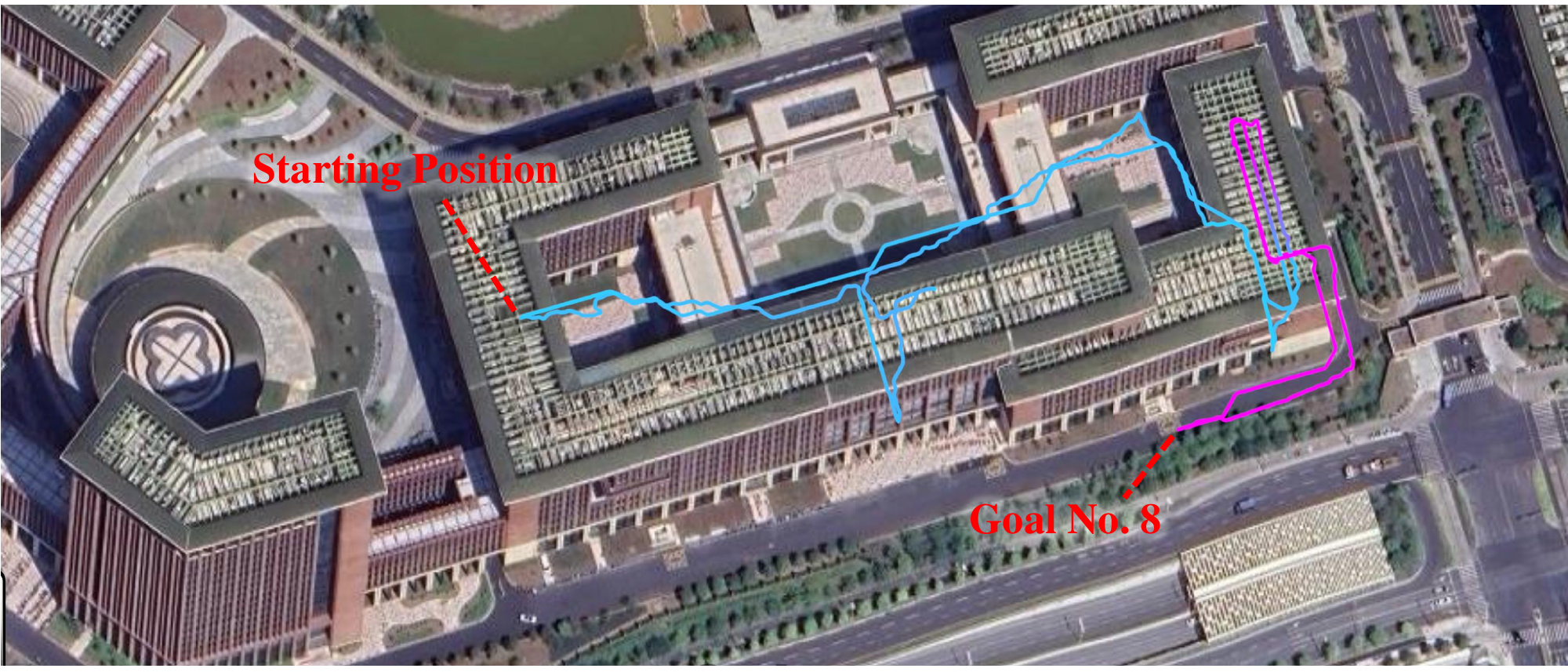}
    \caption{Complete MAV flight trajectory, including the outbound trajectory to Goal No.8 and the return to the starting position, overlaid on a Google Maps satellite image. The trajectory is color-coded by altitude.}
    \label{fig:real_google}
\end{figure}

\begin{figure*}[t]
    \centering
    \includegraphics[width=\linewidth]{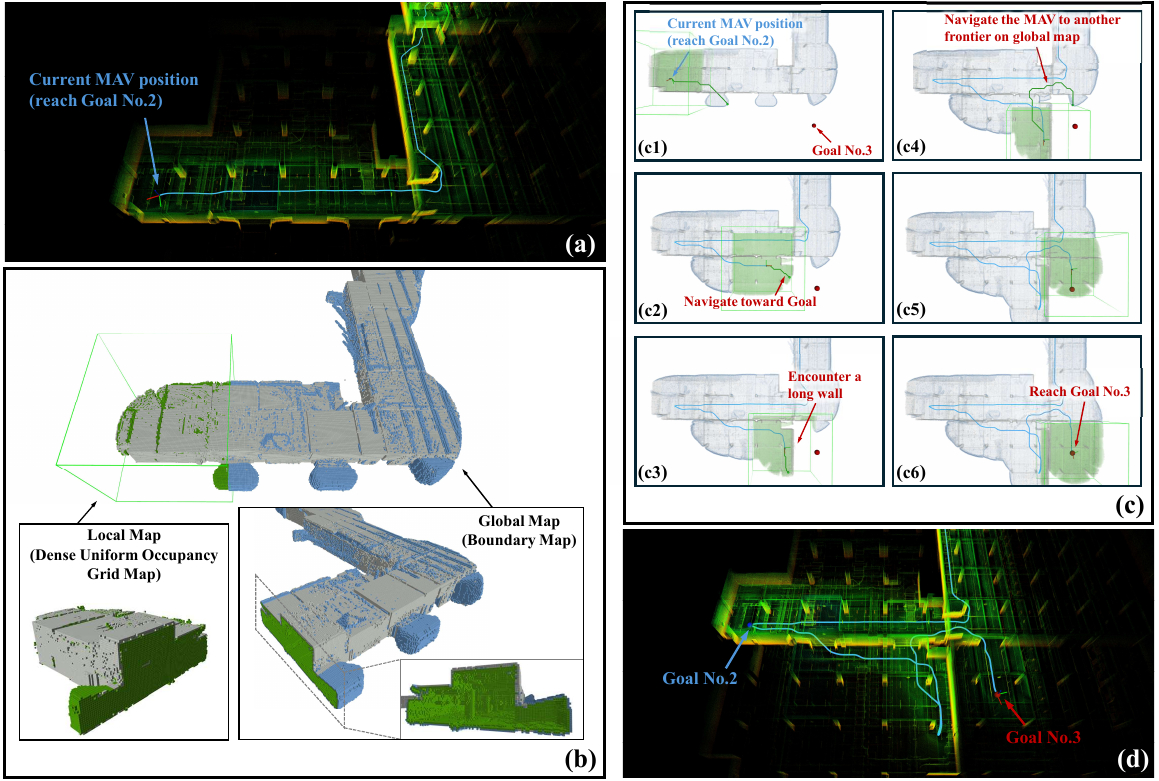}
    \caption{Illustration of how the MAV navigates itself to a long-range goal, using the segment from Goal No.2 to Goal No.3 as an example. (a) The accumulated point cloud and trajectory when reaching Goal No.2. (b) Visualization of the corresponding mapping framework at this moment, including the global map (i.e., boundary map) and the local map. The free, unknown and occupied voxels are colored by green, blue and grey, respectively. (c) Autonomous navigation process toward Goal No.3. The red dot marks the location of Goal No.3. The green dot indicates the selected target on the frontier. The green line shows the computed A* path to the frontier target. For clarity, both the global and local maps are rendered semi-transparent. (d) Accumulated point cloud and trajectory upon reaching Goal No.3.
}
    \label{fig:real_plan_g3}
\end{figure*}

\begin{figure}[ht]
    \centering
    \includegraphics[width=\linewidth]{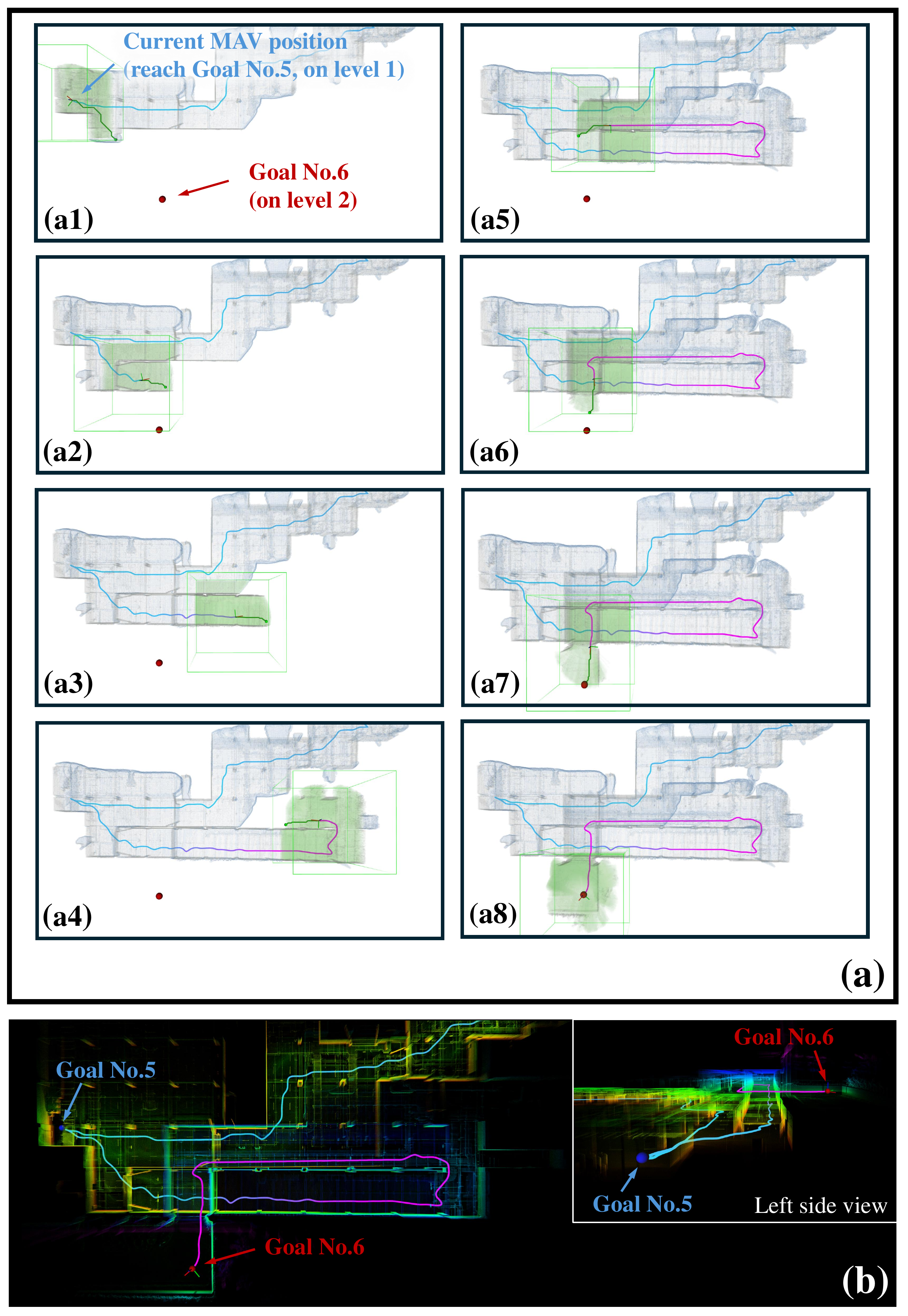}
    \caption{(a) Illustration of the autonomous navigation process from Goal No.5 toward Goal No.6. (b) Visualization of the accumulated point cloud and the MAV trajectory upon reaching Goal No.6. The trajectory is color-coded by altitude.}
    \label{fig:real_plan_g6}
\end{figure}

Long-range autonomous navigation tasks are highly challenging, particularly when the goal lies beyond the coverage of the local map. In such cases, the vehicle must maintain a global occupancy map to prevent falling into local traps, where the vehicle may become stuck in back-and-forth movements, hindering its ability to efficiently reach the goal. However, conventional global occupancy grid map requires extensive memory, which poses challenges to memory-constrained onboard platforms. Our mapping framework offers an effective and memory-efficient solution for addressing these challenges.

\subsection{System Design}
First, we present the strategy employed to tackle the long-range autonomous navigation task. If the goal lies within the free region, a collision-free A* path is computed from the current vehicle position to the goal. The A* path is entirely contained within known free region to ensure safety. If the goal is located within the unknown region, the planner instead navigates the vehicle toward a frontier—the interface between free and unknown space. To select an optimized frontier, we design a cost function inspired by~\cite{bircher2016receding}, which balances the cost of reaching the frontier with the frontier's proximity to the goal. An A* path is then computed from the current vehicle position to the selected frontier, confined within the free region, thereby ensuring a safe flight. After obtaining the A* path, we generate local trajectories for the vehicle along this path to ensure smooth and energy-efficient motion, following the approach outlined in~\cite{ren2025safety}.

In conventional occupancy grid maps, identifying frontiers typically requires scanning the entire map, which is computationally expensive. In contrast, our mapping framework naturally exposes frontier candidates through boundary exterior (unknown) voxels (i.e., \(\mathbf{b}_{\mathrm{ukn}}\)), which inherently indicate the interface between free and unknown regions. These voxels are explicitly maintained in the global map and can be directly retrieved. Thus, in our framework, scanning and frontier detection {are} required only in the local map. Since the local map typically covers a significantly smaller area than the global map in long-range navigation tasks, our approach enables substantially more efficient frontier identification compared to conventional methods.

Additional system components are integrated as follows. For MAV localization, we utilize a modified version \cite{zhu2022decentralized} of FAST-LIO2~\cite{xu2022fast}, which provides high-accuracy state estimation at a frequency of 100 Hz. For trajectory tracking and control, we employ the on-manifold model predictive controller proposed in~\cite{lu2022manifold}.

\subsection{Experiment Setup}

Our MAV platform is equipped with a semi-solid state LiDAR, the Livox MID-360, a flight control unit (FCU) running PX4 Autopilot~\cite{liu2024omninxt}, and an Intel NUC which has an Intel i7-1260P CPU and 64 GB of RAM. The extrinsic parameters between the LiDAR and the built-in IMU of the FCU are calibrated using LI-Init~\cite{zhu2022robust}.

The experiments were conducted in a two-level underground parking lot with an environment scale of \(277\text{m} \times 123\text{m} \times 19\text{m}\). {Several long-range goals were assigned during the mission. Specifically, when the vehicle reaches one goal, the next goal is {randomly set}. All of the goals lie outside the current local map (i.e., \(22\text{m} \times 22\text{m} \times 22\text{m}\) )}, with the farthest goal requiring the vehicle to travel a distance of \(\text{520.4m}\) to reach it.

\subsection{Results and Analysis}

At the end of the experiment, the MAV traveled a total distance of \(1,247.5\) meters, and mapped a total volume of \({29,153.4}\, \text{m}^3\).  The goals positions, complete flight trajectory, and accumulated point cloud scans are shown in Figure~\ref{fig:real_seg_pc}. An overview of the trajectory overlaid on a Google Maps satellite image is provided in Figure~\ref{fig:real_google}. Representative segments of the planning process are illustrated in Figure~\ref{fig:real_plan_g3} and Figure~\ref{fig:real_plan_g6}, including the visualization of the global and local maps, planned A* paths, and executed trajectories. Additional information for this flight is summarized in Table~\ref{tab:real_flight}. A multimedia demonstration of the full planning process is provided in Extension~1. 

Furthermore, we performed a performance evaluation of our method against several benchmarked methods, including Hash Grid, Octomap, UFOMap, and D-Map, for this flight. Figure~\ref{fig:mem_real} shows the memory consumption (in MB) of our method and the baselines during real-world flights under various map resolutions.
A notable steep increase in memory consumption is observed for Hash Grid, which occurs due to the rehashing process being triggered. A similar behavior is observed in our method and D-Map, as they also utilize hash-based grid maps in their schemes.

The results for map update times during real-world flights are presented in Figure~\ref{fig:upd_real}. At map resolution of 0.2m, our method, Hash Grid, and D-Map exhibit comparable map update efficiency. In contrast, the octree-based methods, including Octomap and UFOMap, demonstrate a significant decline in performance. At a map resolution of 0.05m, our method continues to maintain high update efficiency, while D-Map shows clear degradation in performance under this setting.

\begin{figure}[t]
    \centering
    \includegraphics[width=\linewidth]{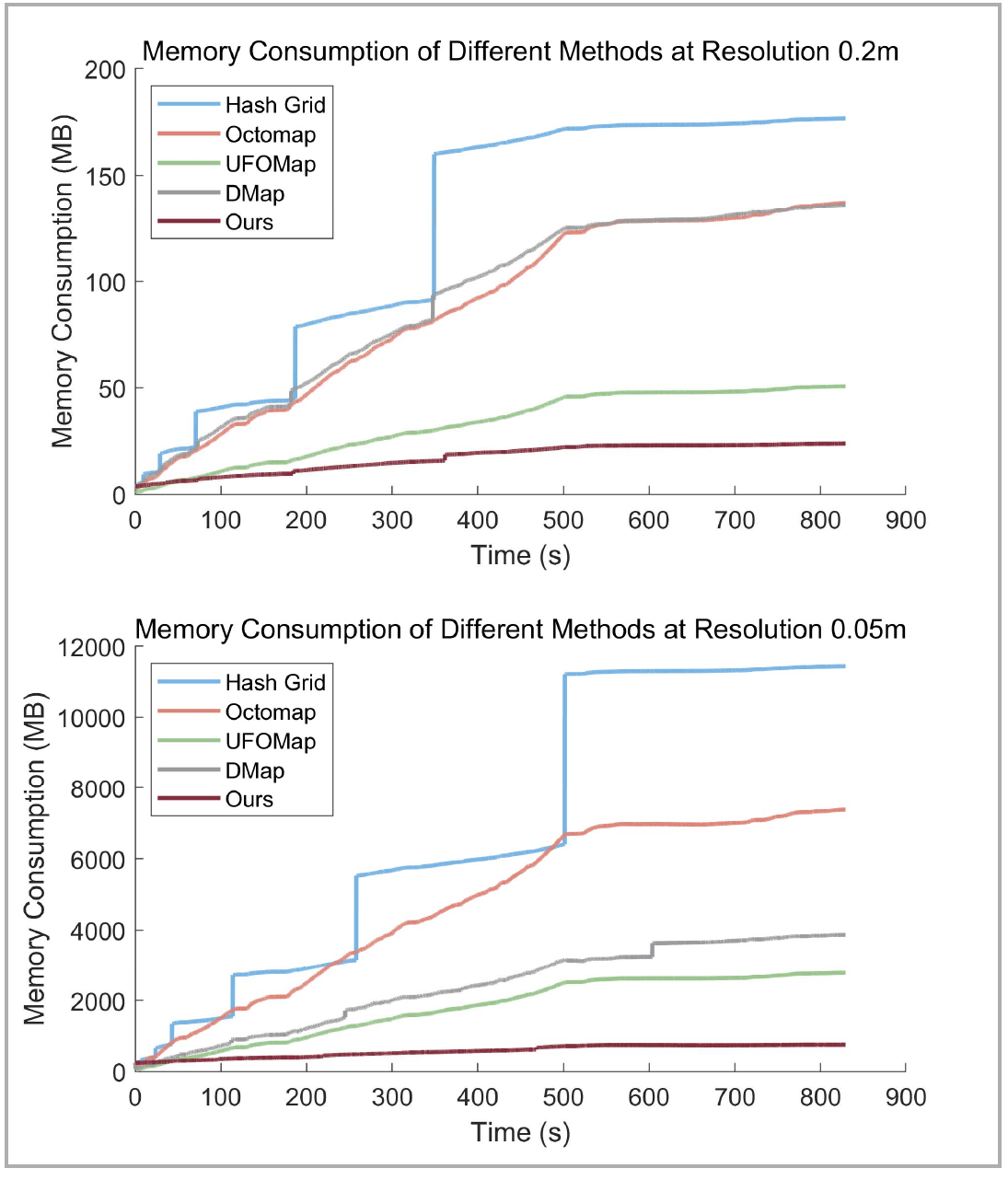}
    \caption{Memory consumption (MB) of our method and other baselines during the flight.}
    \label{fig:mem_real}
\end{figure}

\begin{table}[t]
	\setlength{\tabcolsep}{8.0pt}
	\centering
	\caption{Details of the Real-world Flight}
	\label{tab:real_flight}	
	\begin{threeparttable}
		\begin{tabular}{@{}ll@{}}
			\toprule
		Environment Scale (Bounding Box) ($\mathrm{m}^3$) & $277\times123\times19$ \\
            Travel Distance ($\mathrm{m}$) & 1,247.5 \\
            Mapped Volume ($\mathrm{m}^3$) & 29153.4 \\
            Number of Scans & 8,291\\
            Average Point Number (per scan) & 16,252.1 \\
			   \bottomrule
		\end{tabular}
	\end{threeparttable}
 \vspace{-0.2cm}
\end{table}

\begin{figure}[t]
    \centering
    \includegraphics[width=\linewidth]{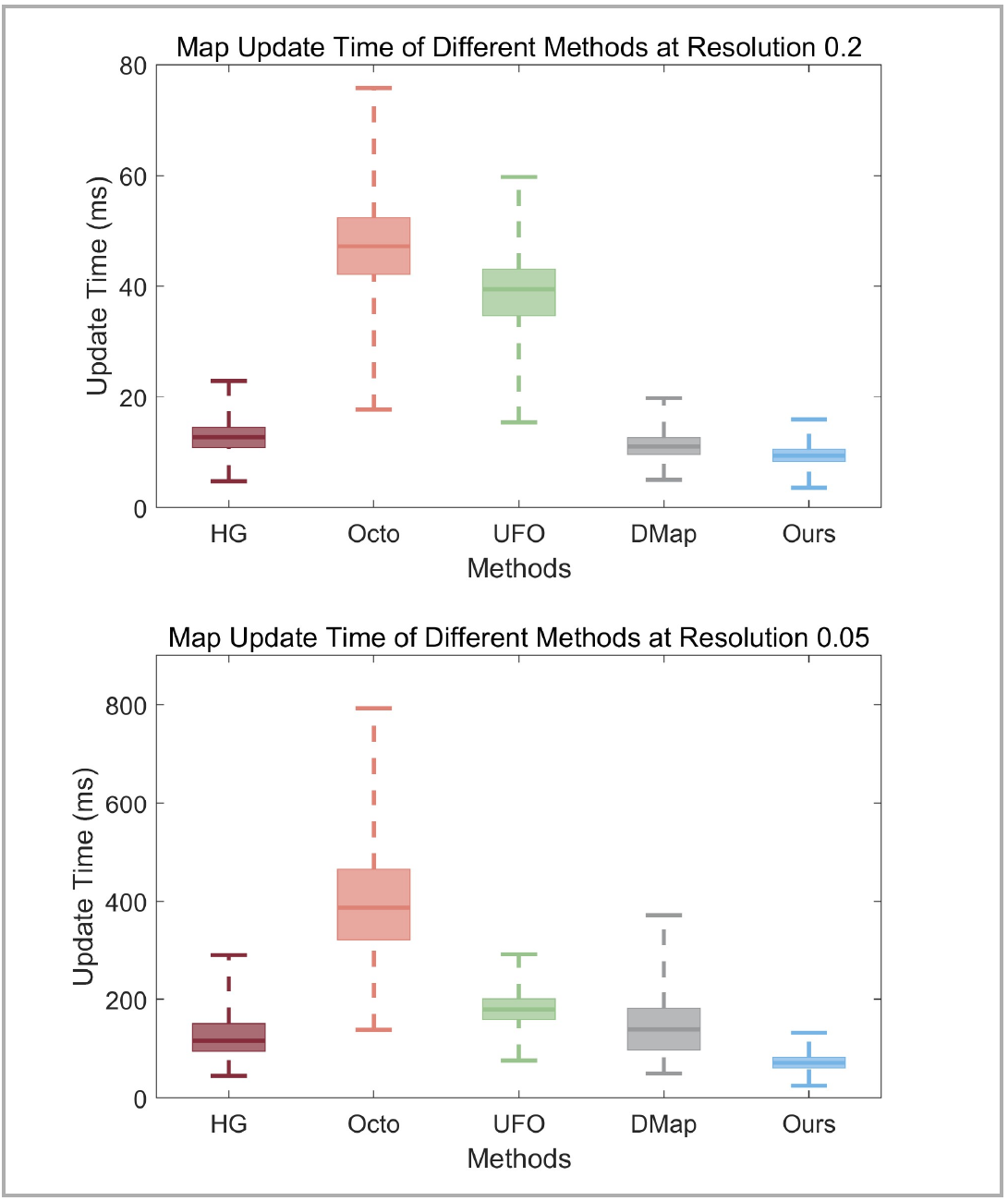}
    \caption{Map update time of our method and other baselines during the flight.}
    \label{fig:upd_real}
\end{figure}

\section{Discussion}
\label{sec:discuss}
{In this section, we first discuss the impact of different projection axis choices in various environments, and discuss a projection axis selection strategy (see~Section~\ref{sec:discuss:axis}). Then, we discuss how our map framework handles dynamic environments and further evaluate it in an environment that contains richer dynamic objects (see~Section~\ref{sec:discuss:dyn}). Finally, we address the limitations of our method and outline potential directions for future work (see~Section~\ref{sec:discuss:limit}).}

\subsection{{Projection Axis Selection}}
\label{sec:discuss:axis}
{Consider an environment with dimensions $a \times b \times c$, with $a$, $b$, and $c$ representing the scales along the $x$-, $y$-, and $z$-axes, respectively, and satisfying $a < b < c$. In this case, the $z$-axis corresponds to the largest spatial dimension of the environment. Projecting along the $z$-axis can cause many boundary voxels to be projected onto the same 2D grid cell. Since the map query involves performing a binary search over the boundary voxels stored in the cell, this concentration can increase the query time. In contrast, projecting along the $x$-axis—which corresponds to the smallest dimension—distributes the boundary voxels more evenly across the 2D grid cells, thereby improving query efficiency. However, this comes at the cost of a larger 2D grid map and, consequently, higher memory consumption.
Conversely, for an environment where $a > b > c$, selecting the $z$-axis as the projection axis results in faster query performance but higher memory usage, whereas choosing the $x$-axis leads to lower memory usage but longer query times. Overall, the choice of projection axis involves a trade-off between memory usage and query efficiency.}

{We conducted an additional experiment to further analyze this trade-off under different projection axis choices.
The test was performed on the \textit{kitti\_02} sequence at a representative map resolution of 0.1m. The \textit{kitti\_02} sequence captures a vehicle traversing urban roads, where the $z$ dimension exhibits the smallest spatial extent. Specifically, the environment scale of the sequence is $1,035\text{m} \times 688\text{m} \times 94\text{m}$, as presented in Table~\ref{tab:dataset}. In our original benchmark experiments (see Section~\ref{sec:benchmake}), the $z$-axis was selected as the projection axis. Here, we additionally evaluate cases where the $x$- and $y$-axes are used instead. We measured and reported the memory consumption and query performance of our mapping framework under different choices of projection axis. In addition, we measure the average number of boundary voxels that are processed by the binary search per query, denoted as $\bar{N}_{\text{bv/q}}$. This value directly influences the query time and reflects the density of boundary voxels stored in a 2D grid cell. The results are summarized in Table~\ref{tab:proj_axis}.}

\begin{table}[t]
\centering
\caption{{Impact of different projection axis choices on memory consumption and query time for the \textit{kitti\_02} sequence at a map resolution of 0.1m. The benchmark experiments in Section~\ref{sec:benchmake} use the $z$-axis as the projection axis.}}
\label{tab:proj_axis}
\begin{tabular}{c c c c}
\toprule
{Projection Axis} & {Memory (MB)} & {Query (ns)} & $\bar{N}_{\text{bv/q}}$\\
\midrule

$x$ & \textbf{637.48} & 82.84 & 33.47 \\
$y$ & 669.44 & 73.45 & 21.82 \\
$z$ & 1450.18 & \textbf{56.78} & \textbf{3.02} \\

\bottomrule
\end{tabular}
\end{table}

{As confirmed by the results, selecting a projection axis with a larger spatial extent (i.e., $x$- or $y$-axis) further reduces memory consumption. However, this choice also causes more boundary voxels being projected onto the same 2D grid cell, as indicated by the increase of $\bar{N}_{\text{bv/q}}$, thereby resulting higher map query time.
In addition, We note that when choosing $x$ or $y$ as the projection axis, $\bar{N}_{\text{bv/q}}$ significantly increases. However, the query time does not increase dramatically. This is because the binary search has a logarithmic time complexity, $\mathcal{O}(\log(n))$, where $n$ represents the number of boundary voxels stored in the 2D grid cell, rather than a linear complexity of $\mathcal{O}(n)$. Thus, the map query process still maintains efficiency under such scenarios.
}

{Based on the results, we discuss a {practical} axis selection strategy to maximize the performance {for deployment}. When initiating the mapping process, the user specifies the approximate environment scale $a \times b \times c$. This information is then used to determine the optimal projection axis. Generally, when the user prioritizes efficient query performance, the projection axis should be aligned with the direction of the smallest spatial extent. Conversely, when more compact memory consumption is desired, the projection axis should be selected as the direction with the larger spatial extent. It is worth noting that, once the projection axis is determined and the mapping process begins, our current framework does not support switching the projection axis during operation. This is because such a change would require reprojecting all boundary voxels in the environment onto a new 2D grid along the newly selected projection axis, which can be time-consuming, especially in large-scale scenarios. Therefore, to fully exploit the potential performance of our framework (e.g., maximizing memory efficiency or minimizing query time), the user is advised to provide an approximate estimation of the environment scale before starting the mapping process. To adaptively change the projection axis during operation, a potential solution is to partition the environment into multiple subspaces, each adopting a projection axis that best fits its local spatial characteristics, which can be further explored in the future work.}

\subsection{{Dynamic Environments}}
\label{sec:discuss:dyn}
{In this section, we first provide a detailed discussion of how our map framework handles dynamic objects. Then, we include an additional sequence with extensive dynamic objects to further demonstrate the robustness of our method in dynamic environments.}

{Our mapping framework inherently handles transient obstacles through its update mechanisms.
Specifically, each voxel in the local map stores log-odds occupancy probability. For dynamic objects within the local map, when such objects leave the voxel and subsequent sensor rays pass through it, the log-odds value is updated linearly towards a free state. This corresponds to a logarithmic change in the occupancy probability. This logarithmic relationship is illustrated in Figure~\ref{fig:logodds_prob}. To ensure that the state of such dynamic voxels can be updated rapidly, we employ a clamping policy. Specifically, an upper clamping threshold $p_{\text{max}}$ constrains each voxel's maximum occupancy probability. This prevents the log-odds value from becoming excessively high, allowing the voxel to quickly transit back to a free state. A lower clamping threshold $p_{\text{min}}$ ensures that a free voxel can rapidly adapt to an occupied state when a dynamic object appears. This clamping policy is also widely adopted in the baseline grid-based and octree-based mapping frameworks. For dynamic objects that slide out of the local map, their last observed occupancy is recorded by the global boundary map. When these regions re-enter the local map, they are re-initialized with a marginal occupancy probability (i.e., set to the threshold value $p_{\text{occ}}$ for the occupied state, and $p_{\text{free}}$ for the free state, respectively). With new sensor measurements, such voxels can also be rapidly updated to the opposite state.}

\begin{figure} [t]
	\centering
	\includegraphics[width=\linewidth]{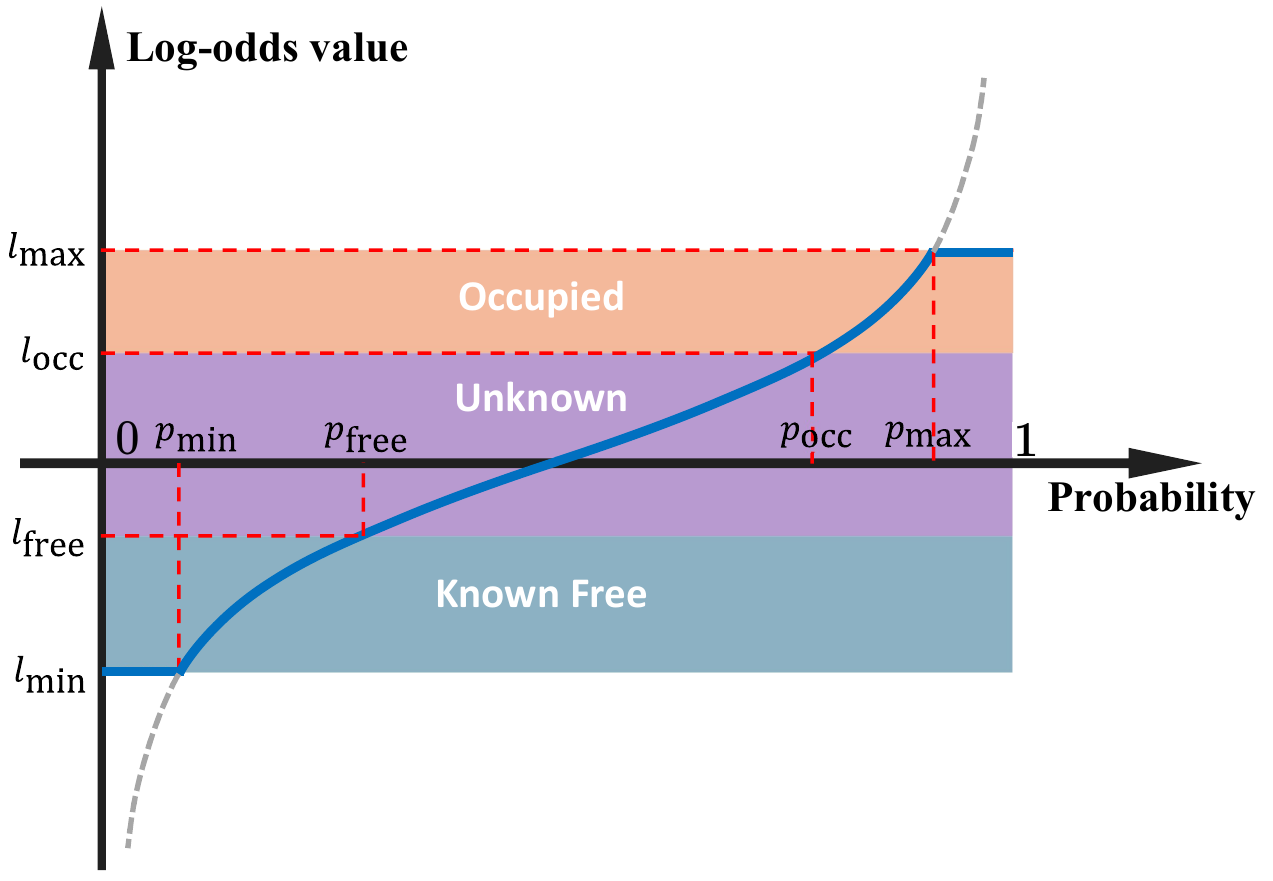}
	\caption{{Illustration of the relationship between occupancy probability and its log-odds value.}}
	\label{fig:logodds_prob}
\end{figure}

{The sequences used in our benchmark experiments (see Section~\ref{sec:benchmake}) already contain some dynamic objects. For instance, both the KITTI and Ford AV sequences were collected during real-world urban driving, which naturally includes some moving objects such as vehicles and pedestrians. The map accuracy experiments in Section~\ref{sec:exp_acc} demonstrate that our method achieves mapping results that closely match those of Octomap, which is recognized as an effective framework for handling dynamic objects through its probabilistic update mechanism. These results indicate that our method can robustly operate in such scenarios.}

{To further evaluate the robustness of our method under more dynamic conditions, we additionally include a sequence from the KITTI dataset (\textit{2011\_09\_29\_drive\_0071}), which features a large number of moving pedestrians, bicycles and vehicles, as illustrated in Fig.~\ref{fig:kitti_raw_result}(a). We present the mapping results obtained by Octomap, our method and D-Map, with all occupied voxels visualized (see Fig.~\ref{fig:kitti_raw_result}(b)). Note that the baseline methods including Uniform Grid, Hash Grid and UFOMap adopt the same update procedure as Octomap and are therefore expected to produce identical mapping results as Octomap.}

\begin{figure*} [t]
	\centering
	\includegraphics[width=\linewidth]{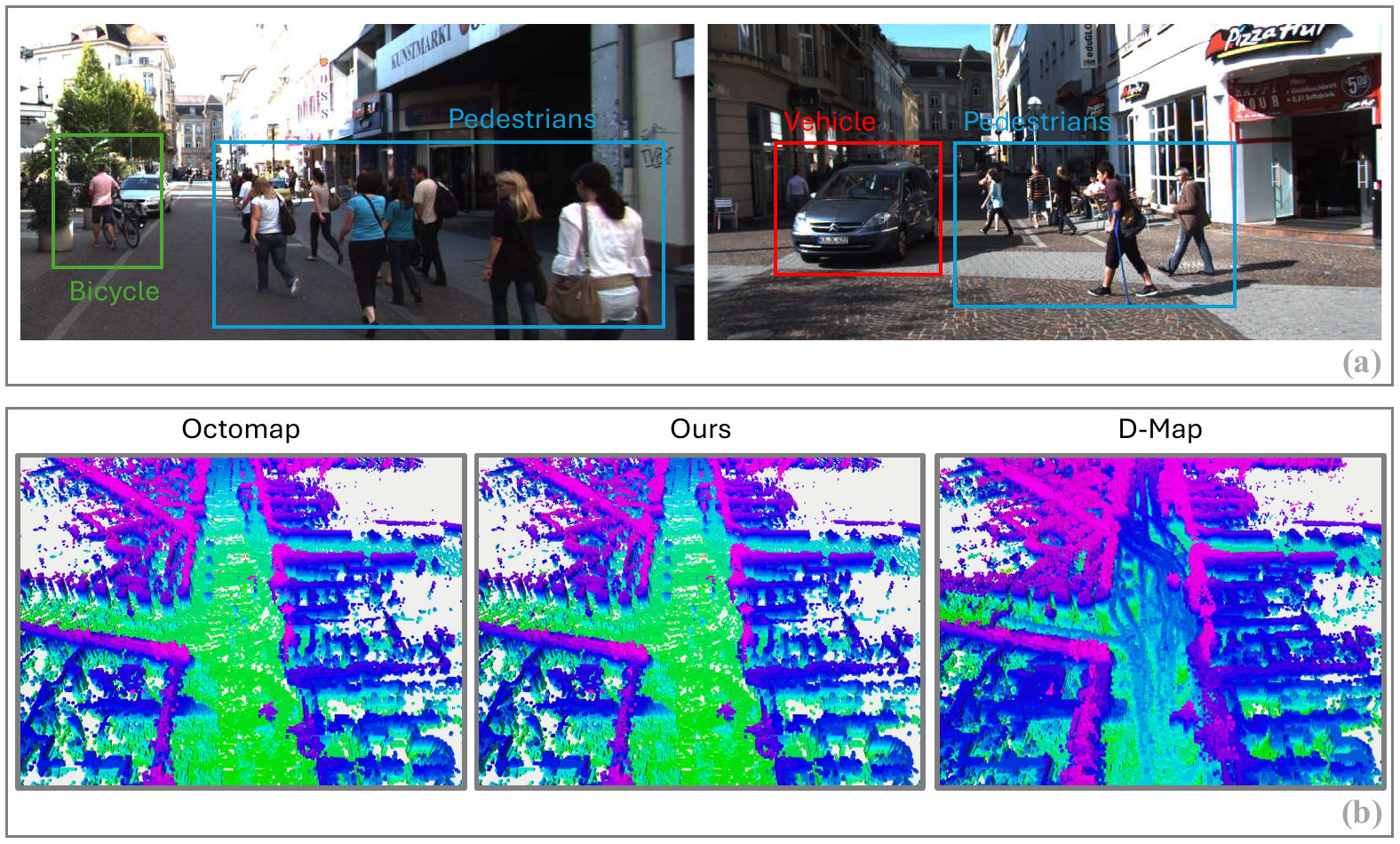}
	\caption{{(a) Sample images from sequence \textit{2011\_09\_29\_drive\_0071}, where dynamic objects (e.g., vehicles, pedestrians and bicycles) are highlighted by boxes. (b) Mapping results of the sequence obtained by Octomap, our method and D-Map, where all occupied voxels are visualized. Since the baseline methods including Uniform Grid, Hash Grid and UFOMap adopt the same update process as Octomap, they are expected to produce identical mapping results with Octomap.}}
	\label{fig:kitti_raw_result}
\end{figure*}

{We can observe that the D-Map's result exhibits residual artifacts left by moving agents. This artifact arises because, in D-Map’s update scheme, once a voxel is marked as occupied, it cannot be reverted to the free state even if subsequent scans no longer observe the obstacle. In contrast, both our method and Octomap effectively handle such dynamic changes and produce clean maps without such artifacts, exhibiting similar map results.}

{We further report the quantitative map accuracy of our method and D-Map on this sequence, using Octomap as the ground-truth reference. The map accuracy of our method is $99.42\%$ while D-Map exhibits noticeable accuracy degradation to $89.81\%$. These results further demonstrate the robustness of our method in this challenging dynamic scenario.}

\subsection{{Limitation and Future Work}}
\label{sec:discuss:limit}
{One limitation of our mapping framework lies in the conversion between log-odds probability value and discrete occupancy state in map updates. Specifically, when a voxel slides out of the local map, its log-odds probability is converted into a discrete occupancy state and represented in the global boundary map. When it re-enters the local map, this discrete occupancy state is converted back into a log-odds probability (see Equation~\ref{eqa:log_odds_init}). Such conversion introduces potential accuracy loss compared with the grid-based and octree-based methods. 
Although this accuracy loss is shown to be negligible in our benchmark experiments (see Section~\ref{sec:exp_acc}), it still represents a limitation of the current framework. 
}

{
In the future, the compactness of the boundary map can be further exploited to address a critical challenge in multi-agent robot systems: the synchronization of occupancy information across multiple agents.
Due to limited real-time communication bandwidth, transmitting large occupancy grid maps between agents is inefficient. In contrast, the compactness of our boundary map makes it particularly well-suited for this task. To enable synchronization of occupancy information by transmitting the boundary map, a method for merging different boundary maps from other agents is necessary. This will be explored as part of our future work.}

\section{Conclusion}
\label{sec:conclude}
In this paper, we introduced a novel method for occupancy grid mapping. Our method maintains only the two-dimensional (2D) boundary voxels rather than all grid voxels in the three-dimensional (3D) space. This low-dimensional representation significantly reduces memory consumption, particularly in high-resolution and large-scale mapping scenarios. A novel method is then proposed for the occupancy state determination of an arbitrary location in 3D environment, based on the 2D boundary voxels. We term our method as the boundary map. Furthermore, we designed a dedicated data structure and an algorithm that enable fast and constant-time occupancy state queries. Moreover, we developed a global-local mapping framework and a corresponding update method, achieving highly efficient map updates from real-time sensor measurements.

Extensive experimental results demonstrated that our method outperforms state-of-the-art occupancy mapping techniques by achieving substantial memory reduction while simultaneously preserving high query and update efficiency. Finally, our system was validated through a real-world experiment, where our mapping framework was deployed to support a long-range autonomous navigation task in a large, complex multi-level building, showcasing its practical capability.

Furthermore, as demonstrated by our experiments, the memory reduction of the proposed 2D boundary voxel representation is especially pronounced in high-resolution and large-scale mapping tasks. This highlights the potential of our method in enabling memory-efficient mapping, when future tasks push toward increasingly large-scale and require more detailed representations.

\section*{Acknowledgement}
The authors would like to thank Dr. Ximin Lyu from Sun Yat-sen University for supporting the field experiments. Some authors gratefully acknowledge the support of scholarships from DJI.

{
\bibliographystyle{SageH}
\bibliography{reference}

@inproceedings{yoder2016autonomous,
  title={Autonomous exploration for infrastructure modeling with a micro aerial vehicle},
  author={Yoder, Luke and Scherer, Sebastian},
  booktitle={Field and Service Robotics: Results of the 10th International Conference},
  pages={427--440},
  year={2016},
  organization={Springer}
}

@article{kong2021avoiding,
  title={Avoiding dynamic small obstacles with onboard sensing and computation on aerial robots},
  author={Kong, Fanze and Xu, Wei and Cai, Yixi and Zhang, Fu},
  journal={IEEE Robotics and Automation Letters},
  volume={6},
  number={4},
  pages={7869--7876},
  year={2021},
  publisher={IEEE}
}

@article{zhou2021fuel,
  title={Fuel: Fast uav exploration using incremental frontier structure and hierarchical planning},
  author={Zhou, Boyu and Zhang, Yichen and Chen, Xinyi and Shen, Shaojie},
  journal={IEEE Robotics and Automation Letters},
  volume={6},
  number={2},
  pages={779--786},
  year={2021},
  publisher={IEEE}
}

@article{cai2023occupancy,
  title={Occupancy grid mapping without ray-casting for high-resolution LiDAR sensors},
  author={Cai, Yixi and Kong, Fanze and Ren, Yunfan and Zhu, Fangcheng and Lin, Jiarong and Zhang, Fu},
  journal={IEEE Transactions on Robotics},
  year={2023},
  publisher={IEEE}
}

@article{xu2022fast,
  title={Fast-lio2: Fast direct lidar-inertial odometry},
  author={Xu, Wei and Cai, Yixi and He, Dongjiao and Lin, Jiarong and Zhang, Fu},
  journal={IEEE Transactions on Robotics},
  volume={38},
  number={4},
  pages={2053--2073},
  year={2022},
  publisher={IEEE}
}

@inproceedings{zhu2022robust,
  title={Robust real-time lidar-inertial initialization},
  author={Zhu, Fangcheng and Ren, Yunfan and Zhang, Fu},
  booktitle={2022 IEEE/RSJ International Conference on Intelligent Robots and Systems (IROS)},
  pages={3948--3955},
  year={2022},
  organization={IEEE}
}

@article{zhu2022decentralized,
  title={Decentralized lidar-inertial swarm odometry},
  author={Zhu, Fangcheng and Ren, Yunfan and Kong, Fanze and Wu, Huajie and Liang, Siqi and Chen, Nan and Xu, Wei and Zhang, Fu},
  journal={arXiv preprint arXiv:2209.06628},
  year={2022}
}

@article{lu2022manifold,
  title={On-manifold model predictive control for trajectory tracking on robotic systems},
  author={Lu, Guozheng and Xu, Wei and Zhang, Fu},
  journal={IEEE Transactions on Industrial Electronics},
  volume={70},
  number={9},
  pages={9192--9202},
  year={2022},
  publisher={IEEE}
}

@article{hornung2013octomap,
  title={OctoMap: An efficient probabilistic 3D mapping framework based on octrees},
  author={Hornung, Armin and Wurm, Kai M and Bennewitz, Maren and Stachniss, Cyrill and Burgard, Wolfram},
  journal={Autonomous robots},
  volume={34},
  pages={189--206},
  year={2013},
  publisher={Springer}
}

@article{duberg2020ufomap,
  title={UFOMap: An efficient probabilistic 3D mapping framework that embraces the unknown},
  author={Duberg, Daniel and Jensfelt, Patric},
  journal={IEEE Robotics and Automation Letters},
  volume={5},
  number={4},
  pages={6411--6418},
  year={2020},
  publisher={IEEE}
}

@inproceedings{dang2019graph,
  title={Graph-based path planning for autonomous robotic exploration in subterranean environments},
  author={Dang, Tung and Mascarich, Frank and Khattak, Shehryar and Papachristos, Christos and Alexis, Kostas},
  booktitle={2019 IEEE/RSJ International Conference on Intelligent Robots and Systems (IROS)},
  pages={3105--3112},
  year={2019},
  organization={IEEE}
}

@inproceedings{amanatides1987fast,
  title={A fast voxel traversal algorithm for ray tracing.},
  author={Amanatides, John and Woo, Andrew and others},
  booktitle={Eurographics},
  volume={87},
  number={3},
  pages={3--10},
  year={1987},
  organization={Citeseer}
}

@article{niessner2013real,
  title={Real-time 3D reconstruction at scale using voxel hashing},
  author={Nie{\ss}ner, Matthias and Zollh{\"o}fer, Michael and Izadi, Shahram and Stamminger, Marc},
  journal={ACM Transactions on Graphics (ToG)},
  volume={32},
  number={6},
  pages={1--11},
  year={2013},
  publisher={ACM New York, NY, USA}
}

@inproceedings{bircher2016receding,
  title={Receding horizon" next-best-view" planner for 3d exploration},
  author={Bircher, Andreas and Kamel, Mina and Alexis, Kostas and Oleynikova, Helen and Siegwart, Roland},
  booktitle={2016 IEEE international conference on robotics and automation (ICRA)},
  pages={1462--1468},
  year={2016},
  organization={IEEE}
}

@article{tranzatto2022cerberus,
  title={Cerberus in the darpa subterranean challenge},
  author={Tranzatto, Marco and Miki, Takahiro and Dharmadhikari, Mihir and Bernreiter, Lukas and Kulkarni, Mihir and Mascarich, Frank and Andersson, Olov and Khattak, Shehryar and Hutter, Marco and Siegwart, Roland and others},
  journal={Science Robotics},
  volume={7},
  number={66},
  pages={eabp9742},
  year={2022},
  publisher={American Association for the Advancement of Science}
}

@article{zhou2023racer,
  title={Racer: Rapid collaborative exploration with a decentralized multi-uav system},
  author={Zhou, Boyu and Xu, Hao and Shen, Shaojie},
  journal={IEEE Transactions on Robotics},
  volume={39},
  number={3},
  pages={1816--1835},
  year={2023},
  publisher={IEEE}
}

@inproceedings{rouvcek2020darpa,
  title={Darpa subterranean challenge: Multi-robotic exploration of underground environments},
  author={Rou{\v{c}}ek, Tom{\'a}{\v{s}} and Pecka, Martin and {\v{C}}{\'\i}{\v{z}}ek, Petr and Pet{\v{r}}{\'\i}{\v{c}}ek, Tom{\'a}{\v{s}} and Bayer, Jan and {\v{S}}alansk{\`y}, Vojt{\v{e}}ch and He{\v{r}}t, Daniel and Petrl{\'\i}k, Mat{\v{e}}j and B{\'a}{\v{c}}a, Tom{\'a}{\v{s}} and Spurn{\`y}, Voj{\v{e}}ch and others},
  booktitle={Modelling and Simulation for Autonomous Systems: 6th International Conference, MESAS 2019, Palermo, Italy, October 29--31, 2019, Revised Selected Papers 6},
  pages={274--290},
  year={2020},
  organization={Springer}
}

@inproceedings{tang2023bubble,
  title={Bubble Explorer: Fast UAV Exploration in Large-Scale and Cluttered 3D-Environments using Occlusion-Free Spheres},
  author={Tang, Benxu and Ren, Yunfan and Zhu, Fangcheng and He, Rui and Liang, Siqi and Kong, Fanze and Zhang, Fu},
  booktitle={2023 IEEE/RSJ International Conference on Intelligent Robots and Systems (IROS)},
  pages={1118--1125},
  year={2023},
  organization={IEEE}
}

@article{zhou2020ego,
  title={Ego-planner: An esdf-free gradient-based local planner for quadrotors},
  author={Zhou, Xin and Wang, Zhepei and Ye, Hongkai and Xu, Chao and Gao, Fei},
  journal={IEEE Robotics and Automation Letters},
  volume={6},
  number={2},
  pages={478--485},
  year={2020},
  publisher={IEEE}
}

@inproceedings{ren2022bubble,
  title={Bubble planner: Planning high-speed smooth quadrotor trajectories using receding corridors},
  author={Ren, Yunfan and Zhu, Fangcheng and Liu, Wenyi and Wang, Zhepei and Lin, Yi and Gao, Fei and Zhang, Fu},
  booktitle={2022 IEEE/RSJ International Conference on Intelligent Robots and Systems (IROS)},
  pages={6332--6339},
  year={2022},
  organization={IEEE}
}

@article{ren2023rog,
  title={ROG-map: An efficient robocentric occupancy grid map for large-scene and high-resolution LiDAR-based motion planning},
  author={Ren, Yunfan and Cai, Yixi and Zhu, Fangcheng and Liang, Siqi and Zhang, Fu},
  journal={arXiv preprint arXiv:2302.14819},
  year={2023}
}

@article{geiger2013vision,
  title={Vision meets robotics: The kitti dataset},
  author={Geiger, Andreas and Lenz, Philip and Stiller, Christoph and Urtasun, Raquel},
  journal={The International Journal of Robotics Research},
  volume={32},
  number={11},
  pages={1231--1237},
  year={2013},
  publisher={Sage Publications Sage UK: London, England}
}

@article{agarwal2020ford,
  title={Ford multi-AV seasonal dataset},
  author={Agarwal, Siddharth and Vora, Ankit and Pandey, Gaurav and Williams, Wayne and Kourous, Helen and McBride, James},
  journal={The International Journal of Robotics Research},
  volume={39},
  number={12},
  pages={1367--1376},
  year={2020},
  publisher={SAGE Publications Sage UK: London, England}
}

@article{ren2025safety,
  title={Safety-assured high-speed navigation for MAVs},
  author={Ren, Yunfan and Zhu, Fangcheng and Lu, Guozheng and Cai, Yixi and Yin, Longji and Kong, Fanze and Lin, Jiarong and Chen, Nan and Zhang, Fu},
  journal={Science Robotics},
  volume={10},
  number={98},
  pages={eado6187},
  year={2025},
  publisher={American Association for the Advancement of Science}
}

@inproceedings{yang2022far,
  title={Far planner: Fast, attemptable route planner using dynamic visibility update},
  author={Yang, Fan and Cao, Chao and Zhu, Hongbiao and Oh, Jean and Zhang, Ji},
  booktitle={2022 ieee/rsj international conference on intelligent robots and systems (iros)},
  pages={9--16},
  year={2022},
  organization={IEEE}
}

@article{o2012gaussian,
  title={Gaussian process occupancy maps},
  author={O’Callaghan, Simon T and Ramos, Fabio T},
  journal={The International Journal of Robotics Research},
  volume={31},
  number={1},
  pages={42--62},
  year={2012},
  publisher={SAGE Publications Sage UK: London, England}
}

@inproceedings{kim2012building,
  title={Building occupancy maps with a mixture of Gaussian processes},
  author={Kim, Soohwan and Kim, Jonghyuk},
  booktitle={2012 IEEE International Conference on Robotics and Automation},
  pages={4756--4761},
  year={2012},
  organization={IEEE}
}

@inproceedings{kim2015gpmap,
  title={GPmap: A unified framework for robotic mapping based on sparse Gaussian processes},
  author={Kim, Soohwan and Kim, Jonghyuk},
  booktitle={Field and Service Robotics: Results of the 9th International Conference},
  pages={319--332},
  year={2015},
  organization={Springer}
}

@inproceedings{wang2016fast,
  title={Fast, accurate gaussian process occupancy maps via test-data octrees and nested bayesian fusion},
  author={Wang, Jinkun and Englot, Brendan},
  booktitle={2016 IEEE International Conference on Robotics and Automation (ICRA)},
  pages={1003--1010},
  year={2016},
  organization={IEEE}
}

@article{guizilini2018towards,
  title={Towards real-time 3D continuous occupancy mapping using Hilbert maps},
  author={Guizilini, Vitor and Ramos, Fabio},
  journal={The International Journal of Robotics Research},
  volume={37},
  number={6},
  pages={566--584},
  year={2018},
  publisher={SAGE Publications Sage UK: London, England}
}

@article{o2018variable,
  title={Variable resolution occupancy mapping using gaussian mixture models},
  author={O’Meadhra, Cormac and Tabib, Wennie and Michael, Nathan},
  journal={IEEE Robotics and Automation Letters},
  volume={4},
  number={2},
  pages={2015--2022},
  year={2018},
  publisher={IEEE}
}

@inproceedings{zhi2019continuous,
  title={Continuous occupancy map fusion with fast bayesian hilbert maps},
  author={Zhi, Weiming and Ott, Lionel and Senanayake, Ransalu and Ramos, Fabio},
  booktitle={2019 International Conference on Robotics and Automation (ICRA)},
  pages={4111--4117},
  year={2019},
  organization={IEEE}
}

@article{tabib2021autonomous,
  title={Autonomous cave surveying with an aerial robot},
  author={Tabib, Wennie and Goel, Kshitij and Yao, John and Boirum, Curtis and Michael, Nathan},
  journal={IEEE Transactions on Robotics},
  volume={38},
  number={2},
  pages={1016--1032},
  year={2021},
  publisher={IEEE}
}

@inproceedings{lopez2017aggressive,
  title={Aggressive 3-D collision avoidance for high-speed navigation.},
  author={Lopez, Brett Thomas and How, Jonathan P},
  booktitle={ICRA},
  pages={5759--5765},
  year={2017}
}

@article{roth1989building,
  title={Building an environment model using depth information},
  author={Roth-Tabak, Yuval and Jain, Ramesh},
  journal={Computer},
  volume={22},
  number={6},
  pages={85--90},
  year={1989},
  publisher={IEEE}
}

@article{moravec1996robot,
  title={Robot evidence grids},
  author={Moravec, Martin C Martin Hans P},
  journal={CMU Robotics Institute Technical Report CMU-RI-TR-96-06},
  year={1996},
  publisher={Citeseer}
}

@inproceedings{elfes1995robot,
  title={Robot navigation: Integrating perception, environmental constraints and task execution within a probabilistic framework},
  author={Elfes, Alberto},
  booktitle={International Workshop on Reasoning with uncertainty in Robotics},
  pages={91--130},
  year={1995},
  organization={Springer}
}

@article{kawatsuma2012emergency,
  title={Emergency response by robots to Fukushima-Daiichi accident: summary and lessons learned},
  author={Kawatsuma, Shinji and Fukushima, Mineo and Okada, Takashi},
  journal={Industrial Robot: An International Journal},
  volume={39},
  number={5},
  pages={428--435},
  year={2012},
  publisher={Emerald Group Publishing Limited}
}

@inproceedings{seungsub2017study,
  title={A study on the disaster response scenarios using robot technology},
  author={SeungSub, Oh and Jehun, Hahm and Hyunjung, Jang and Soyeon, Lee and Jinho, Suh},
  booktitle={2017 14th International Conference on Ubiquitous Robots and Ambient Intelligence (URAI)},
  pages={520--523},
  year={2017},
  organization={IEEE}
}

@article{schmid2020efficient,
  title={An efficient sampling-based method for online informative path planning in unknown environments},
  author={Schmid, Lukas and Pantic, Michael and Khanna, Raghav and Ott, Lionel and Siegwart, Roland and Nieto, Juan},
  journal={IEEE Robotics and Automation Letters},
  volume={5},
  number={2},
  pages={1500--1507},
  year={2020},
  publisher={IEEE}
}

@inproceedings{isler2016information,
  title={An information gain formulation for active volumetric 3D reconstruction},
  author={Isler, Stefan and Sabzevari, Reza and Delmerico, Jeffrey and Scaramuzza, Davide},
  booktitle={2016 IEEE International Conference on Robotics and Automation (ICRA)},
  pages={3477--3484},
  year={2016},
  organization={IEEE}
}

@inproceedings{cao2021tare,
  title={TARE: A Hierarchical Framework for Efficiently Exploring Complex 3D Environments.},
  author={Cao, Chao and Zhu, Hongbiao and Choset, Howie and Zhang, Ji},
  booktitle={Robotics: Science and Systems},
  volume={5},
  pages={2},
  year={2021}
}

@inproceedings{liu2024omninxt,
  title={Omninxt: A fully open-source and compact aerial robot with omnidirectional visual perception},
  author={Liu, Peize and Feng, Chen and Xu, Yang and Ning, Yan and Xu, Hao and Shen, Shaojie},
  booktitle={2024 IEEE/RSJ International Conference on Intelligent Robots and Systems (IROS)},
  pages={10605--10612},
  year={2024},
  organization={IEEE}
}

@article{maddern20171,
  title={1 year, 1000 km: The oxford robotcar dataset},
  author={Maddern, Will and Pascoe, Geoffrey and Linegar, Chris and Newman, Paul},
  journal={The International Journal of Robotics Research},
  volume={36},
  number={1},
  pages={3--15},
  year={2017},
  publisher={SAGE Publications Sage UK: London, England}
}

@article{jung2024helipr,
  title={HeLiPR: Heterogeneous LiDAR dataset for inter-LiDAR place recognition under spatiotemporal variations},
  author={Jung, Minwoo and Yang, Wooseong and Lee, Dongjae and Gil, Hyeonjae and Kim, Giseop and Kim, Ayoung},
  journal={The International Journal of Robotics Research},
  volume={43},
  number={12},
  pages={1867--1883},
  year={2024},
  publisher={SAGE Publications Sage UK: London, England}
}
}

{
\appendix

\section{Proof of Theorem~\ref{thm:query}}
\label{app:proof_query}
\begin{proof}
By the condition of the theorem, the nearest boundary voxel in the search direction \(z^+\), denoted as \(\mathbf{b}_{{nn}}\), exists. As illustrated in Figure~\ref{fig:proof}, let \(\{\mathbf{n}_1, \mathbf{n}_2, \ldots, \mathbf{n}_k\}\) denote the sequence of voxels between the query voxel \(\mathbf{q}\) and \(\mathbf{b}_{{nn}}\), where \(\mathbf{n}_1\) is adjacent to \(\mathbf{q}\) in \(z^+\) direction, and \(\mathbf{b}_{{nn}}\) is adjacent to \(\mathbf{n}_k\) in \(z^+\) direction.

We first consider the case where \(\mathbf{b}_{nn}\) is a boundary interior voxel \(\mathbf{b}_\mathrm{int}\). In this case, \(\mathbf{b}_{nn}\) has occupancy state as free. Since \(\mathbf{b}_{nn}\) is the nearest boundary voxel of \(\textbf{q}\) in the search direction \(z^+\), none of the voxels in \(\{\mathbf{n}_1, \mathbf{n}_2, \ldots, \mathbf{n}_k\}\) are boundary voxels. In particular, consider voxel \(\mathbf{n}_k\). If \(\mathbf{n}_k\) were {occupied} or {unknown}, it would, by definition, be classified as a boundary exterior voxel \(\mathbf{b}_\mathrm{ext}\). Specifically, if it were occupied, it would be classified as \(\mathbf{b}_\mathrm{occ}\); if unknown, it would be classified as \(\mathbf{b}_\mathrm{ukn}\), due to having a 6-neighbor \(\mathbf{b}_{{nn}}\) in state {free} (see Equation~\ref{eqa:boundary_def}). This contradicts the condition that \(\mathbf{n}_k\) is not a boundary voxel. Hence, \(\mathbf{n}_k\) must be {free}. By induction, since \(\mathbf{n}_i\) is {free} and \(\mathbf{n}_{i-1}\) is not a boundary voxel, then \(\mathbf{n}_{i-1}\) must also be {free}, for all \(i = k, k-1, \ldots, 2\). Finally, since \(\mathbf{q}\) is not a boundary voxel and \(\mathbf{n}_1\) is {free}, it follows that \(\mathbf{q}\) is also {free}.

Alternatively, if \(\mathbf{b}_{{nn}}\) is a boundary exterior voxel \(\mathbf{b}_\mathrm{ext}\), its occupancy state is unknown or occupied. Again, since \(\mathbf{n}_k\) is not a boundary voxel and is adjacent to \(\mathbf{b}_{{nn}}\), it cannot be {free} or {occupied}, otherwise it would be classified as a boundary voxel (see Equation~\ref{eqa:boundary_def}). Thus, \(\mathbf{n}_k\) must be {unknown}. The same inductive argument applies in reverse: since \(\mathbf{n}_i\) is {unknown} and \(\mathbf{n}_{i-1}\) is not a boundary voxel, \(\mathbf{n}_{i-1}\) is also {unknown}, for all \(i = k, k-1, \ldots, 2\). Finally, since \(\mathbf{n}_1\) is {unknown} and the query voxel \(\mathbf{q}\) is not a boundary voxel, we conclude that \(\mathbf{q}\) must also be {unknown}.
\end{proof}

\section{Proof of Theorem~\ref{thm:query_null}}
\label{app:proof_query_null}
\begin{proof}

As illustrated in Figure~\ref{fig:proof}, consider a reference voxel \(\mathbf{m}_{\text{ref}}\), located along the search direction \(z^+\) but outside the spatial extent of the mapping environment. Since voxels beyond the extent are never updated by sensor measurements, \(\mathbf{m}_{\text{ref}}\) must remain in the {unknown} state.

By the condition of the theorem, there exists no boundary voxel between \(\mathbf{m}_{\text{ref}}\) and the query voxel \(\mathbf{q}\), including \(\mathbf{q}\) itself. Similar to previous proof (see Appendix~\ref{app:proof_query}), we denote the sequence of voxels between \(\mathbf{q}\) and \(\mathbf{m}_{\text{ref}}\) as \(\{\mathbf{m}_1, \mathbf{m}_2, \ldots, \mathbf{m}_l\}\) , where \(\mathbf{m}_1\) is adjacent to \(\mathbf{q}\) and \(\mathbf{m}_l\) is adjacent to \(\mathbf{m}_{\text{ref}}\). Since \(\mathbf{m}_l\) is not a boundary voxel and has a 6-neighbor \(\mathbf{m}_{\text{ref}}\), which is {unknown}, \(\mathbf{m}_l\) cannot be in the {free} or {occupied} state. This is because if so, \(\mathbf{m}_l\) would be classified as a boundary voxel (see Equation~\ref{eqa:boundary_def}), contradicting the condition. Thus, \(\mathbf{m}_l\) must be {unknown}. By induction, we conclude that for all \(i = l, l-1, \ldots, 2\), since \(\mathbf{m}_i\) is {unknown} and \(\mathbf{m}_{i-1}\) is not a boundary voxel, \(\mathbf{m}_{i-1}\) must also be {unknown}. Finally, since \(\mathbf{m}_1\) is {unknown} and \(\mathbf{q}\) is not a boundary voxel, it follows that \(\mathbf{q}\) is also {unknown}.
\end{proof}

\begin{figure}[t] 
    \centering
    \includegraphics[width=\linewidth]{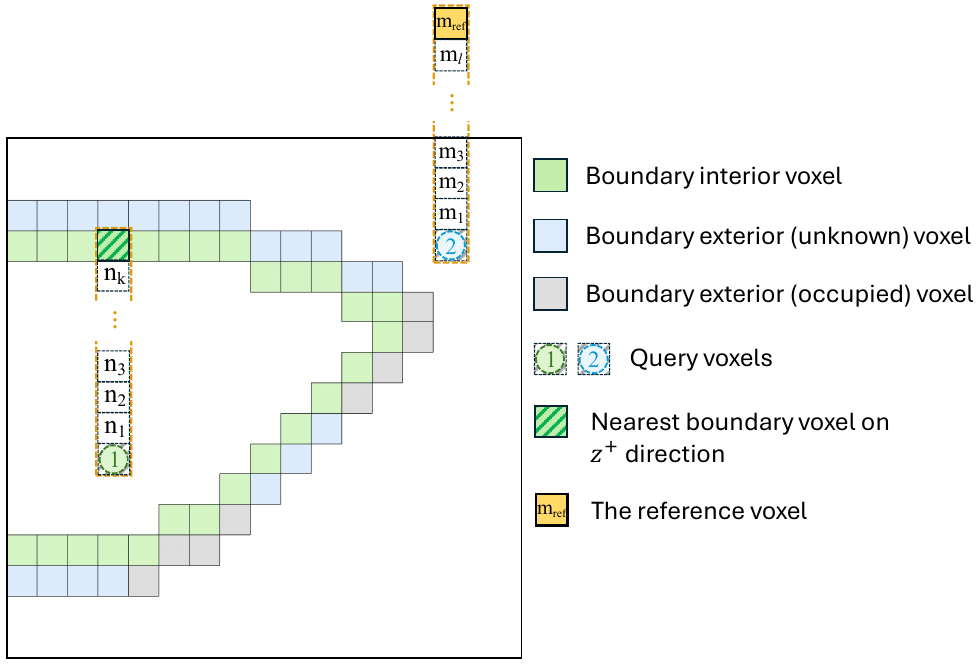} 
    \vspace{-10pt} 
    \caption{Illustration of the proof of Theorem~\ref{thm:query} and Theorem~\ref{thm:query_null}.}
    \label{fig:proof}
\end{figure}

\section{Proof of Theorem~\ref{thm:time_complx}}
\label{app:proof_time_complx}
\begin{proof}
Let the mapping environment be denoted by \(\mathcal{H}\), with its spatial scale denoted by \( D \) and map resolution denoted by \( d \). The total time complexity of querying the occupancy state for all voxels in \(\mathcal{H}\) is denoted by \( \mathcal{Q}_\text{all} \).

Consider the occupancy query process for a voxel \(\mathbf{q}\) in \(\mathcal{H}\). This process involves two steps. First, the array saved in the 2D grid cell \(\mathbf{c}_q = (q^x, q^y) \in \mathbb{Z}^2\) is retrieved. This operation involves a hash table lookup and has an average-case time complexity of \( \mathcal{O}(1) \). Second, a binary search is performed on the retrieved array. Let \( k(\mathbf{q}) \) denote the size of the array. This size also corresponds to the number of boundary voxels that share the same \( (x, y) \)-coordinates as the query voxel \(\mathbf{q}\), which is \((q^x, q^y)\). The binary search has time complexity \( \mathcal{O}(\log k(\mathbf{q})) \), which we conservatively upper bound as \( \mathcal{O}(k(\mathbf{q})) \) for analysis. Therefore, the per-voxel query complexity is \( \mathcal{O}(1 + k(\mathbf{q})) \). Then, the total time complexity for querying all voxels in the environment \(\mathcal{H}\) is given by:

\begin{align}
\mathcal{Q}_\text{all} &= \mathcal{O} \left( \sum_{\mathbf{q} \in \mathcal{H}} \left(1 + k(\mathbf{q}) \right) \right) \\
 &= \mathcal{O}\left( \sum_{q^z=-D}^{D} \sum_{q^y=-D}^{D} \sum_{q^x=-D}^{D} (1+ k(\mathbf{q})) \right)
\end{align}
where the triple summation corresponds to iterating over all voxels in the 3D environment \(\mathcal{H}\).

Note that \( k(\mathbf{q}) \) represents the total number of all boundary voxels that share the same \( (x, y) \)-coordinates as \((q^x, q^y)\). Therefore, \( \sum_{q^y=-D}^{D} \sum_{q^x=-D}^{D} k(\mathbf{q}) \) corresponds to the total number of all boundary voxels in the boundary map, which we denote as \( K \). Since this total number \( K \) is constant across all \( q^z \in [-D, D] \), and the number of voxels along the \( z \)-axis is \( \frac{D}{d} \), we obtain:

\begin{align}
\mathcal{Q}_\text{all} &= \mathcal{O}\left( \sum_{q^z=-D}^{D} \sum_{q^y=-D}^{D} \sum_{q^x=-D}^{D} (1+ k(\mathbf{q}))\right)\\ 
&= \mathcal{O}\left( \left(\frac{D}{d} \right)^3 + \sum_{q^z=-D}^{D} \left(\sum_{q^y=-D}^{D} \sum_{q^x=-D}^{D} k(\mathbf{q})\right)\right)\\ 
&= \mathcal{O}\left( \left(\frac{D}{d} \right)^3 + \sum_{q^z=-D}^{D} K \right) \\
&= \mathcal{O}\left( \left(\frac{D}{d} \right)^3 + \frac{D}{d} K \right)
\end{align}

{The total number of boundary voxels in the environment, denoted as \( K \), depends on the geometric complexity of the environment. 
As illustrated in Fig.~\ref{fig:proof_o1}(a), in an environment where the free space forms a tortuous and narrow corridor, the boundary voxels may densely fill almost the entire environment, leading to a worst-case complexity on the order of \( \mathcal{O}\!\left( \left( \frac{D}{d} \right)^3 \right) \). However, such a scenario is a highly idealized worst-case.
In contrast, in typical real-world environments, the free space forms relatively smooth and enclosed regions, as illustrated in Fig.~\ref{fig:proof_o1}(b). The boundary voxels are distributed along the surfaces, and their number scales with the surface area, i.e., \( \mathcal{O}\!\left( \left( \frac{D}{d} \right)^2 \right) \). Thus, in the real-world average case, we can consider \( \mathcal{O}(K) \) as \( \mathcal{O}\!\left( \left( \frac{D}{d} \right)^2 \right) \). Therefore, we obtain:}

\begin{align}
\mathcal{Q}_\text{all} &= \mathcal{O}\left( \left(\frac{D}{d} \right)^3 + \frac{D}{d} K \right) \\ &= \mathcal{O}\left( \left(\frac{D}{d} \right)^3 + \left( \frac{D}{d} \right)^3 \right) \\ &= \mathcal{O}\left(\left( \frac{D}{d} \right)^3 \right)
\end{align}

Finally, the environment \(\mathcal{H}\) consists of \( \left( \frac{D}{d} \right)^3 \) voxels. Therefore, the average query time complexity per voxel, denoted by \( \mathcal{Q}_\text{p} \), is given by:

\begin{align}
\mathcal{Q}_\text{p} = \frac{\mathcal{Q}_\text{all}}{ \left( \frac{D}{d} \right)^3} = \frac{\mathcal{O}\left(\left( \frac{D}{d} \right)^3 \right)}{ \left( \frac{D}{d} \right)^3} = \mathcal{O}(1)
\end{align}

\begin{figure}[t] 
    \centering
    \includegraphics[width=\linewidth]{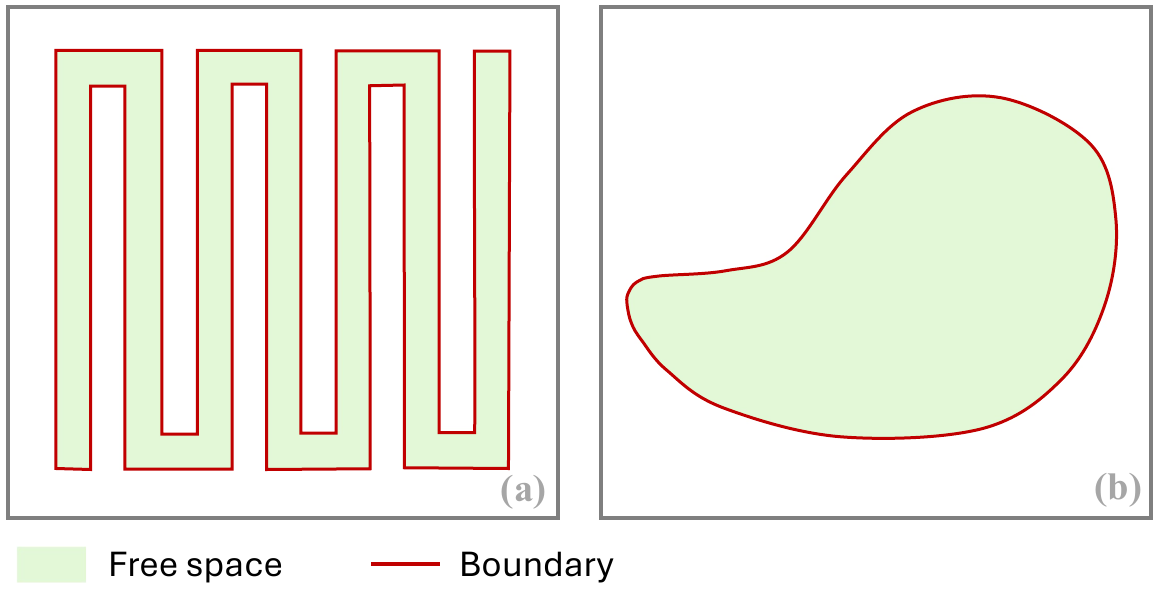} 
    \vspace{-10pt} 
    \caption{{Illustration of boundary voxels distribution in the environment. (a) An idealized worst-case scenario. (b) A typical real-world scenario.}}
    \label{fig:proof_o1}
\end{figure}

\end{proof}

\section{Proof of Claim in Section~\ref{sec:update_global}}
\label{app:proof_upd}
\begin{proof}
By the definition of a boundary voxel (see Equation~\ref{eqa:boundary_def}), a voxel's boundary voxel status is determined solely by its own occupancy state and the occupancy states of its six neighbors. Therefore, if the occupancy state of neither the voxel nor any of its six neighbors changes, its boundary voxel status remains unchanged. By construction, for any voxel not included in \( \mathcal{U} \), both its own occupancy state and the occupancy states of all its six neighbors remain unchanged. Thus, the boundary voxel status of such a voxel remains unchanged.
\end{proof}

\section{Proof of Theorem~\ref{thm:restore_null}}
\label{app:proof_restore_null}
\begin{proof}
Let \(\mathbf{n}_p, \mathbf{n}_q \) in region \( \mathcal{V}_a \cap \mathcal{S_I}\) be any two voxels such that \(\mathbf{n}_q\) lies above \(\mathbf{n}_p\) in the \(z\)-axis. Given that there are no boundary voxels within the region \(\mathcal{V}_a \cap \mathcal{S_I}\), we have the following: Since \(\mathbf{n}_p\) is not a boundary voxel, it cannot be {occupied}; therefore, its occupancy state must be either {free} or {unknown}. Let \(\mathbf{n}_1\) denote the neighbor in \(z^+\) direction of \(\mathbf{n}_p\). We first consider the case where \(\mathbf{n}_p\) is free. If \(\mathbf{n}_1\) has an occupancy state of occupied or unknown, then \(\mathbf{n}_p\) would be classified as a boundary voxel (see Equation~\ref{eqa:boundary_def}), which contradicts the condition of the theorem. Hence \(\mathbf{n}_1\) has an occupancy state of free, which is the same as \(\mathbf{n}_p\). By similar reasoning, if \(\mathbf{n}_p\) is unknown, we can conclude that \(\mathbf{n}_1\) must also be unknown. Therefore, we conclude that \(\mathbf{m}_1\) must share the same state (i.e., free or unknown) as \(\mathbf{n}_p\). Then, we consider the the neighbor in \(z^+\) direction of \(\mathbf{n}_1\), denoted as \(\mathbf{n}_2\). By induction, we can conclude that \(\mathbf{n}_2\) must share the same state as \(\mathbf{n}_1\), and by repeating this reasoning recursively along the \(z^+\) direction, we can conclude that all voxels between \(\mathbf{n}_p\) and \(\mathbf{n}_q\), including the \(\mathbf{n}_q\), must share the same occupancy state as \(\mathbf{n}_p\), which is either {free} or {unknown}.

\end{proof}

\section{Proof of Theorem~\ref{thm:restore_free}}
\label{app:proof_restore_RG}
\begin{proof}
We first prove that all voxels within the constructed tiles are free. By construction, for any voxel within the tile, the nearest boundary voxel along either the $z^+$ or $z^-$ direction is of type $\mathbf{b}_\mathrm{int}$. By Theorem~\ref{thm:query}, such a voxel's occupancy state is determined as free, which completes the first part of the proof.

Next, we show that any non-boundary voxel within region \(\mathcal{V}_a \cap \mathcal{S_I}\) that is not included in the constructed tiles must be labeled as unknown. 
We prove this by contradiction. Suppose there exists a free voxel \(\mathbf{n}_f\) in \(\mathcal{V}_a \cap \mathcal{S_I}\) that is not included in any constructed tile. By the condition of the theorem, which indicates that boundary voxel(s) exist within the region \(\mathcal{V}_a \cap \mathcal{S_I}\), there must be a nearest boundary voxel to \(\mathbf{n}_f\) along either the \(z^+\) or \(z^-\) direction, which we denote as \(\mathbf{b}_{nn}\). By Theorem~\ref{thm:query}, \(\mathbf{b}_{nn}\) must be of type \(\mathbf{b}_\mathrm{int}\), as \(\mathbf{n}_f\) is assumed to be a free voxel. By construction, no boundary voxel exists between \(\mathbf{b}_{nn}\) and \(\mathbf{n}_f\), since \(\mathbf{b}_{nn}\) is the nearest boundary voxel. The tile construction begins at \(\mathbf{b}_{nn}\), which is a \(\mathbf{b}_\mathrm{int}\) voxel, and continues until it encounters a boundary voxel. Therefore, \(\mathbf{n}_f\) must be included in the tile construction, which contradicts the initial assumption and completes the proof.

\end{proof}

\section{{Analytical Memory Usage Estimation}}
{In this appendix, we provide a rough estimation of the memory requirements of our mapping framework under different environmental conditions. 
Let the {environment sparsity} be denoted as \( \rho_e \), defined as the ratio of occupied voxels (\( N_o \)) to the total number of voxels in the environment (\( N_{\text{total}} \)). 
In average cases, each boundary exterior voxel corresponds to one boundary interior voxel. Thus, the overall memory usage can be estimated as:
\begin{equation}
\texttt{memory} \propto 2 N_o.
\end{equation}
The number of occupied voxels can be further expressed as:
\begin{equation}
N_o = \rho_e N_{\text{total}} = \rho_e \frac{{V}_e}{d^3},
\end{equation}
where \( {V}_e \) denotes the environment total volume and \( d \) is the map resolution (i.e., voxel size). 
Thus, the estimated memory usage can be written as:
\begin{equation}
\texttt{memory} \propto 2 \rho_e \frac{{V}_e}{d^3}.
\end{equation}
}

}

\end{document}